\documentclass{article}

\usepackage[final]{neurips_2025}

\usepackage[utf8]{inputenc} 
\usepackage[T1]{fontenc}    
\usepackage{hyperref}       
\usepackage{url}            
\usepackage{booktabs}       
\usepackage{amsfonts}       
\usepackage{amssymb} 
\usepackage{amsthm} 
\usepackage{nicefrac}       
\usepackage{microtype}      
\usepackage{xcolor}         
\hypersetup{
    colorlinks=true,
    linkcolor=red,
    citecolor=teal,
    urlcolor=cyan,
}
\usepackage{enumitem}
\usepackage{etoc}
\etocdepthtag.toc{mtchapter}
\etocsettagdepth{mtchapter}{subsubsection}
\etocsettagdepth{mtappendix}{none}

\usepackage{sail}
\usepackage{amsmath}
\usepackage{wrapfig}

\usepackage{tabularx}
\usepackage{multirow}
\usepackage{makecell}
\usepackage{colortbl}
\usepackage{xcolor}
\usepackage{threeparttable} 
\usepackage{graphicx}
\usepackage{subfig}
\usepackage{algorithm}
\usepackage{algorithmic}
\makeatletter
\newcommand\figcaption{\def\@captype{figure}\caption}
\newcommand\tabcaption{\def\@captype{table}\caption}

\makeatother

\newcolumntype{B}{>{\bfseries}c}

\title{Physics-Driven Spatiotemporal Modeling for AI-Generated Video Detection}

\author{
    Shuhai Zhang\textsuperscript{\rm 1 \rm 4}\thanks{Equal contribution. Email: shuhaizhangshz@gmail.com, lianzihaolzh@gmail.com} ~~ 
    Zihao Lian\textsuperscript{\rm 1 }\footnotemark[1] ~~
    Jiahao Yang\textsuperscript{\rm 1} ~~
    Daiyuan Li\textsuperscript{\rm 1} ~~
    Guoxuan Pang\textsuperscript{\rm 2}\\
    \textbf{Feng Liu}\textsuperscript{\rm 5} ~
    \textbf{Bo Han}\textsuperscript{\rm 7} ~
    \textbf{Shutao Li}\textsuperscript{\rm 6}\footnotemark[2] ~
    \textbf{Mingkui Tan}\textsuperscript{\rm 1 \rm 3}\thanks{Corresponding author. Email: mingkuitan@scut.edu.cn, shutao\_li@hnu.edu.cn} ~ \\
    \textsuperscript{\scriptsize{\rm 1}}\small{South China University of Technology,}
    \textsuperscript{\rm 2}\small{University of Science and Technology of China} \\
    \textsuperscript{\rm 3}\small{Key Laboratory of Big Data and Intelligent Robot, Ministry of Education,} \textsuperscript{\rm 4}\small{Pazhou Lab} \\
    \textsuperscript{\scriptsize{\rm 5}}\small{University of Melbourne,}
    \textsuperscript{\scriptsize{\rm 6}}\small{Hunan University,}
    \textsuperscript{\scriptsize{\rm 7}}\small{Hong Kong Baptist University}    
}

\begin{document}

\maketitle

\begin{abstract}
AI-generated videos have achieved near-perfect visual realism (\eg, Sora), urgently necessitating reliable detection mechanisms. However, detecting such videos faces significant challenges in modeling high-dimensional spatiotemporal dynamics and identifying subtle anomalies that violate physical laws. In this paper, we propose the first physics-driven AI-generated video detection paradigm based on probability flow conservation principles. Specifically, we propose a statistic called \textit{Normalized Spatiotemporal Gradient} (NSG), which quantifies the ratio of spatial probability gradients to temporal density changes, explicitly capturing deviations from natural video dynamics. Leveraging pre-trained diffusion models, we develop an NSG estimator through spatial gradients approximation and motion-aware temporal modeling without complex motion decomposition while preserving physical constraints. Building on this, we propose an NSG-based video detection method (NSG-VD) that computes the \textit{Maximum Mean Discrepancy} (MMD) between NSG features of the test and real videos as a detection metric. Last, we derive an upper bound of NSG feature distances between real and generated videos, proving that generated videos exhibit amplified discrepancies due to distributional shifts. Extensive experiments confirm that NSG-VD outperforms state-of-the-art baselines by 16.00\% in Recall and 10.75\% in F1-Score,  validating the superior performance of NSG-VD. The source code is available at \url{https://github.com/ZSHsh98/NSG-VD}.
\end{abstract}

\section{Introduction}

The rapid advancement of generative models \cite{blattmann2023stable,blattmann2023align,brooks2024video,wu2025customcrafter,chi2024unveiling,huang2025fine}, particularly diffusion-based frameworks (\eg, Sora \cite{brooks2024video}), has achieved unprecedented capabilities in synthesizing photorealistic video content. While these breakthroughs enable transformative applications in content creation for creative industries \cite{chen2024videocrafter2,xing2024dynamicrafter,shi2024motion,li2025image}, they simultaneously pose critical societal risks through malicious manipulation (\eg, deepfake disinformation \cite{bao2018towards,guo2024pulid,zheng2023out,wang2021hififace,zhao2023diffswap,oorloff2024avff,li2023learning}, synthetic media fraud \cite{chen2024videocrafter2,rombach2022high}). As AI-generated videos become increasingly realistic in both spatial and temporal domains, developing effective video detection methods becomes critically urgent for preserving societal trust in digital media.

A fundamental challenge for AI-generated video detection lies in modeling the spatiotemporal dynamics of video evolution. Intuitively, natural videos inherently obey physical laws like motion coherence and texture continuity \cite{Huang_2024_CVPR,zheng2025vbench}, while AI-generated videos often exhibit subtle yet systematic inconsistencies in spatiotemporal coherence \cite{ouyang2024codef}. This observation raises a crucial question: 

\textit{How can we model intrinsic spatiotemporal dynamics of natural videos to expose synthetic anomalies?}

Two critical difficulties arise: 1) Video content inherently contains complex \textit{spatial domain correlations} (\eg, texture structure) and \textit{temporal domain dependencies} (\eg, motion trajectories), requiring modeling frameworks that jointly capture both spatial structures and temporal dynamics characteristics. 2) AI-generated videos are rapidly approaching the perceptual quality of natural videos, with discrepancies that may become vanishingly \textit{subtle} in both visual appearance and temporal evolution.

\begin{figure*}[t]
    \vspace{-22pt}
    \begin{center}
    \subfloat[Traditional Spatiotemporal Modeling]
    {\includegraphics[width=0.48\linewidth]{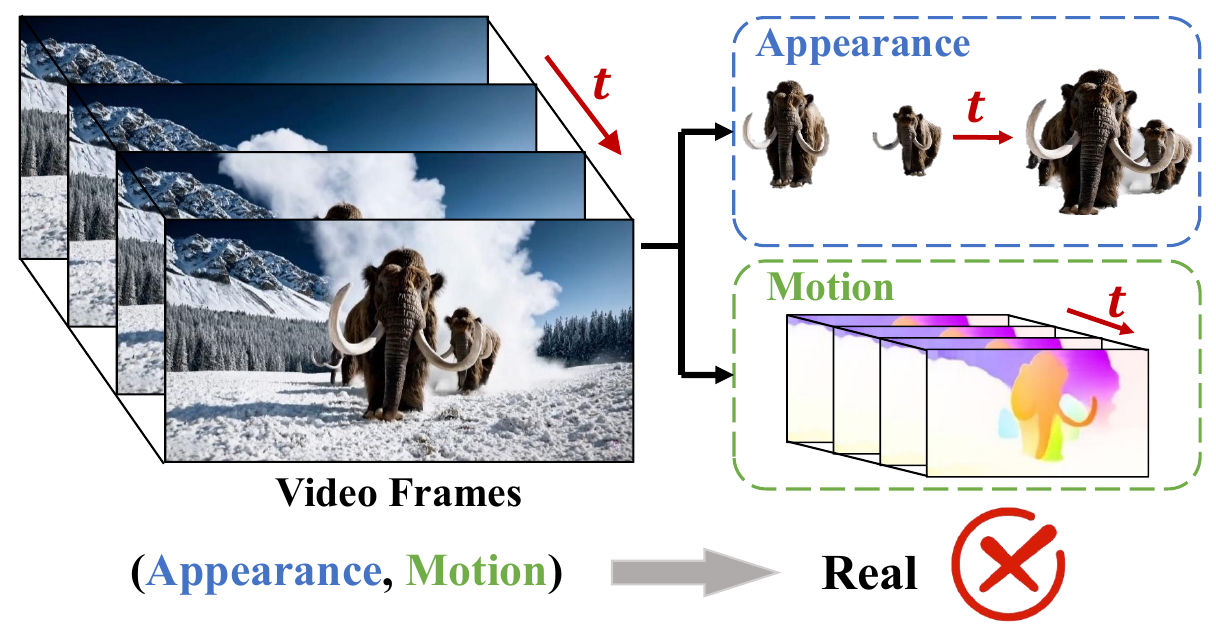}}
    \hskip 0.2in
    \subfloat[Physics-Driven Spatiotemporal Modeling (Ours)]
    {\includegraphics[width=0.47\linewidth]{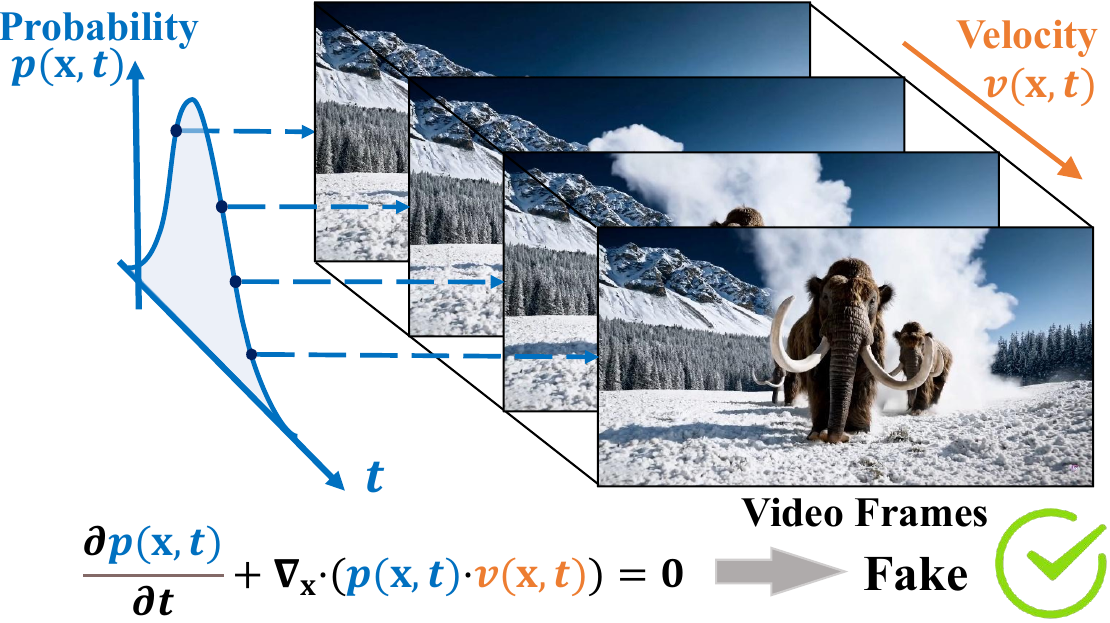}}
    \vspace{-3pt}
    \caption{Comparisons of traditional and physics-driven paradigms for spatiotemporal modeling in AI-generated video detection. (a) Traditional methods \cite{amerini2019deepfake,wang2023altfreezing,xu2023tall} often rely on specific artifacts like appearance consistency and optical flow-based motion modeling, struggling with highly realistic content yet physically implausible (\eg, Sora). (b) Our physics-driven approach explicitly models video dynamics via physics conservation laws, effectively identifying violations of physical laws.}
    \label{fig: motivation}
    \end{center}
    \vspace{-15pt}
\end{figure*}

Existing  AI-generated video detection methods primarily rely on local feature inconsistencies (\eg, optical flow-based motion modeling \cite{amerini2019deepfake},  appearance consistency modeling \cite{wang2023altfreezing}) or supervised learning with large-scale datasets \cite{chen2024demamba,xu2023tall,qian2020thinking,tan2024rethinking}. However, they often ignore physics-driven constraints governing spatiotemporal evolution inherent to natural videos. This limitation exhibits inherent vulnerabilities when confronting synthetic anomalies that violate physical laws, \eg, non-physical motion patterns in Sora-generated videos \cite{brooks2024video} (Figure \ref{fig: motivation}-a), leading to inferior performance.

In this paper, we propose the first physics-driven paradigm based on probability flow conservation principles \cite{batchelor2000introduction,rieutord2014fluid}. By modeling video dynamics as fluid mechanics, we formulate video evolution through a probability flow velocity field governed by continuity equations (see Figure \ref{fig: motivation}-b and Section \ref{sec: Modeling Spatiotemporal Dynamics}). This reveals a key insight: \textit{natural video dynamics preserve the product between the velocity field and the ratio of spatial probability gradients to temporal density changes}. Inspired by this, we introduce a \textbf{Normalized Spatiotemporal Gradient (NSG)} statistic, which quantifies the ratio of spatial probability gradients to temporal density changes. NSG captures fundamental discrepancies in how videos adhere to physical constraints while eliminating reliance on specific artifacts, enabling sensitive detection even when visual differences are imperceptible to humans or conventional models.

To enable practical estimation, we develop an NSG estimator leveraging pre-trained diffusion models' inherent gradient estimation ability \cite{song2019generative,song2020score} in Section \ref{sec: Estimating statistic}. By approximating spatial gradients with learned score functions (\ie, the gradient of the log probability density) from the diffusion models and temporal derivatives through motion-aware temporal dynamics via a brightness constancy constraint \cite{horn1981determining}, our method avoids explicit flow computation while preserving essential physical constraints. This estimator eliminates reliance on complex motion modeling by physics-inspired priors while maintaining sensitivity to subtle spatiotemporal inconsistencies inherent to synthetic content.

Building on this foundation, we propose an \textbf{NSG-based video detection method (NSG-VD)} in Section \ref{sec: Exploring statistic}, which computes the \textit{Maximum Mean Discrepancy} (MMD) \cite{gretton2012kernel, liu2020learning} between NSG features of real videos and the test video as the detection metric, as illustrated in Figure \ref{framework}. We further theoretically derive an upper bound of the distance between NSG features of real and generated data in Section \ref{sec: Theoretical Analysis}, showing this bound expands with increasing distribution shifts in generated videos. This implies that the MMD between NSGs of real videos tends to be smaller than that between real and generated videos, establishing the theoretical basis for the effectiveness of NSG-VD. 
Extensive experiments show that NSG-VD achieves $16.00\%$ higher Recall and $10.75\%$ higher F1-score than baselines, validating the superior performance of NSG-VD. 
Our contributions are summarized as:
\begin{itemize}[leftmargin=*]
    \item A physics-driven NSG statistic: We propose the first formulation of video evolution via a probability flow velocity field with a continuity equation, introducing a statistic Normalized Spatiotemporal Gradient (NSG) that explicitly models spatiotemporal dynamics of videos. By quantifying the ratio of spatial probability gradients to temporal density changes, NSG fundamentally captures violations of physical continuity in AI-generated videos without reliance on artifact-specific supervision.
    \item A diffusion-guided NSG estimation with physical priors: We develop an NSG estimator by spatial gradients approximation and motion-aware temporal dynamics modeling using pre-trained diffusion models. By avoiding explicit flow modeling and instead enforcing brightness constancy constraints, our method achieves effective NSG approximation without domain-specific motion modeling.
    \item An AI-generated video detection method with theoretical and empirical justifications: We propose an NSG-based video detection method (NSG-VD), which quantifies distributional shifts in NSG features using \textit{Maximum Mean Discrepancy} (MMD). We derive an upper bound of NSG feature distances between real and generated videos, proving that generated videos exhibit amplified discrepancies under distribution shifts. Empirical results also show the superiority of our NSG-VD.
\end{itemize}

\section{Related Work}

\textbf{AI-Generated Video Detection.} Early generated video detection methods primarily focus on identifying synthetic facial videos. Yang et al.~\cite{yang2019exposing} and Amerini et al.~\cite{amerini2019deepfake} exploit auxiliary facial motion cues (landmark dynamics vs. optical flow) for deepfake detection. Gu et al.~\cite{gu2021spatiotemporal} separately model spatial and temporal inconsistencies, and introduce a vertical slicing feature fusion mechanism to establish a more comprehensive spatial-temporal representation. Wang et al.~\cite{wang2023altfreezing} propose an alternating-freezing strategy with spatiotemporal augmentation for facial consistency modeling. Xu et al.~\cite{xu2023tall} transform consecutive frames into a predefined layout via masking/resizing to enable efficient spatiotemporal modeling. Peng et al.~\cite{peng2024deepfakes} integrate multi-feature fusion of facial perspectives, textures, and attributes. While most methods utilize facial priors, their reliance on domain-specific features limits their generalizability to more general AI-generated content detection.

With the rapid advancement in video generation, detecting general AI-generated content has become challenging. Bai et al.~\cite{bai2024ai} fuse frame-level and optical flow predictions to detect spatial-temporal anomalies. To jointly capture spatiotemporal cues, Ma et al.~\cite{ma2024decof}  and Chen et al.~\cite{chen2024demamba} propose Transformer- and mamba-based frameworks to model spatiotemporal relationships in video frame features for detection. Song et al.~\cite{songlearning} exploit the cross-modal perception and reasoning in vision-language large models to learn general forgery features. Despite this progress, these methods mainly focus on appearance inconsistencies, while overlooking the intrinsic spatiotemporal dynamics cues, thereby struggling to tackle visual cues from diverse video generative models.

\textbf{Diffusion Models.}
Diffusion models \cite{ho2020denoising, song2019generative, song2020improved, song2020score} have emerged as powerful probabilistic generative models, benefiting from their diffusion-denoising paradigm that perturb data into noise through Gaussian processes and reconstruct samples via iterative denoising.
Intuitively, the high-quality and diverse generative capabilities of diffusion models come from their ability to capture and exploit the distributional characteristics of natural data, enabling effective discrimination between natural samples and outliers. Motivated by this, a growing body of research has leveraged diffusion models for the detection of adversarial \cite{nie2022diffusion,yoon2021adversarial,zhangs2023EPSAD} and generated samples \cite{song2025detecting,wang2023dire,zhang2024detecting}, wherein the score model emerges as a powerful discriminative tool. Nevertheless, it remains challenging to simultaneously capture and integrate spatiotemporal features when relying solely on score models.

\section{Modeling Spatiotemporal Dynamics for AI-Generated Video Detection}

\textbf{AI-Generated Video Detection.} Let $\mathbb{P}$ be a Borel probability measure on a separable spatiotemporal metric space $\mathcal{X} {\subset} \mathbb{R}^{T \times d}$, where $T$ is the number of frames and $d$ is the spatial dimension. Given  independent and identically distributed (i.i.d.) samples $S_{\mathbb{P}} {=} \{\mathbf{x}^{(i)}\}_{i=1}^n$ from the real video distribution $\mathbb{P}$, we aim to determine whether each sample $\mathbf{y}^{(j)}$ in $S_{\mathbb{Q}} {=} \{\tilde{\mathbf{y}}^{(j)}\}_{j=1}^m$ originates from $\mathbb{P}$.

\textbf{Challenges for Video Detection.} The complex spatiotemporal dynamics in high-dimensional video data often require modeling both spatial irregularities (\eg, unnatural textures) and temporal inconsistencies (\eg, implausible motions). Moreover, the diversity of generative paradigms (\eg, diffusion models \cite{blattmann2023stable} and generative adversarial networks \cite{brooks2022generating}) introduces heterogeneous distribution shifts that exhibit as subtle statistical inconsistencies rather than explicit artifacts.  These challenges are further worsened by rapidly evolving generative techniques (\eg, Sora \cite{brooks2024video}), which continuously produce novel spatiotemporal patterns surpassing existing detection mechanisms. 

\textbf{Method Overview.} To address these challenges, we propose a physics-driven method based on physical conservation principles to model \textit{spatiotemporal dynamics} and introduce a novel statistic \textbf{Normalized Spatiotemporal Gradient (NSG)}, which quantifies the ratio of spatial probability gradients to temporal density changes, capturing subtle anomalies in videos (Section \ref{sec: Modeling Spatiotemporal Dynamics}). Leveraging diffusion models, we develop an effective NSG estimator by spatial gradients approximation and motion-aware temporal dynamics modeling (Section \ref{sec: Estimating statistic}). Building on this, we develop an \textbf{NSG-based video detection method (NSG-VD)}, which computes the \textit{Maximum Mean Discrepancy} (MMD) between NSG features of the test video and real videos as a detection characteristic (Section \ref{sec: Exploring statistic}), where its framework is shown in Figure \ref{framework}. Last, we theoretically show that the MMD between NSGs of real videos tends to be smaller than that between real and generated videos (Section \ref{sec: Theoretical Analysis}).

\begin{figure*}[t]
    \centering
    \includegraphics[width=\linewidth]{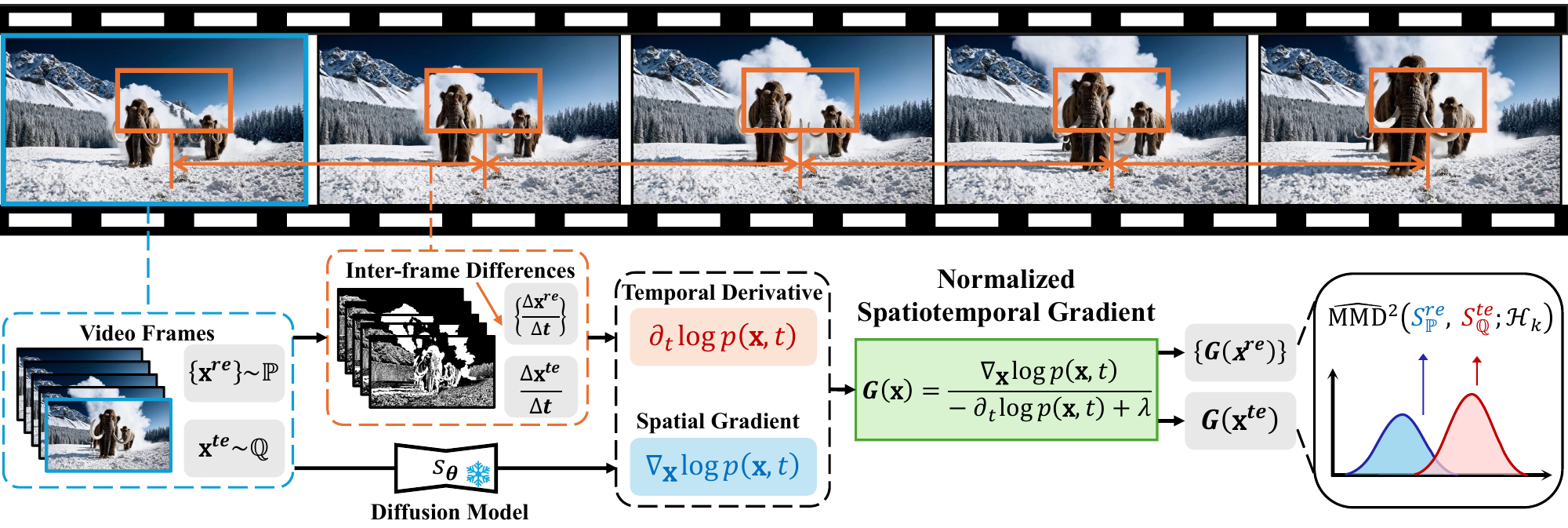}
    \vspace{-15pt}
    \caption{Overview of the proposed NSG-VD. Given a reference set of real videos $\{\mathbf{x}^{re}\}$ and a test video $\mathbf{x}^{te}$, we estimate their spatial gradients $\nabla_{\mathbf{x}} \log p(\mathbf{x}, t)$ and temporal derivatives $\partial_t \log p(\mathbf{x}, t)$ via a pre-trained diffusion model $s_\theta$, from which we derive their Normalized Spatiotemporal Gradients (NSGs) and calculate the MMD between NSG features of real and test videos as a detection metric.
    }
    \label{framework}
\end{figure*}

\subsection{Modeling Spatiotemporal Dynamics via Normalized Spatiotemporal Gradient}
\label{sec: Modeling Spatiotemporal Dynamics}

Detecting AI-generated videos requires capturing both spatial irregularities and temporal inconsistencies in synthetic content.  Inspired by conservation laws in physics (\eg, mass or energy transport), we initially formulate the probability flow velocity field $\mathbf{v}(\mathbf{x}, t)$ to model the evolution of probability density $p(\mathbf{x}, t)$, which satisfies a continuity equation for global consistency across spatiotemporal domains. However, solving $\mathbf{v}$ faces challenges due to its underdetermined nature (Eqn. (\ref{eqn: apprx_v})). To address this, we propose \textbf{Normalized Spatiotemporal Gradient (NSG)} $\mathbf{g}(\mathbf{x}, t)$, a \textit{dual} field statistic of $\mathbf{v}$ combining both spatial gradients and temporal dynamics of $p(\mathbf{x}, t)$, as defined in Eqn. (\ref{eqn:NSG}).

\textbf{Probability Flow Velocity Field $\mathbf{v}(\mathbf{x}, t)$.} We begin to conceptualize the \textit{probability flow} (also called \textit{probability current}) as the movement of probability mass of $\mathbf{x}$ over time $t$ \cite{wilczek1999quantum,synge1978tensor}. To formalize this flow, we define the \textit{probability flow density} $\mathbf{J}(\mathbf{x}, t)$, analogous to fluid mechanics \cite{hodge2014electron}:
\begin{equation}
    \mathbf{J}(\mathbf{x}, t) = p(\mathbf{x}, t) \cdot \mathbf{v}(\mathbf{x}, t),
\end{equation}
where $p(\mathbf{x}, t)$ denotes the probability density and $\mathbf{v}(\mathbf{x}, t)$ represents the velocity field guiding the flow of probability mass. The conservation of probability mass \cite{batchelor2000introduction,rieutord2014fluid} implies the continuity equation:
\begin{equation}\label{eqn: conservation equation}
    \frac{\partial p(\mathbf{x}, t)}{\partial t} + \nabla_{\mathbf{x}} \cdot \mathbf{J}(\mathbf{x}, t) = 0,
\end{equation}
where $\nabla_{\mathbf{x}} \cdot \mathbf{J}{=} \sum_{i} \frac{\partial \mathbf{J}_i}{\partial \mathbf{x}_i}$ denotes the divergence of the vector field $\mathbf{J}$ \cite{synge1978tensor}. This is not a video-specific assumption but a universal mathematical formulation of probability mass conservation, which holds for any time-evolving probability density $p(\mathbf{x}, t)$ \cite{risken1989fokker,rieutord2014fluid}.
Intuitively, this equation shows that the rate of change in probability density $\partial_t p$ at a point equals the difference between \textit{inflow} (negative divergence) or \textit{outflow} (positive divergence) of the probability flow $\mathbf{J}$. 
Substituting $\mathbf{J}(\mathbf{x}, t)$ into Eqn. (\ref{eqn: conservation equation}), dividing by $p(\mathbf{x}, t)$, and applying the chain rule to $\log p(\mathbf{x}, t)$, yields:
\begin{equation}\label{eqn: log_p_with_delta_v}
    \partial_t \log p(\mathbf{x}, t) + \nabla_{\mathbf{x}} \cdot \mathbf{v}(\mathbf{x}, t) + \mathbf{v}(\mathbf{x}, t) \cdot \nabla_{\mathbf{x}} \log p(\mathbf{x}, t) = 0.
\end{equation}
This expression reveals how the velocity field $\mathbf{v}(\mathbf{x}, t)$ simultaneously encodes temporal evolution ($\partial_t \log p(\mathbf{x}, t)$) and spatial gradients ($\nabla_{\mathbf{x}} \log p(\mathbf{x}, t)$) of the probability distribution.

\textbf{Normalized Spatiotemporal Gradient $\mathbf{g}(\mathbf{x}, t)$ as Dual Field of $\mathbf{v}(\mathbf{x}, t)$.} To solve $\mathbf{v}(\mathbf{x}, t)$, we focus on the dominant components of Eqn. (\ref{eqn: log_p_with_delta_v}). Assuming that the divergence term $\nabla_{\mathbf{x}} \cdot \mathbf{v}$ is subdominant in smoothly varying distributions (\eg, incompressible flow approximations \cite{batchelor2000introduction, panton2024incompressible}), a condition commonly used in fluid dynamics \cite{batchelor2000introduction} and quantum mechanics \cite{bohm2013quantum}, Eqn. (\ref{eqn: log_p_with_delta_v}) simplifies to:
\begin{equation}\label{eqn: apprx_v}
    \mathbf{v}(\mathbf{x}, t) \cdot \nabla_{\mathbf{x}} \log p(\mathbf{x}, t) \approx -\partial_t \log p(\mathbf{x}, t).
\end{equation}
Considering the non-uniqueness of solutions to $\mathbf{v}(\mathbf{x}, t)$ in Eqn. (\ref{eqn: apprx_v}), we normalize both sides into
\begin{equation}\label{eqn: unit}
    \mathbf{v}(\mathbf{x}, t) \cdot \frac{ \nabla_{\mathbf{x}} \log p(\mathbf{x}, t)}{-\partial_t \log p(\mathbf{x}, t)} \approx 1.
\end{equation}
\begin{deftn}(\textbf{Normalized Spatiotemporal Gradient (NSG)}.) 
The relation in Eqn. (\ref{eqn: unit}) reveals that \textit{natural video dynamics preserve the product between the velocity field and the ratio of spatial probability gradients to temporal density changes}. We formalize this constrained ratio as the Normalized Spatiotemporal Gradient (NSG), defined as:
\begin{equation}\label{eqn:NSG}
    \mathbf{g}(\mathbf{x}, t) = \frac{\nabla_{\mathbf{x}} \log p(\mathbf{x}, t)}{-\partial_t \log p(\mathbf{x}, t)+\lambda}.
\end{equation}  
\end{deftn}
Here, $\lambda >0$ prevents numerical instability. 
Eqn. (\ref{eqn: unit}) and (\ref{eqn:NSG}) imply that $\mathbf{g}(\mathbf{x}, t)$ acts as a \textit{dual field} to $\mathbf{v}(\mathbf{x}, t)$,  satisfying $\mathbf{v} \cdot \mathbf{g}\approx 1$. The formulation of $\mathbf{g}(\mathbf{x}, t)$ bypasses the ill-posed velocity $\mathbf{v}(\mathbf{x}, t)$ inversion problem while preserving the critical information about spatiotemporal gradient dynamics.

\textbf{Interpretation and Advantages}. The NSG statistic $\mathbf{g}(\mathbf{x}, t)$ quantifies the directional sensitivity of probability flow per unit temporal variation, driven by both spatial gradients ($\nabla_{\mathbf{x}} \log p(\mathbf{x}, t)$) and temporal derivatives ($\partial_t \log p(\mathbf{x}, t)$). This statistic captures both spatial irregularities (via $\nabla_{\mathbf{x}} \log p(\mathbf{x}, t)$) and temporal inconsistencies (via $\partial_t \log p(\mathbf{x}, t)$), enabling comprehensive analysis of video dynamics. Moreover, by modeling fundamental probability flow dynamics, NSG avoids dependencies on specific artifacts, making it suitable for detecting generated videos across diverse generation paradigms.

\subsection{Estimating NSG with Diffusion Models}
\label{sec: Estimating statistic}

The NSG statistic $\mathbf{g}(\mathbf{x}, t)$ in Eqn. (\ref{eqn:NSG}) requires estimating two key components: the spatial gradients $\nabla_{\mathbf{x}} \log p(\mathbf{x}, t)$ and the temporal derivatives $\partial_t \log p(\mathbf{x}, t)$.
Using diffusion models' inherent gradient estimation ability \cite{song2019generative,song2020score}, we propose an effective estimator combining spatial gradients from pre-trained diffusion models with motion-aware temporal dynamics using Eqn. (\ref{eqn: score_approx}) and (\ref{eqn: partial_t_approx}), yielding: 
\begin{equation}\label{eqn: NSG_est}
    \mathbf{g}(\mathbf{x}, t) \approx \frac{\mathbf{s}_{\theta}(\mathbf{x}_t)}{ \mathbf{s}_{\theta}(\mathbf{x}_t) \cdot \frac{\mathbf{x}_{t+\Delta t} - \mathbf{x}_t}{\Delta t} + \lambda },
\end{equation}
where $\mathbf{s}_{\theta}$ denotes the learned score function from diffusion models and $\mathbf{x}_t$ represents the $t$-th video frame. This estimator eliminates the need for explicit flow computation while preserving critical spatiotemporal dynamics through physics-inspired modeling. Below, we detail its derivation.

\textbf{Spatial Gradients Estimation.}
Diffusion models \cite{song2020score,dhariwal2021diffusion} explicitly learn a score network $\mathbf{s}_{\theta}$ through score matching \cite{song2019generative} or denoising diffusion modeling \cite{ho2020denoising} to approximate $\nabla_{\mathbf{x}} \log p(\mathbf{x}, t)$. For a given video $\mathbf{x}$ at $t$-th frame, the spatial gradient is estimated by:  
\begin{equation}\label{eqn: score_approx}
    \nabla_{\mathbf{x}} \log p(\mathbf{x}, t) \approx \mathbf{s}_{\theta}(\mathbf{x}_t),
\end{equation}
where $\mathbf{x}_t$ is the $t$-th frame of the video $\mathbf{x}$.
Here, we omit the diffusion timestep to make the notation clearer. This estimation allows direct computation of $\nabla_{\mathbf{x}} \log p(\mathbf{x}, t)$ in NSG via a single forward pass of the pre-trained diffusion model, eliminating the need for numerical differentiation.

\textbf{Temporal Derivatives Approximation.} To estimate \(\partial_t \log p(\mathbf{x}, t)\), we exploit the temporal coherence of video sequences under the \textit{brightness constancy assumption} \cite{horn1981determining}, which posits that the probability density along motion trajectories remains constant. This leads to the following approximation:  
\begin{prop}\label{prop:partial_t_approx}  
Under the brightness constancy assumption $p(\mathbf{x}+\Delta\mathbf{x}, t+\Delta t) \approx  p(\mathbf{x}, t)$ with small inter-frame motion ($\Delta t \to 0$) and inter-frame displacement ($\Delta \mathbf{x} \to 0$), we have  
\begin{equation}\label{eqn: partial_t_approx}  
\partial_t \log p(\mathbf{x}, t) \approx -\frac{\nabla_{\mathbf{x}} \log p(\mathbf{x}, t) \cdot \Delta\mathbf{x}}{\Delta t}.
\end{equation} 
\end{prop}

\subsection{Exploring NSG for Detecting AI-Generated Videos}
\label{sec: Exploring statistic}

To effectively distinguish AI-generated content from real videos, it is crucial to design metrics that capture subtle distributional discrepancies in high-dimensional spatiotemporal features. Recent studies show \textit{Maximum Mean Discrepancy} (MMD) \cite{gretton2012kernel}—a non-parametric statistic for distribution alignment—has demonstrated remarkable capabilities in measuring distributional differences, \eg, AI-text detection \cite{zhang2024detecting,song2025detecting} and adversarial samples detection \cite{gao2021maximum,zhang2023detecting}. Building upon MMD's theoretical foundation and NSG's unique strength in modeling spatiotemporal dynamics, we propose an \textbf{NSG-based video detection method (NSG-VD)}  that integrates MMD with the NSG feature representation. 

\textbf{MMD Formulation with NSG Features.} We aggregate NSG features across $T$ frames in each video as $\mathbf{G}(\mathbf{x}){=}\{\mathbf{g}(\mathbf{x},t)\}_{t=1}^T$. Let $S^{re}_\mathbb{P} {=} \{\mathbf{x}^{(i)}\}_{i=1}^n$ denote a reference set of real videos and $S^{te}_\mathbb{Q} {=} \{\tilde{\mathbf{y}}\}$ represent a test video. The MMD \cite{gretton2012kernel} between $S^{re}_\mathbb{P}$ and $S^{te}_\mathbb{Q}$ in terms of NSG is computed as:
\begin{equation}\label{eqn: MMD_NSG}
\widehat{\mathrm{MMD}}_{b}^{2}\!\left[S^{re}_\mathbb{P},S^{te}_\mathbb{Q};\mathcal{H}_{k}\right] {=} \frac{1}{n^{2}} \!\!\sum_{i,j=1}^{n}\! k\!\left(\mathbf{G}^{(i)}, \mathbf{G}^{(j)}\right) {-} \frac{2}{n} \!\sum_{i=1}^{n} k\!\left(\mathbf{G}^{(i)}, \mathbf{G}^{(\text{test})}\!\right) {+} k\!\left(\mathbf{G}^{(\text{test})}, \mathbf{G}^{(\text{test})}\!\right)\!,
\end{equation}
where \(\mathbf{G}^{(i)} = \mathbf{G}(\mathbf{x}^{(i)})\) and \(\mathbf{G}^{(\text{test})} = \mathbf{G}(\tilde{\mathbf{y}})\) are NSG features extracted from real and test videos. The kernel \(k: \mathcal{G} \times \mathcal{G} \to \mathbb{R}\) maps NSG features to a reproducing kernel Hilbert space (RKHS) \(\mathcal{H}_k\), such as the Gaussian kernel $k \left(\mathbf{a},\mathbf{b}\right)=\exp \left(-\left\|\mathbf{a}-\mathbf{b}\right\|^{2} /\left(2 \sigma^{2}\right)\right)$.
Note that while MMD is conventionally used for distribution-level comparisons, recent studies \cite{zhang2024detecting,song2025detecting,zhang2023detecting} validate its efficacy in single-sample detection by quantifying deviations from reference distributions. 
Crucially, while MMD provides a viable solution for distributional comparison, the core advantage of NSG-VD stems from the NSG itself modeling fundamental spatiotemporal dynamics (see details in Appendix \ref{sec: Impact of MMD}).

\textbf{Detection Protocol with MMD Metric.} Let $f(\tilde{\mathbf{y}}; S_{\mathbb{P}}, k_\omega, \tau)=\mathbb{I}\left(\widehat{\mathrm{MMD}}_{b}^{2} > \tau\right)$, where $\mathbb{I}$ is the indicator function and $\tau$ is a threshold for the decision. Given a test video $\tilde{\mathbf{y}}$, we compute the MMD with NSG against a referenced real video set and give the decision:
\begin{equation}
\label{eqn: decision}
    f(\tilde{\mathbf{y}}) = \begin{cases} 
\mathrm{Fake}, & \mathrm{if}~ f(\tilde{\mathbf{y}}; S_{\mathbb{P}}, k_\omega,\tau)=1, \\
\mathrm{Real}, & \mathrm{if}~ f(\tilde{\mathbf{y}}; S_{\mathbb{P}},k_\omega, \tau)=0.
\end{cases}
\end{equation}
\textbf{Optimization for NSG-VD.} To enhance discriminative power, we use a deep kernel \cite{liu2020learning} for MMD:
\begin{equation}\label{eqn:deep_kernel}
    k_\omega(\mathbf{x},\mathbf{y}) = \left[(1-\epsilon)\kappa\left(\phi_{\mathbf{G}}(\mathbf{x}), \phi_{\mathbf{G}}(\mathbf{y})\right) + \epsilon\right] \cdot \Phi\left(\mathbf{G}(\mathbf{x}), \mathbf{G}(\mathbf{y})\right),
\end{equation}
where \(\phi_{\mathbf{G}}(\mathbf{x}) = \phi(\mathbf{G}(\mathbf{x}))\) is a deep neural network,
\(\kappa\) and \(\Phi\) are Gaussian kernels with bandwidths \(\sigma_\phi\) and \(\sigma_{\Phi}\), and \(\epsilon \in (0, 1)\).
The kernel parameters \(\omega {=} \{\epsilon, \phi, \sigma_\phi, \sigma_{\Phi}\}\) will be optimized by Eqn. (\ref{eqn:objMMDopt}) to maximize the detection ability.
Considering the multiple-population scenarios across diverse video distributions \cite{zhang2024detecting}, we adopt a \textit{multi-population aware optimization} for the kernel training:
\begin{equation}\label{eqn:objMMDopt}
    k_\omega^* {=} \arg\max_{k_\omega} \frac{\widehat{\mathrm{MPP}}_{u}(S^{tr}_{\mathbb{P}},S^{tr}_{\mathbb{Q}};k_\omega)}{\sqrt{\hat{\sigma}^{2}(S^{tr}_{\mathbb{P}},S^{tr}_{\mathbb{Q}};k_\omega)+\lambda}},
    ~\hat{\sigma}^{2} {=} \frac{4}{N^{3}}\sum_{i = 1}^{N}\left(\sum_{j = 1}^{N}H_{ij}^{*}\right)^{2} \!\!{-} \frac{4}{N^{4}}\left(\sum_{i = 1}^{N}\sum_{j = 1}^{N}H_{ij}^{*}\right)^{2},
\end{equation}  
where $S^{tr}_{\mathbb{P}}$ and $S^{tr}_{\mathbb{Q}}$ denote the training real and generated videos, respectively, $\widehat{\mathrm{MPP}}_{u}(S^{tr}_{\mathbb{P}},S^{tr}_{\mathbb{Q}};k_\omega)=\frac{1}{N(N-1)}\sum_{i\neq j}H^*_{ij}$ and $H_{ij}^*{=}k_\omega(\mathbf{x}_i,\mathbf{x}_j){-}k_\omega(\mathbf{x}_i,\mathbf{y}_j){-}k_\omega(\mathbf{y}_i,\mathbf{x}_j)$.

\subsection{Theoretical Guarantees for NSG-VD}
\label{sec: Theoretical Analysis}

The effectiveness of NSG-VD relies on ensuring the MMD between NSG features of real videos is smaller than that between real and generated videos. To formalize this, we analyze the MMD formulation in Eqn. (\ref{eqn: MMD_NSG}), where the key discriminative information lies in the cross-term $k\left(\mathbf{G}^{(i)}, \mathbf{G}^{(\text{test})}\right)$ since the first and third terms remain invariant for fixed reference sets.
Under the Gaussian kernel, this cross-term is dominated by the exponential squared distance between NSG features. Note that analyzing $\frac{\nabla_{\mathbf{x}} \log p(\mathbf{x}, t)}{-\partial_t \log p(\mathbf{x}, t)+\lambda}$ under practical distributions can be very difficult and infeasible, 
we adopt a common practice \cite{zheng2023toward,wang2024embedding} by assuming Gaussian-distributed data to derive theoretical insights.
Below, we first characterize the NSG statistics for real and generated videos under Gaussian assumptions.
\begin{prop}\label{prop: NSG}
Let the real video distribution be $p(\mathbf{x}, t)=\mathcal{N}(\mathbf{0}, \sigma(t)^2 \mathbf{I}_d)$ and the generated video distribution be $q(\mathbf{y}, t){=}\mathcal{N}(\boldsymbol{\mu}, \sigma(t)^2 \mathbf{I}_d)$, respectively, where $\mathbf{I}_d \in \mathbb{R}^{d\times d}$ is an identity matrix and $\boldsymbol{\mu}{\neq} \mathbf{0} \in \mathbb{R}^d$ is the distribution shift and $\sigma(t){\neq} \mathbf{0}$, the NSG $\mathbf{g}(\mathbf{x}, t)$ and $\mathbf{g}(\mathbf{y}, t)$ satisfy:
\begin{align*}
   \mathbf{g}(\mathbf{x}, t) &= -\frac{\mathbf{x}/\sigma(t)^2}{ D_r(\mathbf{x})}, ~~-\frac{\mathbf{x} }{\sigma(t)^2} \sim \mathcal{N} \Big(\mathbf{0}, \sigma(t)^2 \mathbf{I}_d \Big), ~D_r(\mathbf{x}) \sim \lambda + \frac{d \dot{\sigma}(t)}{\sigma(t)} - \frac{\dot{\sigma}(t)}{\sigma(t)} \chi^2(d); \\
   \mathbf{g}(\mathbf{y}, t) &= -\frac{\mathbf{y}/\sigma(t)^2}{ D_f(\mathbf{y})},~-\frac{\mathbf{y} }{\sigma(t)^2} \sim \mathcal{N}\Big(-\frac{\boldsymbol{\mu}}{\sigma(t)}, \sigma(t)^2 \mathbf{I}_d \Big), ~D_f(\mathbf{y}) \sim \lambda +\frac{d \dot{\sigma}(t)}{\sigma(t)} - \frac{\dot{\sigma}(t)}{\sigma(t)} \chi^2(d, \varphi),
\end{align*}
where $D_r(\mathbf{x}) = \lambda +\frac{d \dot{\sigma}(t)}{\sigma(t)} - \frac{\|\mathbf{x}\|^2 \dot{\sigma}(t)}{\sigma(t)^3}$, $D_f(\mathbf{y}) = \lambda + \frac{d \dot{\sigma}(t)}{\sigma(t)} - \frac{\|\mathbf{y}\|^2 \dot{\sigma}(t)}{\sigma(t)^3}$, $\dot{\sigma}(t) \triangleq \frac{d}{dt}\sigma(t)$,  and $\varphi = \frac{\|\boldsymbol{\mu}\|^2}{\sigma(t)^2}$, $\chi^2(d)$ is the central chi-squared distribution with $d$ degrees of freedom and $\chi^2(d, \varphi)$ is the noncentral chi-squared distribution with noncentrality parameter $\varphi$ and $d$ degrees of freedom \cite{abdel1954approximate}.
\end{prop}
 Proposition \ref{prop: NSG} reveals that the distribution shift \(\boldsymbol{\mu}\) in generated videos introduces deviations in both the numerator and denominator of the NSG, \ie, \textit{noncentral} Gaussian and chi-squared distributions. To quantify this deviation, we derive an upper bound on the squared distance between NSGs:
\begin{thm}\label{thm: bound of NSG}
    Let the real video distribution be  $\mathbf{x} {\sim} \mathcal{N}(\mathbf{0}, \sigma(t)^2 \mathbf{I}_d)$ and the generated video distribution be $\mathbf{y} {\sim} \mathcal{N}(\boldsymbol{\mu}, \sigma(t)^2 \mathbf{I}_d) $, respectively, where $\mathbf{I}_d {\in} \mathbb{R}^{d\times d}$ is an identity matrix and $\boldsymbol{\mu} {\neq} \mathbf{0} {\in} \mathbb{R}^d$ is the distribution shift. Given $\mathbf{G}(\mathbf{x}){=}\{\mathbf{g}(\mathbf{x},t)\}_{t=1}^T$,  denote $\varphi {=} {\|\boldsymbol{\mu}\|^2}/{\sigma(t)^2}$ and assume $\left|-\partial_t \log p(\mathbf{x}, t)+\lambda \right|\geq C>0$ and $\left|-\partial_t \log p(\mathbf{y}, t)+\lambda \right|\geq C>0$, with probability at least $1 - \delta$, we have
\begin{equation*}
    \left\| \mathbf{G}(\mathbf{x}) {-} \mathbf{G}(\mathbf{y}) \right\|^2 {\leq} \mathcal{O} \left(\frac{T}{C^4 \sigma(t)^2} \! \left[\varphi d +d^2 + \varphi + \log \frac{T}{\delta} \! \cdot \! (\varphi + d) + \log^2 \frac{T}{\delta}\right]\right).
\end{equation*}
\end{thm}
Theorem \ref{thm: bound of NSG} reveals that the bound of the squared distance between NSG features of real and fake data will be smaller if the distribution shift term $\varphi = {\|\boldsymbol{\mu}\|^2}/{\sigma(t)^2}$ is closer to zero for a given $\delta$. This formalizes the intuition that small distribution shifts produce small geometric distortions in NSG space, while significant deviations in synthetic content lead to large separations from real data. Under the Gaussian kernel, this implies that the real data have a larger $k(\mathbf{G}(\mathbf{x}), \mathbf{G}(\mathbf{y}))$ than the fake data since the distribution shift term $\varphi=0$ for real data. Therefore, when substituted into Eqn. (\ref{eqn: MMD_NSG}), the MMD between NSG features of real videos is smaller than that between real and generated videos.

\section{Experiments}
\label{sec: experiments}

\textbf{Datasets.} We evaluate our methods on the GenVideo benchmark \cite{chen2024demamba}, a large-scale dataset for AI-generated video detection that includes diverse real-world videos and synthetic content from multiple generative models. We use Kinetics-400 \cite{kay2017kinetics} as the real video source, SEINE \cite{chen2023seine} or Pika \cite{wang2024pika} as the AI-generated videos for training. The test set comprises MSR-VTT \cite{xu2016msr} and 10 diverse AI-generated datasets from different generation paradigms. More details are in Appendix \ref{sec:details on dataset}.

\textbf{Evaluation Metrics.} We evaluate the performance of video detection on Recall, Accuracy, F1-score \cite{chinchor1993muc} and AUROC \cite{huang2005using} metrics. More details are provided in Appendix~\ref{sec:details on metric}. We use \textbf{bold} numbers to indicate the best results and \underline{underlined} numbers to denote the second-best results in tables.

\textbf{Baselines}. We compare our NSG-VD with following baselines: TALL~\cite{xu2023tall}, NPR~\cite{tan2024rethinking}, STIL~\cite{gu2021spatiotemporal}, and Demamba~\cite{chen2024demamba}. 
These baselines are implemented based on the codebase provided by Demamba~\cite{chen2024demamba}.

\subsection{Comparisons on Standard Evaluation}

We start by comparing our NSG-VD with baselines using $10,000$ real videos from Kinetics-400 and $10,000$ generated videos from Pika (Table \ref{tab: standard Pika}) and SENIE (Table \ref{tab: standard SENIE}) for training, respectively.

\textbf{Results on Trained with Kinetics-400 and Pika.}  From Table \ref{tab: standard Pika}, existing methods exhibit critical limitations. For instance, Demamba struggles with generative paradigms like HotShot ($40.60\%$ Recall) and Sora ($48.21\%$ Recall), while NPR shows unstable performance with Accuracy ranging from $57.20\%$ to $98.20\%$. TALL fails on synthetic outliers (\eg, $25.00\%$ Recall on Sora) and STIL collapses completely on critical cases (\eg, $1.40\%$ Recall on HotShot and $1.79\%$ Recall on Sora), revealing limitations of their inherent dependencies on generator-specific artifacts.

In contrast, our NSG-VD achieves state-of-the-art performance across all metrics, significantly outperforming baselines despite not being pre-trained on large-scale videos. Remarkably, NSG-VD demonstrates exceptional reliability on challenging closed-source generators like Sora ($78.57\%$ Recall vs. $48.21\%$ for Demamba) and emerging paradigms like HotShot ($92.50\%$ Recall vs. $40.60\%$ for Demamba), and maintains reliability across other diverse domains (\eg, MorphStudio, MoonValley). Notably, our NSG-VD achieves $16.00\% \uparrow$ average Recall and $10.75\% \uparrow$ F1-score over Demamba, and $55.05\% \uparrow$ F1-score over STIL. These results confirm its generalization across both open-source and closed-source generated models, highlighting the advantages of physics-driven modeling.  

\begin{table}[t]
    \vspace{-23pt}
    \caption{Comparisons with baselines on \textit{a standard evaluation} (\%), where we train all models with $10, 000$ real and generated videos from Kinetics-400 and Pika, respectively.}
    \label{tab: standard Pika}
    \begin{threeparttable}
    \resizebox{1.0\linewidth}{!}{
    \begin{tabular}{c|c|cccccccccc|c}
    \toprule
       \multirow{2}{*}{Method} & \multirow{2}{*}{Metric} & Model & Morph & Moon & \multirow{2}{*}{HotShot} & \multirow{2}{*}{Show1} & \multirow{2}{*}{Gen2} & \multirow{2}{*}{Crafter} & \multirow{2}{*}{Lavie} & \multirow{2}{*}{Sora} & Wild & \multirow{2}{*}{ Avg.} \\
       &  & Scope & Studio & Valley &  & & & & & & Scrape &  \\ \midrule
        \multirow{4}{*}{DeMamba} & Recall & 87.00 & 93.60 & 98.80 & 40.60 & 48.40 & 98.00 & 88.40 & 59.00 & 48.21 & 58.20 & \cellcolor{blue!8}\underline{72.02} \\
         & Accuracy & 91.70 & 95.00 & 97.60 & 68.50 & 72.40 & 97.20 & 92.40 & 77.70 & 72.32 & 77.30 & \cellcolor{blue!8}\underline{84.21} \\
         & F1 & 91.29 & 94.93 & 97.63 & 56.31 & 63.68 & 97.22 & 92.08 & 72.57 & 63.53 & 71.94 & \cellcolor{blue!8}\underline{80.12} \\
         & AUROC & 98.04 & 98.82 & 99.68 & 87.84 & 90.12 & 99.46 & 97.81 & 91.32 & 88.36 & 87.38 & \cellcolor{blue!8}93.88 \\
        \midrule
        \multirow{4}{*}{NPR} & Recall & 61.20 & 80.00 & 98.00 & 16.00 & 33.00 & 91.20 & 80.60 & 34.60 & 35.71 & 43.20 & \cellcolor{blue!8}57.35 \\
         & Accuracy & 79.80 & 89.20 & 98.20 & 57.20 & 65.70 & 94.80 & 89.50 & 66.50 & 67.86 & 70.80 & \cellcolor{blue!8}77.96 \\
         & F1 & 75.18 & 88.11 & 98.20 & 27.21 & 49.03 & 94.61 & 88.47 & 50.81 & 52.63 & 59.67 & \cellcolor{blue!8}68.39 \\
         & AUROC & 93.05 & 97.18 & 99.66 & 82.97 & 90.50 & 99.13 & 97.87 & 87.54 & 90.47 & 91.84 & \cellcolor{blue!8}93.02 \\
        \midrule
        \multirow{4}{*}{TALL} & Recall & 51.20 & 65.20 & 93.40 & 32.00 & 61.60 & 94.80 & 81.80 & 49.20 & 25.00 & 53.60 & \cellcolor{blue!8}60.78 \\
         & Accuracy & 75.10 & 82.10 & 96.20 & 65.50 & 80.30 & 96.90 & 90.40 & 74.10 & 61.61 & 76.30 & \cellcolor{blue!8}79.85 \\
         & F1 & 67.28 & 78.46 & 96.09 & 48.12 & 75.77 & 96.83 & 89.50 & 65.51 & 39.44 & 69.34 & \cellcolor{blue!8}72.63 \\
         & AUROC & 95.82 & 97.14 & 99.73 & 92.55 & 97.36 & 99.79 & 99.09 & 94.84 & 86.67 & 93.75 & \cellcolor{blue!8}\underline{95.67} \\
        \midrule
        \multirow{4}{*}{STIL} & Recall & 73.80 & 70.80 & 43.40 & 1.40 & 2.00 & 45.00 & 13.20 & 7.20 & 1.79 & 11.60 & \cellcolor{blue!8}27.02 \\
         & Accuracy & 86.90 & 85.40 & 71.70 & 50.70 & 51.00 & 72.50 & 56.60 & 53.60 & 50.89 & 55.80 & \cellcolor{blue!8}63.51 \\
         & F1 & 84.93 & 82.90 & 60.53 & 2.76 & 3.92 & 62.07 & 23.32 & 13.43 & 3.51 & 20.79 & \cellcolor{blue!8}35.82 \\
         & AUROC & 96.43 & 97.77 & 99.34 & 86.66 & 90.56 & 98.88 & 97.04 & 88.16 & 92.57 & 87.52 & \cellcolor{blue!8}93.49 \\
        \midrule
        \multirow{4}{*}{\shortstack{NSG-VD\\(Ours)}} & Recall & 68.33 & 98.33 & 100.00 & 92.50 & 87.50 & 80.00 & 98.33 & 94.17 & 78.57 & 82.50 & \cellcolor{pink!30}\textbf{88.02} \\
         & Accuracy & 81.67 & 98.33 & 96.67 & 91.67 & 90.83 & 88.33 & 95.83 & 94.17 & 88.39 & 88.75 & \cellcolor{pink!30}\textbf{91.46} \\
         & F1 & 78.85 & 98.33 & 96.77 & 91.74 & 90.52 & 87.27 & 95.93 & 94.17 & 87.13 & 88.00 & \cellcolor{pink!30}\textbf{90.87} \\
         & AUROC & 92.26 & 98.66 & 98.15 & 94.45 & 96.38 & 94.83 & 98.16 & 97.41 & 96.40 & 94.73 & \cellcolor{pink!30}\textbf{96.14} \\
    \bottomrule
    \end{tabular}
    }
    \end{threeparttable}
    \vspace{-11pt}
\end{table}

\begin{table}[t]
    \caption{Comparisons with baselines on \textit{a standard evaluation} (\%), where we train all models with $10, 000$ real and generated videos from Kinetics-400 and SEINE, respectively.}
    \label{tab: standard SENIE}
    \begin{threeparttable}
    \resizebox{1.0\linewidth}{!}{
    \begin{tabular}{c|c|cccccccccc|c}
     \toprule
       \multirow{2}{*}{Method} & \multirow{2}{*}{Metric} & Model & Morph & Moon & \multirow{2}{*}{HotShot} & \multirow{2}{*}{Show1} & \multirow{2}{*}{Gen2} & \multirow{2}{*}{Crafter} & \multirow{2}{*}{Lavie} & \multirow{2}{*}{Sora} & Wild & \multirow{2}{*}{ Avg.} \\
       &  & Scope & Studio & Valley &  & & & & & & Scrape &  \\ \midrule
        \multirow{4}{*}{DeMamba} & Recall & 47.40 & 87.80 & 88.20 & 77.40 & 75.00 & 85.60 & 91.60 & 68.60 & 42.86 & 48.00 & \cellcolor{blue!8}\underline{71.25} \\
         & Accuracy & 72.80 & 93.00 & 93.20 & 87.80 & 86.60 & 91.90 & 94.90 & 83.40 & 68.75 & 73.10 & \cellcolor{blue!8}\underline{84.54} \\
         & F1 & 63.54 & 92.62 & 92.84 & 86.38 & 84.84 & 91.36 & 94.73 & 80.52 & 57.83 & 64.09 & \cellcolor{blue!8}\underline{80.87} \\
         & AUROC & 88.29 & 98.39 & 98.76 & 97.84 & 96.89 & 98.76 & 99.35 & 96.87 & 80.93 & 88.11 & \cellcolor{blue!8}94.42 \\
        \midrule
        \multirow{4}{*}{NPR} & Recall & 46.40 & 76.40 & 69.80 & 63.80 & 56.00 & 75.00 & 83.80 & 58.80 & 35.71 & 27.40 & \cellcolor{blue!8}59.31 \\
         & Accuracy & 71.40 & 86.40 & 83.10 & 80.10 & 76.20 & 85.70 & 90.10 & 77.60 & 66.96 & 61.90 & \cellcolor{blue!8}77.95 \\
         & F1 & 61.87 & 84.89 & 80.51 & 76.22 & 70.18 & 83.99 & 89.43 & 72.41 & 51.95 & 41.83 & \cellcolor{blue!8}71.33 \\
         & AUROC & 85.73 & 96.01 & 93.79 & 91.44 & 89.96 & 95.13 & 96.87 & 89.46 & 84.15 & 76.66 & \cellcolor{blue!8}89.92 \\
        \midrule
        \multirow{4}{*}{TALL} & Recall & 58.60 & 75.00 & 79.40 & 60.20 & 62.00 & 77.80 & 88.20 & 43.80 & 33.93 & 35.80 & \cellcolor{blue!8}61.47 \\
         & Accuracy & 78.80 & 87.00 & 89.20 & 79.60 & 80.50 & 88.40 & 93.60 & 71.40 & 66.07 & 67.40 & \cellcolor{blue!8}80.20 \\
         & F1 & 73.43 & 85.23 & 88.03 & 74.69 & 76.07 & 87.02 & 93.23 & 60.50 & 50.00 & 52.34 & \cellcolor{blue!8}74.05 \\
         & AUROC & 97.10 & 98.12 & 98.63 & 96.37 & 96.45 & 97.76 & 99.38 & 94.80 & 83.35 & 89.45 & \cellcolor{blue!8}\underline{95.14} \\
        \midrule
        \multirow{4}{*}{STIL} & Recall & 28.60 & 57.40 & 78.40 & 46.80 & 18.80 & 66.40 & 69.00 & 24.80 & 14.29 & 19.00 & \cellcolor{blue!8}42.35 \\
         & Accuracy & 64.20 & 78.60 & 89.10 & 73.30 & 59.30 & 83.10 & 84.40 & 62.30 & 57.14 & 59.40 & \cellcolor{blue!8}71.08 \\
         & F1 & 44.41 & 72.84 & 87.79 & 63.67 & 31.60 & 79.71 & 81.56 & 39.68 & 25.00 & 31.88 & \cellcolor{blue!8}55.81 \\
         & AUROC & 95.53 & 97.91 & 99.40 & 96.49 & 92.79 & 98.06 & 98.86 & 91.00 & 92.79 & 86.58 & \cellcolor{blue!8}94.94 \\
        \midrule
        \multirow{4}{*}{\shortstack{NSG-VD\\(Ours)}} & Recall & 91.67 & 100.00 & 100.00 & 100.00 & 100.00 & 98.33 & 100.00 & 97.50 & 94.64 & 89.17 & \cellcolor{pink!30}\textbf{97.13} \\
         & Accuracy & 82.50 & 88.33 & 89.58 & 84.58 & 86.25 & 87.08 & 86.67 & 87.92 & 89.29 & 78.33 & \cellcolor{pink!30}\textbf{86.05} \\
         & F1 & 83.97 & 89.55 & 90.57 & 86.64 & 87.91 & 88.39 & 88.24 & 88.97 & 89.83 & 80.45 & \cellcolor{pink!30}\textbf{87.45} \\
         & AUROC & 90.67 & 97.62 & 98.38 & 95.88 & 96.69 & 97.87 & 97.64 & 95.09 & 96.14 & 88.65 & \cellcolor{pink!30}\textbf{95.46} \\
    \bottomrule
    \end{tabular}
    }
    \end{threeparttable}
\vspace{-11pt}
\end{table}

\textbf{Results on Trained with Kinetics-400 and SENIE.} As shown in Table \ref{tab: standard SENIE}, our NSG-VD achieves superior detection performance across all metrics compared to baselines. 
Notably, it attains near-perfect Recall ($\geq 98.33\%$) on models like MoonValley, HotShot and Show1, while maintaining balanced performance across diverse domains (\eg, ModelScope, WildScrape). 
These results are consistent with the results on Pika in Table \ref{tab: standard Pika}, further demonstrating the effectiveness of our proposed method. In contrast, existing baselines exhibit pronounced limitations under this setting. Demamba’s performance is more constrained ($\leq 85.60\%$ Recall on most models), and NPR’s F1-score varies widely ($41.83\% \sim 89.43\%$). TALL shows instability on models like Sora ($33.93\%$ Recall), while STIL fails entirely on critical cases (\eg, $19.00\%$ Recall on WildScrape). These failures highlight the fragility of artifact-based approaches in capturing subtle spatiotemporal inconsistencies.

Quantitatively, NSG-VD surpasses Demamba by $25.88\% \uparrow$ in average Recall ($97.13\%$ vs. $71.25\%$) and NPR by $16.12\% \uparrow$ in average F1-score ($87.45\%$ vs. $71.33\%$). On closed-source models like Sora, it achieves $94.64\%$ Recall—nearly twice Demamba's ($42.86\%$) and sextuple STIL’s ($14.29\%$). This improvement highlights NSG-VD’s sensitivity to synthetic anomalies, especially in near-photorealistic videos (\eg, Sora), where subtle spatiotemporal inconsistencies are amplified by the NSG but not effectively captured by baselines, indicating reliable detection across diverse generation paradigms.

\subsection{Comparisons on Challenging Data-Imbalanced Scenarios}

In real-world scenarios, natural videos are often abundant and accessible, while collecting sufficient AI-generated videos remains challenging due to rapidly evolving generation techniques. To thoroughly assess reliability under these conditions, we train all models using $10,000$ Kinetics-400 real videos and only $1,000$ SENIE-generated videos.  
As shown in Table \ref{tab: unbalanced Pika}, all baselines exhibit significant limitations. Demamba fails catastrophically on challenging generators like Sora ($33.93\%$ Recall) and WildScrape ($43.20\%$ Recall), while NPR exhibits fluctuations in Accuracy  ($55.36\% \sim 83.00\%$). TALL fails completely on emerging paradigms like WildScrape ($18.20\%$ Recall) and Lavie ($22.60\%$ Recall), and STIL shows highly variable performance, \eg, $25.00\% \sim 70.00\%$ Recall. Such instability indicates over-reliance on synthetic data volume or sensitivity to superficial artifacts.  

In contrast, NSG-VD achieves strong generalization across 10 diverse generations. Notably, NSG-VD attains superior performance on critical test cases: $82.14\%$ Recall on Sora (vs. $10.71\% \sim 33.93\%$ for baselines) and $81.67\%$ Recall on WildScrape (vs. $12.20\% \sim 43.20\%$). Critically, NSG-VD achieves $29.12\% \uparrow$ higher average Recall than Demamba and $38.08\% \uparrow$ higher F1-score than TALL. These results confirm NSG-VD's reliable generalization from limited synthetic data without compromising discriminative power, demonstrating that adherence to universal physical principles outperforms domain-specific feature reliance even when synthetic training data is severely constrained.

\begin{table}[t]
\vspace{-16pt}
    \caption{Comparisons with baselines under \textit{data-imbalanced scenarios} (\%), where we train all models with $10, 000$ real and $1,000$ generated videos from Kinetics-400 and SEINE, respectively.}
    \label{tab: unbalanced Pika}
    \begin{threeparttable}
    \resizebox{1.0\linewidth}{!}{
    \begin{tabular}{c|c|cccccccccc|c}
\toprule
       \multirow{2}{*}{Method} & \multirow{2}{*}{Metric} & Model & Morph & Moon & \multirow{2}{*}{HotShot} & \multirow{2}{*}{Show1} & \multirow{2}{*}{Gen2} & \multirow{2}{*}{Crafter} & \multirow{2}{*}{Lavie} & \multirow{2}{*}{Sora} & Wild & \multirow{2}{*}{ Avg.} \\
       &  & Scope & Studio & Valley &  & & & & & & Scrape &  \\ \midrule
        \multirow{4}{*}{DeMamba} & Recall & 56.80 & 80.40 & 82.60 & 65.60 & 63.80 & 78.20 & 83.00 & 53.40 & 33.93 & 43.20 & \cellcolor{blue!8}\underline{64.09} \\
         & Accuracy & 78.10 & 89.90 & 91.00 & 82.50 & 81.60 & 88.80 & 91.20 & 76.40 & 65.18 & 71.30 & \cellcolor{blue!8}\underline{81.60} \\
         & F1 & 72.17 & 88.84 & 90.17 & 78.94 & 77.62 & 87.47 & 90.41 & 69.35 & 49.35 & 60.08 & \cellcolor{blue!8}\underline{76.44} \\
         & AUROC & 93.01 & 98.17 & 98.90 & 96.42 & 95.36 & 98.38 & 98.74 & 95.50 & 86.51 & 87.49 & \cellcolor{blue!8}\underline{94.85} \\
        \midrule
        \multirow{4}{*}{NPR} & Recall & 25.40 & 52.20 & 42.40 & 26.00 & 21.40 & 48.20 & 66.60 & 22.00 & 10.71 & 12.20 & \cellcolor{blue!8}32.71 \\
         & Accuracy & 62.40 & 75.80 & 70.90 & 62.70 & 60.40 & 73.80 & 83.00 & 60.70 & 55.36 & 55.80 & \cellcolor{blue!8}66.09 \\
         & F1 & 40.32 & 68.32 & 59.30 & 41.07 & 35.08 & 64.78 & 79.67 & 35.89 & 19.35 & 21.63 & \cellcolor{blue!8}46.54 \\
         & AUROC & 83.64 & 94.85 & 92.44 & 86.68 & 83.77 & 94.33 & 95.77 & 84.34 & 84.60 & 70.52 & \cellcolor{blue!8}87.10 \\
        \midrule
        \multirow{4}{*}{TALL} & Recall & 28.20 & 45.20 & 41.20 & 26.20 & 33.80 & 60.20 & 60.20 & 22.60 & 25.00 & 18.20 & \cellcolor{blue!8}36.08 \\
         & Accuracy & 64.00 & 72.50 & 70.50 & 63.00 & 66.80 & 80.00 & 80.00 & 61.20 & 62.50 & 59.00 & \cellcolor{blue!8}67.95 \\
         & F1 & 43.93 & 62.17 & 58.27 & 41.46 & 50.45 & 75.06 & 75.06 & 36.81 & 40.00 & 30.74 & \cellcolor{blue!8}51.40 \\
         & AUROC & 93.34 & 94.56 & 94.25 & 91.64 & 91.63 & 94.99 & 97.60 & 91.46 & 84.92 & 85.20 & \cellcolor{blue!8}91.96 \\
        \midrule
        \multirow{4}{*}{STIL} & Recall & 25.80 & 64.80 & 68.40 & 46.20 & 29.20 & 67.20 & 70.00 & 44.40 & 26.79 & 25.00 & \cellcolor{blue!8}46.78 \\
         & Accuracy & 62.70 & 82.20 & 84.00 & 72.90 & 64.40 & 83.40 & 84.80 & 72.00 & 63.39 & 62.30 & \cellcolor{blue!8}73.21 \\
         & F1 & 40.89 & 78.45 & 81.04 & 63.03 & 45.06 & 80.19 & 82.16 & 61.33 & 42.25 & 39.87 & \cellcolor{blue!8}61.43 \\
         & AUROC & 85.14 & 95.74 & 96.87 & 89.46 & 83.22 & 96.09 & 96.23 & 90.36 & 89.89 & 78.99 & \cellcolor{blue!8}90.20 \\
        \midrule
        \multirow{4}{*}{\shortstack{NSG-VD\\(Ours)}} & Recall & 85.83 & 99.17 & 100.00 & 99.17 & 97.50 & 95.83 & 99.17 & 91.67 & 82.14 & 81.67 & \cellcolor{pink!30}\textbf{93.21} \\
         & Accuracy & 84.58 & 92.50 & 93.75 & 89.58 & 89.58 & 90.83 & 92.50 & 90.00 & 86.61 & 81.67 & \cellcolor{pink!30}\textbf{89.16} \\
         & F1 & 84.77 & 92.97 & 94.12 & 90.49 & 90.35 & 91.27 & 92.97 & 90.16 & 85.98 & 81.67 & \cellcolor{pink!30}\textbf{89.48} \\
         & AUROC & 90.76 & 98.18 & 98.18 & 95.03 & 95.48 & 96.97 & 96.53 & 95.11 & 95.73 & 87.13 & \cellcolor{pink!30}\textbf{94.91} \\
    \bottomrule
    \end{tabular}
    }
    \end{threeparttable}
\vspace{-10pt}
\end{table}

\subsection{Impact of Spatial Gradients and Temporal Derivatives for NSG-VD}

\begin{wraptable}[7]{r}{0.445\textwidth}
\vspace{-1.45em}
    \caption{Impact of spatial gradients and temporal derivatives on average metrics (\%).}
    \vspace{-1.1em}
    \label{tab: Impact of Spatial and Temporal}
    \begin{center}
    \begin{threeparttable}
    \LARGE
    \renewcommand{\arraystretch}{1.45}
    \resizebox{\linewidth}{!}{
    \begin{tabular}{l|cccc}
    \toprule
        Method & Recall & Accuracy & F1 & AUROC  \\ \midrule
        Spatial Gradients &  87.99 & 82.84 & 83.40 & 91.85  \\
        Temporal Derivatives &  60.35 & 71.09 & 66.97 & 78.95  \\
        \rowcolor{pink!30} NSG-VD (Ours) & \textbf{88.02} & \textbf{91.46} & \textbf{90.87} &  \textbf{96.14} \\ \bottomrule
    \end{tabular}
    }
    \end{threeparttable}
    \end{center}
    \vspace{-0.3em}  
\end{wraptable}

To investigate the impact of the spatial gradients $\nabla_{\mathbf{x}} \log p(\mathbf{x}, t)$ and temporal derivatives $\partial_t \log p(\mathbf{x}, t)$ for our NSG-VD, we evaluate these components as independent detection statistics for AI-generated video detection. To this end, we train these separate models 
with $10,000$ real videos from Kinetics-400 and generated videos from Pika. 
From Table \ref{tab: Impact of Spatial and Temporal}, the spatial gradient achieves moderate performance (\eg, \(87.99\% \) Recall, \(83.40\% \) F1-score), suggesting its ability to capture spatial anomalies, which may arise from its sensitivity to localized variations in texture or geometry. The temporal derivative, however, shows limited detection power (\eg, \(60.35\% \) Recall, \(66.97\% \) F1-score), likely due to its sensitivity to transient noise in dynamic modeling. In contrast, our NSG-VD integrating both components achieves significantly enhanced performance (\eg, \(88.02\% \) Recall, \(90.87\%\) F1-score). This demonstrates that the interplay between spatial gradients and temporal derivatives formalized via physical conservation principles is critical for video detection.  

\subsection{Impact of Decision Threshold for NSG-VD}
\begin{wrapfigure}[10]{r}{0.44\textwidth}
\centering
\vspace{-14pt}
\includegraphics[width=0.99\linewidth]{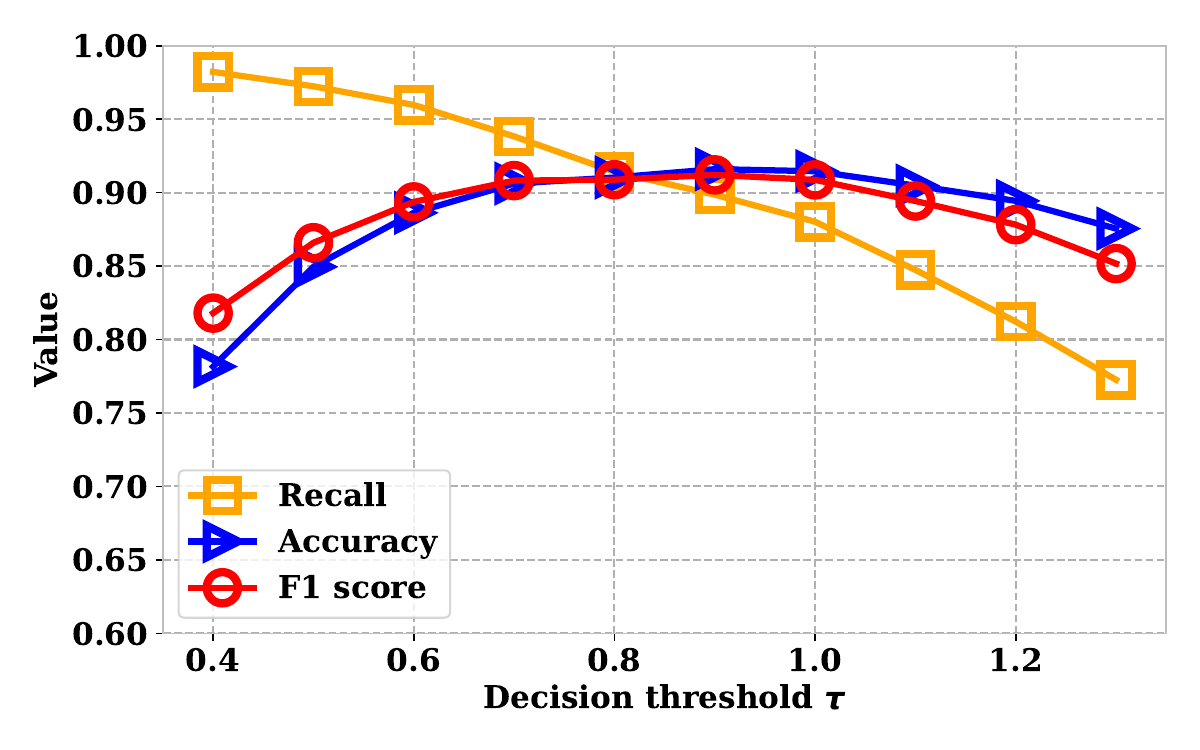}
\vspace{-17pt}
    \caption{Impact of decision threshold.
    }
    \label{fig: Threshold}
\end{wrapfigure}
We evaluate the decision threshold $\tau$ in Eqn. (\ref{eqn: decision}) for NSG-VD by testing $\tau \in [0.4, 1.3]$ under the same settings as Table \ref{tab: standard Pika}. As shown in Figure \ref{fig: Threshold}, our NSG-VD maintains remarkably stable performance across a wide range of $\tau$ values without requiring fine-grained tuning. Specifically, NSG-VD consistently shows high detection performance as $\tau\in[0.7, 1.1]$ for average Recall, Accuracy and F1-Score across diverse generators. These results indicate that NSG features create a clear separation between real and fake distributions. We set $\tau=1.0$ as the default throughout all settings.

\section{Conclusion}
In this paper, we propose a physics-driven AI-generated video detection paradigm by modeling spatiotemporal dynamics through the Normalized Spatiotemporal Gradient (NSG), a novel statistic based on probability flow conservation principles. Leveraging pre-trained diffusion models, we propose an NSG-based video detection method (NSG-VD). Theoretical analyses and extensive experiments validate the superiority of our NSG-VD in detecting advanced generated videos.

\section*{Acknowledgements}
This work was partially supported by the Joint Funds of the National Natural Science Foundation of China (Grant No.U24A20327), RGC Young Collaborative Research Grant No. C2005-24Y, RGC General Research Fund No. 12200725, and NSFC General Program No. 62376235.


{
    \bibliographystyle{unsrt}
    \bibliography{ref}

@string{AAAI = "AAAI Conference on Artificial Intelligence"}

@string{CVPR = "IEEE Conference on Computer Vision and Pattern Recognition"}

@string{ICML = "International Conference on Machine Learning"}

@string{NEURIPS = "Advances in Neural Information Processing Systems"}

@book{panton2024incompressible,
  title={Incompressible flow},
  author={Panton, Ronald L},
  year={2024},
  publisher={John Wiley \& Sons}
}

@book{synge1978tensor,
  title={Tensor calculus},
  author={Synge, John Lighton and Schild, Alfred},
  volume={5},
  year={1978},
  publisher={Courier Corporation}
}

@article{wilczek1999quantum,
  title={Quantum field theory},
  author={Wilczek, Frank},
  journal={Reviews of Modern Physics},
  volume={71},
  number={2},
  pages={S85},
  year={1999},
  publisher={APS}
}

@article{hodge2014electron,
  title={Electron spin and probability current density in quantum mechanics},
  author={Hodge, WB and Migirditch, SV and Kerr, William C},
  journal={American Journal of Physics},
  volume={82},
  number={7},
  pages={681--690},
  year={2014},
  publisher={American Association of Physics Teachers}
}

@inproceedings{song2020score,
  title={Score-based generative modeling through stochastic differential equations},
  author={Song, Yang and Sohl-Dickstein, Jascha and Kingma, Diederik P and Kumar, Abhishek and Ermon, Stefano and Poole, Ben},
  booktitle={International Conference on Learning Representations},
  year={2021}
}

@article{song2019generative,
  title={Generative modeling by estimating gradients of the data distribution},
  author={Song, Yang and Ermon, Stefano},
  journal={Advances in Neural Information Processing Systems},
  volume={32},
  year={2019}
}

@article{horn1981determining,
  title={Determining optical flow},
  author={Horn, Berthold KP and Schunck, Brian G},
  journal={Artificial intelligence},
  volume={17},
  number={1-3},
  pages={185--203},
  year={1981},
  publisher={Elsevier}
}

@article{ho2020denoising,
  title={Denoising diffusion probabilistic models},
  author={Ho, Jonathan and Jain, Ajay and Abbeel, Pieter},
  journal={Advances in Neural Information Processing Systems},
  volume={33},
  pages={6840--6851},
  year={2020}
}

@article{dhariwal2021diffusion,
  title={Diffusion models beat gans on image synthesis},
  author={Dhariwal, Prafulla and Nichol, Alexander},
  journal={Advances in Neural Information Processing Systems},
  volume={34},
  pages={8780--8794},
  year={2021}
}

@inproceedings{zhang2024detecting,
  title={Detecting machine-generated texts by multi-population aware optimization for maximum mean discrepancy},
  author={Zhang, Shuhai and Song, Yiliao and Yang, Jiahao and Li, Yuanqing and Han, Bo and Tan, Mingkui},
  booktitle={International Conference on Learning Representations},
  year={2024}
}

@inproceedings{song2025detecting,
  title={Deep Kernel Relative Test for Machine-generated Text Detection},
  author={Song, Yiliao and Yuan, Zhenqiao and Zhang, Shuhai and Fang, Zhen and Yu, Jun and Liu, Feng},
  booktitle={The Thirteenth International Conference on Learning Representations},
  year={2025}
}

@inproceedings{gao2021maximum,
  title={Maximum mean discrepancy test is aware of adversarial attacks},
  author={Gao, Ruize and Liu, Feng and Zhang, Jingfeng and Han, Bo and Liu, Tongliang and Niu, Gang and Sugiyama, Masashi},
  booktitle={International Conference on Machine Learning},
  pages={3564--3575},
  year={2021},
  organization={PMLR}
}

@inproceedings{zhang2023detecting,
  title={Detecting adversarial data by probing multiple perturbations using expected perturbation score},
  author={Zhang, Shuhai and Liu, Feng and Yang, Jiahao and Yang, Yifan and Li, Changsheng and Han, Bo and Tan, Mingkui},
  booktitle={International conference on machine learning},
  pages={41429--41451},
  year={2023},
  organization={PMLR}
}

@article{gretton2012kernel,
  title={A kernel two-sample test},
  author={Gretton, Arthur and Borgwardt, Karsten M and Rasch, Malte J and Sch{\"o}lkopf, Bernhard and Smola, Alexander},
  journal={Journal of Machine Learning Research},
  volume={13},
  number={1},
  pages={723--773},
  year={2012},
  publisher={JMLR. org}
}

@inproceedings{liu2020learning,
  title={Learning deep kernels for non-parametric two-sample tests},
  author={Liu, Feng and Xu, Wenkai and Lu, Jie and Zhang, Guangquan and Gretton, Arthur and Sutherland, Danica J},
  booktitle={International Conference on Machine Learning},
  pages={6316--6326},
  year={2020},
  organization={PMLR}
}

@book{batchelor2000introduction,
  title={An introduction to fluid dynamics},
  author={Batchelor, George Keith},
  year={2000},
  publisher={Cambridge university press}
}

@book{rieutord2014fluid,
  title={Fluid dynamics: an introduction},
  author={Rieutord, Michel},
  year={2014},
  publisher={Springer}
}

@book{bohm2013quantum,
  title={Quantum mechanics: foundations and applications},
  author={B{\"o}hm, Arno},
  year={2013},
  publisher={Springer Science \& Business Media}
}

@inproceedings{blattmann2023align,
  title={Align your latents: High-resolution video synthesis with latent diffusion models},
  author={Blattmann, Andreas and Rombach, Robin and Ling, Huan and Dockhorn, Tim and Kim, Seung Wook and Fidler, Sanja and Kreis, Karsten},
  booktitle={Proceedings of the IEEE/CVF conference on computer vision and pattern recognition},
  pages={22563--22575},
  year={2023}
}

@article{brooks2024video,
  title={Video generation models as world simulators},
  author={Brooks, Tim and Peebles, Bill and Holmes, Connor and DePue, Will and Guo, Yufei and Jing, Li and Schnurr, David and Taylor, Joe and Luhman, Troy and Luhman, Eric and others},
  journal={OpenAI Blog},
  volume={1},
  pages={8},
  year={2024}
}

@inproceedings{huang2025fine,
  title={Fine-grained controllable video generation via object appearance and context},
  author={Huang, Hsin-Ping and Su, Yu-Chuan and Sun, Deqing and Jiang, Lu and Jia, Xuhui and Zhu, Yukun and Yang, Ming-Hsuan},
  booktitle={2025 IEEE/CVF Winter Conference on Applications of Computer Vision (WACV)},
  pages={3698--3708},
  year={2025},
  organization={IEEE}
}

@inproceedings{wu2025customcrafter,
  title={Customcrafter: Customized video generation with preserving motion and concept composition abilities},
  author={Wu, Tao and Zhang, Yong and Wang, Xintao and Zhou, Xianpan and Zheng, Guangcong and Qi, Zhongang and Shan, Ying and Li, Xi},
  booktitle={Proceedings of the AAAI Conference on Artificial Intelligence},
  volume={39},
  number={8},
  pages={8469--8477},
  year={2025}
}

@article{brooks2022generating,
  title={Generating long videos of dynamic scenes},
  author={Brooks, Tim and Hellsten, Janne and Aittala, Miika and Wang, Ting-Chun and Aila, Timo and Lehtinen, Jaakko and Liu, Ming-Yu and Efros, Alexei and Karras, Tero},
  journal={Advances in Neural Information Processing Systems},
  volume={35},
  pages={31769--31781},
  year={2022}
}

@article{blattmann2023stable,
  title={Stable video diffusion: Scaling latent video diffusion models to large datasets},
  author={Blattmann, Andreas and Dockhorn, Tim and Kulal, Sumith and Mendelevitch, Daniel and Kilian, Maciej and Lorenz, Dominik and Levi, Yam and English, Zion and Voleti, Vikram and Letts, Adam and others},
  journal={arXiv preprint arXiv:2311.15127},
  year={2023}
}

@inproceedings{xu2023tall,
  title={Tall: Thumbnail layout for deepfake video detection},
  author={Xu, Yuting and Liang, Jian and Jia, Gengyun and Yang, Ziming and Zhang, Yanhao and He, Ran},
  booktitle={Proceedings of the IEEE/CVF international conference on computer vision},
  pages={22658--22668},
  year={2023}
}

@inproceedings{qian2020thinking,
	title={Thinking in frequency: Face forgery detection by mining frequency-aware clues},
	author={Qian, Yuyang and Yin, Guojun and Sheng, Lu and Chen, Zixuan and Shao, Jing},
	booktitle={European conference on computer vision},
	pages={86--103},
	year={2020},
	organization={Springer}
}

@article{chen2024demamba,
	title={DeMamba: AI-Generated Video Detection on Million-Scale GenVideo Benchmark},
	author={Chen, Haoxing and Hong, Yan and Huang, Zizheng and Xu, Zhuoer and Gu, Zhangxuan and Li, Yaohui and Lan, Jun and Zhu, Huijia and Zhang, Jianfu and Wang, Weiqiang and others},
	journal={arXiv preprint arXiv:2405.19707},
	year={2024}
}

@inproceedings{tan2024rethinking,
  title={Rethinking the up-sampling operations in cnn-based generative network for generalizable deepfake detection},
  author={Tan, Chuangchuang and Zhao, Yao and Wei, Shikui and Gu, Guanghua and Liu, Ping and Wei, Yunchao},
  booktitle={Proceedings of the IEEE/CVF Conference on Computer Vision and Pattern Recognition},
  pages={28130--28139},
  year={2024}
}

@article{abdel1954approximate,
  title={Approximate formulae for the percentage points and the probability integral of the non-central $\chi$ 2 distribution},
  author={Abdel-Aty, SH},
  journal={Biometrika},
  volume={41},
  number={3/4},
  pages={538--540},
  year={1954},
  publisher={JSTOR}
}

@inproceedings{amerini2019deepfake,
	title={Deepfake video detection through optical flow based cnn},
	author={Amerini, Irene and Galteri, Leonardo and Caldelli, Roberto and Del Bimbo, Alberto},
	booktitle={Proceedings of the IEEE/CVF international conference on computer vision workshops},
	pages={0--0},
	year={2019}
}

@inproceedings{yang2019exposing,
	title={Exposing deep fakes using inconsistent head poses},
	author={Yang, Xin and Li, Yuezun and Lyu, Siwei},
	booktitle={ICASSP 2019-2019 IEEE international conference on acoustics, speech and signal processing (ICASSP)},
	pages={8261--8265},
	year={2019},
	organization={IEEE}
}

@inproceedings{wang2023altfreezing,
	title={Altfreezing for more general video face forgery detection},
	author={Wang, Zhendong and Bao, Jianmin and Zhou, Wengang and Wang, Weilun and Li, Houqiang},
	booktitle={Proceedings of the IEEE/CVF conference on computer vision and pattern recognition},
	pages={4129--4138},
	year={2023}
}

@article{peng2024deepfakes,
	title={Where deepfakes gaze at? spatial-temporal gaze inconsistency analysis for video face forgery detection},
	author={Peng, Chunlei and Miao, Zimin and Liu, Decheng and Wang, Nannan and Hu, Ruimin and Gao, Xinbo},
	journal={IEEE Transactions on Information Forensics and Security},
	year={2024},
	publisher={IEEE}
}

@inproceedings{gu2021spatiotemporal,
  title={Spatiotemporal inconsistency learning for deepfake video detection},
  author={Gu, Zhihao and Chen, Yang and Yao, Taiping and Ding, Shouhong and Li, Jilin and Huang, Feiyue and Ma, Lizhuang},
  booktitle={Proceedings of the 29th ACM international conference on multimedia},
  pages={3473--3481},
  year={2021}
}

@inproceedings{bai2024ai,
	title={AI-Generated Video Detection via Spatial-Temporal Anomaly Learning},
	author={Bai, Jianfa and Lin, Man and Cao, Gang and Lou, Zijie},
	booktitle={Chinese Conference on Pattern Recognition and Computer Vision (PRCV)},
	pages={460--470},
	year={2024},
	organization={Springer}
}

@article{ma2024decof,
	title={DeCoF: Generated Video Detection via Frame Consistency: The First Benchmark Dataset},
	author={Ma, Long and Zhang, Jiajia and Deng, Hongping and Zhang, Ningyu and Guo, Qinglang and Yu, Haiyang and Liao, Yong and Zhou, Pengyuan},
	journal={arXiv preprint arXiv:2402.02085},
	year={2024}
}

@inproceedings{songlearning,
	title={On Learning Multi-Modal Forgery Representation for Diffusion Generated Video Detection},
	author={Song, Xiufeng and Guo, Xiao and Zhang, Jiache and Li, Qirui and BAI, LEI and Liu, Xiaoming and Zhai, Guangtao and Liu, Xiaohong},
	booktitle={The Thirty-eighth Annual Conference on Neural Information Processing Systems}
}

@article{birge1998minimum,
  title={Minimum contrast estimators on sieves: exponential bounds and rates of convergence},
  author={Birg{\'e}, Lucien and Massart, Pascal},
  year={1998}
}

@article{muller1997integral,
  title={Integral probability metrics and their generating classes of functions},
  author={M{\"u}ller, Alfred},
  journal={Advances in applied probability},
  volume={29},
  number={2},
  pages={429--443},
  year={1997},
  publisher={Cambridge University Press}
}

@article{tolstikhin2016minimax,
  title={Minimax estimation of maximum mean discrepancy with radial kernels},
  author={Tolstikhin, Ilya O and Sriperumbudur, Bharath K and Sch{\"o}lkopf, Bernhard},
  journal={Advances in Neural Information Processing Systems},
  volume={29},
  year={2016}
}

@article{kim2022minimax,
  title={Minimax optimality of permutation tests},
  author={Kim, Ilmun and Balakrishnan, Sivaraman and Wasserman, Larry},
  journal={The Annals of Statistics},
  volume={50},
  number={1},
  pages={225--251},
  year={2022},
  publisher={Institute of Mathematical Statistics}
}

@inproceedings{zhangs2023EPSAD,
  title={Detecting Adversarial Data by Probing Multiple Perturbations Using Expected Perturbation Score},
  author={Zhang, Shuhai and Liu, Feng and Yang, Jiahao and Yang, Yifan and Li, Changsheng and Han, Bo and Tan, Mingkui},
  booktitle = {International Conference on Machine Learning},
  pages={41429--41451},
  year={2023},
  organization={PMLR}
}

@inproceedings{song2020improved,
  title={Improved techniques for training score-based generative models},
  author={Song, Yang and Ermon, Stefano},
  booktitle={NeurIPS},
  volume={33},
  pages={12438--12448},
  year={2020}
}

@inproceedings{yoon2021adversarial,
  title={Adversarial purification with score-based generative models},
  author={Yoon, Jongmin and Hwang, Sung Ju and Lee, Juho},
  booktitle={ICML},
  pages={12062--12072},
  year={2021},
  organization={PMLR}
}

@inproceedings{nie2022diffusion,
  title={Diffusion Models for Adversarial Purification},
  author={Nie, Weili and Guo, Brandon and Huang, Yujia and Xiao, Chaowei and Vahdat, Arash and Anandkumar, Animashree},
  booktitle={ICML},
  pages={16805--16827},
  year={2022},
  organization={PMLR}
}

@inproceedings{wang2023dire,
  title={Dire for diffusion-generated image detection},
  author={Wang, Zhendong and Bao, Jianmin and Zhou, Wengang and Wang, Weilun and Hu, Hezhen and Chen, Hong and Li, Houqiang},
  booktitle={Proceedings of the IEEE/CVF International Conference on Computer Vision},
  pages={22445--22455},
  year={2023}
}

@article{huang2005using,
  title={Using AUC and accuracy in evaluating learning algorithms},
  author={Huang, Jin and Ling, Charles X},
  journal={IEEE Transactions on knowledge and Data Engineering},
  volume={17},
  number={3},
  pages={299--310},
  year={2005},
  publisher={IEEE}
}

@article{jimenez2012insights,
  title={Insights into the area under the receiver operating characteristic curve (AUC) as a discrimination measure in species distribution modelling},
  author={Jim{\'e}nez-Valverde, Alberto},
  journal={Global Ecology and Biogeography},
  volume={21},
  number={4},
  pages={498--507},
  year={2012},
  publisher={Wiley Online Library}
}

@inproceedings{chinchor1993muc,
  title={MUC-5 evaluation metrics},
  author={Chinchor, Nancy and Sundheim, Beth M},
  booktitle={Fifth Message Understanding Conference (MUC-5): Proceedings of a Conference Held in Baltimore, Maryland, August 25-27, 1993},
  year={1993}
}

@inproceedings{xu2016msr,
  title={Msr-vtt: A large video description dataset for bridging video and language},
  author={Xu, Jun and Mei, Tao and Yao, Ting and Rui, Yong},
  booktitle={Proceedings of the IEEE conference on computer vision and pattern recognition},
  pages={5288--5296},
  year={2016}
}

@article{kay2017kinetics,
  title={The kinetics human action video dataset},
  author={Kay, Will and Carreira, Joao and Simonyan, Karen and Zhang, Brian and Hillier, Chloe and Vijayanarasimhan, Sudheendra and Viola, Fabio and Green, Tim and Back, Trevor and Natsev, Paul and others},
  journal={arXiv preprint arXiv:1705.06950},
  year={2017}
}

@article{xu2023youku,
  title={Youku-mplug: A 10 million large-scale chinese video-language dataset for pre-training and benchmarks},
  author={Xu, Haiyang and Ye, Qinghao and Wu, Xuan and Yan, Ming and Miao, Yuan and Ye, Jiabo and Xu, Guohai and Hu, Anwen and Shi, Yaya and Xu, Guangwei and others},
  journal={arXiv preprint arXiv:2306.04362},
  year={2023}
}

@article{wang2024lavie,
  title={Lavie: High-quality video generation with cascaded latent diffusion models},
  author={Wang, Yaohui and Chen, Xinyuan and Ma, Xin and Zhou, Shangchen and Huang, Ziqi and Wang, Yi and Yang, Ceyuan and He, Yinan and Yu, Jiashuo and Yang, Peiqing and others},
  journal={International Journal of Computer Vision},
  pages={1--20},
  year={2024},
  publisher={Springer}
}

@inproceedings{chen2024videocrafter2,
  title={Videocrafter2: Overcoming data limitations for high-quality video diffusion models},
  author={Chen, Haoxin and Zhang, Yong and Cun, Xiaodong and Xia, Menghan and Wang, Xintao and Weng, Chao and Shan, Ying},
  booktitle={Proceedings of the IEEE/CVF Conference on Computer Vision and Pattern Recognition},
  pages={7310--7320},
  year={2024}
}

@inproceedings{chen2023seine,
  title={Seine: Short-to-long video diffusion model for generative transition and prediction},
  author={Chen, Xinyuan and Wang, Yaohui and Zhang, Lingjun and Zhuang, Shaobin and Ma, Xin and Yu, Jiashuo and Wang, Yali and Lin, Dahua and Qiao, Yu and Liu, Ziwei},
  booktitle={The Twelfth International Conference on Learning Representations},
  year={2023}
}

@article{zheng2024open,
  title={Open-sora: Democratizing efficient video production for all},
  author={Zheng, Zangwei and Peng, Xiangyu and Yang, Tianji and Shen, Chenhui and Li, Shenggui and Liu, Hongxin and Zhou, Yukun and Li, Tianyi and You, Yang},
  journal={arXiv preprint arXiv:2412.20404},
  year={2024}
}

@article{zhang2023i2vgen,
  title={I2vgen-xl: High-quality image-to-video synthesis via cascaded diffusion models},
  author={Zhang, Shiwei and Wang, Jiayu and Zhang, Yingya and Zhao, Kang and Yuan, Hangjie and Qin, Zhiwu and Wang, Xiang and Zhao, Deli and Zhou, Jingren},
  journal={arXiv preprint arXiv:2311.04145},
  year={2023}
}

@inproceedings{xing2024dynamicrafter,
  title={Dynamicrafter: Animating open-domain images with video diffusion priors},
  author={Xing, Jinbo and Xia, Menghan and Zhang, Yong and Chen, Haoxin and Yu, Wangbo and Liu, Hanyuan and Liu, Gongye and Wang, Xintao and Shan, Ying and Wong, Tien-Tsin},
  booktitle={European Conference on Computer Vision},
  pages={399--417},
  year={2024},
  organization={Springer}
}

@article{ma2024latte,
  title={Latte: Latent diffusion transformer for video generation},
  author={Ma, Xin and Wang, Yaohui and Jia, Gengyun and Chen, Xinyuan and Liu, Ziwei and Li, Yuan-Fang and Chen, Cunjian and Qiao, Yu},
  journal={arXiv preprint arXiv:2401.03048},
  year={2024}
}

@inproceedings{zhang2024pia,
  title={Pia: Your personalized image animator via plug-and-play modules in text-to-image models},
  author={Zhang, Yiming and Xing, Zhening and Zeng, Yanhong and Fang, Youqing and Chen, Kai},
  booktitle={Proceedings of the IEEE/CVF conference on computer vision and pattern recognition},
  pages={7747--7756},
  year={2024}
}

@article{wang2023modelscope,
  title={Modelscope text-to-video technical report},
  author={Wang, Jiuniu and Yuan, Hangjie and Chen, Dayou and Zhang, Yingya and Wang, Xiang and Zhang, Shiwei},
  journal={arXiv preprint arXiv:2308.06571},
  year={2023}
}

@inproceedings{esser2023structure,
  title={Structure and content-guided video synthesis with diffusion models},
  author={Esser, Patrick and Chiu, Johnathan and Atighehchian, Parmida and Granskog, Jonathan and Germanidis, Anastasis},
  booktitle={Proceedings of the IEEE/CVF international conference on computer vision},
  pages={7346--7356},
  year={2023}
}

@article{chen2023videocrafter1,
  title={Videocrafter1: Open diffusion models for high-quality video generation},
  author={Chen, Haoxin and Xia, Menghan and He, Yingqing and Zhang, Yong and Cun, Xiaodong and Yang, Shaoshu and Xing, Jinbo and Liu, Yaofang and Chen, Qifeng and Wang, Xintao and others},
  journal={arXiv preprint arXiv:2310.19512},
  year={2023}
}

@article{gerber2023kernel,
  title={Kernel-based tests for likelihood-free hypothesis testing},
  author={Gerber, Patrik R{\'o}bert and Jiang, Tianze and Polyanskiy, Yury and Sun, Rui},
  journal={Advances in Neural Information Processing Systems},
  volume={36},
  pages={15680--15715},
  year={2023}
}

@inproceedings{kalinke2023nystrom,
  title={Nystr{\"o}m $ M $-Hilbert-Schmidt independence criterion},
  author={Kalinke, Florian and Szab{\'o}, Zolt{\'a}n},
  booktitle={Uncertainty in Artificial Intelligence},
  pages={1005--1015},
  year={2023},
  organization={PMLR}
}

@inproceedings{bao2018towards,
  title={Towards open-set identity preserving face synthesis},
  author={Bao, Jianmin and Chen, Dong and Wen, Fang and Li, Houqiang and Hua, Gang},
  booktitle={Proceedings of the IEEE conference on computer vision and pattern recognition},
  pages={6713--6722},
  year={2018}
}

@inproceedings{wang2021hififace,
  title={HifiFace: 3D Shape and Semantic Prior Guided High Fidelity Face Swapping},
  author={Wang, Yuhan and Chen, Xu and Zhu, Junwei and Chu, Wenqing and Tai, Ying and Wang, Chengjie and Li, Jilin and Wu, Yongjian and Huang, Feiyue and Ji, Rongrong},
  booktitle={Proceedings of the Thirtieth International Joint Conference on Artificial Intelligence},
  pages={1136--1142},
  year={2021},
  organization={International Joint Conferences on Artificial Intelligence Organization}
}

@inproceedings{zhao2023diffswap,
  title={Diffswap: High-fidelity and controllable face swapping via 3d-aware masked diffusion},
  author={Zhao, Wenliang and Rao, Yongming and Shi, Weikang and Liu, Zuyan and Zhou, Jie and Lu, Jiwen},
  booktitle={Proceedings of the IEEE/CVF Conference on Computer Vision and Pattern Recognition},
  pages={8568--8577},
  year={2023}
}

@inproceedings{oorloff2024avff,
  title={Avff: Audio-visual feature fusion for video deepfake detection},
  author={Oorloff, Trevine and Koppisetti, Surya and Bonettini, Nicol{\`o} and Solanki, Divyaraj and Colman, Ben and Yacoob, Yaser and Shahriyari, Ali and Bharaj, Gaurav},
  booktitle={Proceedings of the IEEE/CVF Conference on Computer Vision and Pattern Recognition},
  pages={27102--27112},
  year={2024}
}

@article{guo2024pulid,
  title={Pulid: Pure and lightning id customization via contrastive alignment},
  author={Guo, Zinan and Wu, Yanze and Zhuowei, Chen and Zhang, Peng and He, Qian and others},
  journal={Advances in neural information processing systems},
  volume={37},
  pages={36777--36804},
  year={2024}
}

@inproceedings{rombach2022high,
  title={High-resolution image synthesis with latent diffusion models},
  author={Rombach, Robin and Blattmann, Andreas and Lorenz, Dominik and Esser, Patrick and Ommer, Bj{\"o}rn},
  booktitle={Proceedings of the IEEE/CVF conference on computer vision and pattern recognition},
  pages={10684--10695},
  year={2022}
}

@inproceedings{shi2024motion,
  title={Motion-i2v: Consistent and controllable image-to-video generation with explicit motion modeling},
  author={Shi, Xiaoyu and Huang, Zhaoyang and Wang, Fu-Yun and Bian, Weikang and Li, Dasong and Zhang, Yi and Zhang, Manyuan and Cheung, Ka Chun and See, Simon and Qin, Hongwei and others},
  booktitle={ACM SIGGRAPH 2024 Conference Papers},
  pages={1--11},
  year={2024}
}

@inproceedings{li2025image,
  title={Image conductor: Precision control for interactive video synthesis},
  author={Li, Yaowei and Wang, Xintao and Zhang, Zhaoyang and Wang, Zhouxia and Yuan, Ziyang and Xie, Liangbin and Shan, Ying and Zou, Yuexian},
  booktitle={Proceedings of the AAAI Conference on Artificial Intelligence},
  volume={39},
  number={5},
  pages={5031--5038},
  year={2025}
}

@inproceedings{ouyang2024codef,
  title={Codef: Content deformation fields for temporally consistent video processing},
  author={Ouyang, Hao and Wang, Qiuyu and Xiao, Yuxi and Bai, Qingyan and Zhang, Juntao and Zheng, Kecheng and Zhou, Xiaowei and Chen, Qifeng and Shen, Yujun},
  booktitle={Proceedings of the IEEE/CVF Conference on Computer Vision and Pattern Recognition},
  pages={8089--8099},
  year={2024}
}

@InProceedings{Huang_2024_CVPR,
    author    = {Huang, Ziqi and He, Yinan and Yu, Jiashuo and Zhang, Fan and Si, Chenyang and Jiang, Yuming and Zhang, Yuanhan and Wu, Tianxing and Jin, Qingyang and Chanpaisit, Nattapol and Wang, Yaohui and Chen, Xinyuan and Wang, Limin and Lin, Dahua and Qiao, Yu and Liu, Ziwei},
    title     = {VBench: Comprehensive Benchmark Suite for Video Generative Models},
    booktitle = {Proceedings of the IEEE/CVF Conference on Computer Vision and Pattern Recognition (CVPR)},
    month     = {June},
    year      = {2024},
    pages     = {21807-21818}
}

@article{zheng2025vbench,
  title={VBench-2.0: Advancing video generation benchmark suite for intrinsic faithfulness},
  author={Zheng, Dian and Huang, Ziqi and Liu, Hongbo and Zou, Kai and He, Yinan and Zhang, Fan and Zhang, Yuanhan and He, Jingwen and Zheng, Wei-Shi and Qiao, Yu and others},
  journal={arXiv preprint arXiv:2503.21755},
  year={2025}
}

@inproceedings{wang2024pika,
  title={Pika: Empowering Non-Programmers to Author Executable Governance Policies in Online Communities},
  author={Wang, Leijie and Vincent, Nicholas and Rukanskait{\.e}, Julija and Zhang, Amy Xian},
  booktitle={Proceedings of the 2024 CHI Conference on Human Factors in Computing Systems},
  pages={1--18},
  year={2024}
}

@misc{zeroscope,
  title         = {Zeroscope},
  howpublished  = {2023. URL \url{https://huggingface.co/cerspense/zeroscope_v2_XL}},
  url        = "https://huggingface.co/cerspense/zeroscope_v2_XL"
}

@misc{hotshotxl,
  title         = {Hotshot},
  howpublished  = {2023. URL \url{https://github.com/hotshotco/Hotshot-XL}},
  url        = "https://github.com/hotshotco/Hotshot-XL"
}

@misc{zhang2023show1,
      title={Show-1: Marrying Pixel and Latent Diffusion Models for Text-to-Video Generation}, 
      author={David Junhao Zhang and Jay Zhangjie Wu and Jia-Wei Liu and Rui Zhao and Lingmin Ran and Yuchao Gu and Difei Gao and Mike Zheng Shou},
      year={2023},
      eprint={2309.15818},
      archivePrefix={arXiv},
      primaryClass={cs.CV}
}

@inproceedings{liu2021swin,
  title={Swin transformer: Hierarchical vision transformer using shifted windows},
  author={Liu, Ze and Lin, Yutong and Cao, Yue and Hu, Han and Wei, Yixuan and Zhang, Zheng and Lin, Stephen and Guo, Baining},
  booktitle={Proceedings of the IEEE/CVF international conference on computer vision},
  pages={10012--10022},
  year={2021}
}

@inproceedings{kingma2014adam,
  title={Adam: A method for stochastic optimization},
  author={Kingma, Diederik P and Ba, Jimmy},
    booktitle={International Conference on Learning Representations},
  year={2015}
}

@inproceedings{xie2023difffit,
  title={Difffit: Unlocking transferability of large diffusion models via simple parameter-efficient fine-tuning},
  author={Xie, Enze and Yao, Lewei and Shi, Han and Liu, Zhili and Zhou, Daquan and Liu, Zhaoqiang and Li, Jiawei and Li, Zhenguo},
  booktitle={Proceedings of the IEEE/CVF International Conference on Computer Vision},
  pages={4230--4239},
  year={2023}
}

@article{denker2024deft,
  title={DEFT: Efficient Fine-tuning of Diffusion Models by Learning the Generalised $ h $-transform},
  author={Denker, Alexander and Vargas, Francisco and Padhy, Shreyas and Didi, Kieran and Mathis, Simon and Barbano, Riccardo and Dutordoir, Vincent and Mathieu, Emile and Komorowska, Urszula Julia and Lio, Pietro},
  journal={Advances in Neural Information Processing Systems},
  volume={37},
  pages={19636--19682},
  year={2024}
}

@article{
fang2023structural,
title={Structural Pruning for Diffusion Models},
author={Gongfan Fang and Xinyin Ma and Xinchao Wang},
  journal={Advances in Neural Information Processing Systems},
year={2023}
}

@inproceedings{ma2024deepcache,
  title={Deepcache: Accelerating diffusion models for free},
  author={Ma, Xinyin and Fang, Gongfan and Wang, Xinchao},
  booktitle={Proceedings of the IEEE/CVF conference on computer vision and pattern recognition},
  pages={15762--15772},
  year={2024}
}

@article{wu2024ptq4dit,
  title={Ptq4dit: Post-training quantization for diffusion transformers},
  author={Wu, Junyi and Wang, Haoxuan and Shang, Yuzhang and Shah, Mubarak and Yan, Yan},
  journal={Advances in Neural Information Processing Systems},
  year={2024}
}

@inproceedings{li2025svdqunat,
  title={Svdqunat: Absorbing outliers by low-rank components for 4-bit diffusion models},
  author={Li, Muyang and Lin, Yujun and Zhang, Zhekai and Cai, Tianle and Li, Xiuyu and Guo, Junxian and Xie, Enze and Meng, Chenlin and Zhu, Jun-Yan and Han, Song},
  booktitle={The Thirteenth International Conference on Learning Representations},
  year={2025}
}

@article{vaswani2017attention,
  title={Attention is all you need},
  author={Vaswani, Ashish and Shazeer, Noam and Parmar, Niki and Uszkoreit, Jakob and Jones, Llion and Gomez, Aidan N and Kaiser, {\L}ukasz and Polosukhin, Illia},
  journal={Advances in neural information processing systems},
  volume={30},
  year={2017}
}

@inproceedings{zeng2025light,
  title={Light-t2m: A lightweight and fast model for text-to-motion generation},
  author={Zeng, Ling-An and Huang, Guohong and Wu, Gaojie and Zheng, Wei-Shi},
  booktitle={Proceedings of the AAAI Conference on Artificial Intelligence},
  volume={39},
  number={9},
  pages={9797--9805},
  year={2025}
}

@inproceedings{zhao2024mobilediffusion,
  title={Mobilediffusion: Instant text-to-image generation on mobile devices},
  author={Zhao, Yang and Xu, Yanwu and Xiao, Zhisheng and Jia, Haolin and Hou, Tingbo},
  booktitle={European Conference on Computer Vision},
  pages={225--242},
  year={2024},
  organization={Springer}
}

@article{wang2024embedding,
  title={Embedding trajectory for out-of-distribution detection in mathematical reasoning},
  author={Wang, Yiming and Zhang, Pei and Yang, Baosong and Wong, Derek and Zhang, Zhuosheng and Wang, Rui},
  journal={Advances in Neural Information Processing Systems},
  volume={37},
  pages={42965--42999},
  year={2024}
}

@article{zheng2023toward,
  title={Toward understanding generative data augmentation},
  author={Zheng, Chenyu and Wu, Guoqiang and Li, Chongxuan},
  journal={Advances in neural information processing systems},
  volume={36},
  pages={54046--54060},
  year={2023}
}

@article{seawead2025seaweed,
  title={Seaweed-7b: Cost-effective training of video generation foundation model},
  author={Seawead, Team and Yang, Ceyuan and Lin, Zhijie and Zhao, Yang and Lin, Shanchuan and Ma, Zhibei and Guo, Haoyuan and Chen, Hao and Qi, Lu and Wang, Sen and others},
  journal={arXiv preprint arXiv:2504.08685},
  year={2025}
}

@incollection{risken1989fokker,
  title={Fokker-planck equation},
  author={Risken, Hannes},
  booktitle={The Fokker-Planck equation: methods of solution and applications},
  pages={63--95},
  year={1989},
  publisher={Springer}
}

@inproceedings{chi2021tohan,
  author    = {Haoang Chi and
               Feng Liu and
               Wenjing Yang and
               Long Lan and
               Tongliang Liu and
               Bo Han and
               William K. Cheung and
               James T. Kwok},
  title     = {{TOHAN:} {A} One-step Approach towards Few-shot Hypothesis Adaptation},
  booktitle = {NeurIPS},
  year      = {2021},
}

@inproceedings{chi2024unveiling,
 	 title={Unveiling causal reasoning in large language models: Reality or mirage?},
  	author={Chi, Haoang and Li, He and Yang, Wenjing and Liu, Feng and Lan, Long and Ren, Xiaoguang and Liu, Tongliang and Han, Bo},
  	journal={NeurIPS},
  	year={2024}
}

@article{zhong2024domain,
  title={Domain generalization enables general cancer cell annotation in single-cell and spatial transcriptomics},
  author={Zhong, Zhixing and Hou, Junchen and Yao, Zhixian and Dong, Lei and Liu, Feng and Yue, Junqiu and Wu, Tiantian and Zheng, Junhua and Ouyang, Gaoliang and Yang, Chaoyong and others},
  journal={Nature Communications},
  volume={15},
  number={1},
  pages={1929},
  year={2024},
  publisher={Nature Publishing Group UK London}
}

@article{guo2022deep,
  title={Deep transfer learning enables lesion tracing of circulating tumor cells},
  author={Guo, Xiaoxu and Lin, Fanghe and Yi, Chuanyou and Song, Juan and Sun, Di and Lin, Li and Zhong, Zhixing and Wu, Zhaorun and Wang, Xiaoyu and Zhang, Yingkun and others},
  journal={Nature Communications},
  volume={13},
  number={1},
  pages={7687},
  year={2022},
  publisher={Nature Publishing Group UK London}
}

@article{han2025trustworthy,
  title={Trustworthy machine learning: From data to models},
  author={Han, Bo and Yao, Jiangchao and Liu, Tongliang and Li, Bo and Koyejo, Sanmi and Liu, Feng and others},
  journal={Foundations and Trends{\textregistered} in Privacy and Security},
  volume={7},
  number={2-3},
  pages={74--246},
  year={2025},
  publisher={Now Publishers, Inc.}
}

@inproceedings{zheng2023out,
  title={Out-of-distribution detection learning with unreliable out-of-distribution sources},
  author={Zheng, Haotian and Wang, Qizhou and Fang, Zhen and Xia, Xiaobo and Liu, Feng and Liu, Tongliang and Han, Bo},
  booktitle={NeurIPS},
  year={2023}
}

@article{li2023learning,
  title={Learning defense transformations for counterattacking adversarial examples},
  author={Li, Jincheng and Zhang, Shuhai and Cao, Jiezhang and Tan, Mingkui},
  journal={Neural Networks},
  volume={164},
  pages={177--185},
  year={2023},
  publisher={Elsevier}
}
}


\newpage
\appendix

\begin{leftline}
	{
		\LARGE{\textsc{Appendix}}
	}
\end{leftline}

\etocdepthtag.toc{mtappendix}
\etocsettagdepth{mtchapter}{none}
\etocsettagdepth{mtappendix}{subsection}

{
    \hypersetup{linkcolor=black}
        \footnotesize\tableofcontents
}

\newpage

\section{Theoretical Analysis}
\label{sec:proofs}

\subsection{Basic Theorems and Corollaries Related to Statistics}

We start to provide some basic theoretical results, laying the foundation for establishing the bounds of the statistics in Appendix \ref{sec: Upper Bounds for Gradients} and \ref{sec: proof bound}.
\begin{thm}\label{lemma: non chi-squared}
Let $X \sim \chi^2(d, \varphi)$ follow a \textbf{noncentral} chi-squared distribution with $d$ degrees of freedom and noncentrality parameter $\varphi$. For any $t > 0$, the following tail bounds hold: 
\begin{align*}
    P\left\{X - (d + \varphi) \geq 2\sqrt{(d + 2\varphi)t} + 2t\right\} \leq e^{-t}, \\
    P\left\{X - (d + \varphi) \leq -2\sqrt{(d + 2\varphi)t}\right\} \leq e^{-t}.
\end{align*}
\end{thm}

\begin{proof}
The moment-generating function of $X$ satisfies 
\begin{align*}
    E[e^{sX}] = \frac{e^{\frac{\varphi s}{1 - 2s}}}{(1 - 2s)^{d/2}} \quad (s < 1/2).
\end{align*}
The log-moment generating function of $X - (d + \varphi)$ is: 
\begin{align}
    \log {\mathbb{E}[e^{s(X-(d+\varphi))}]}=& \log {\mathbb{E}[e^{sX}]}- s(d + \varphi) \nonumber\\
    =& - \frac{d}{2} \log(1 - 2s) +\frac{\varphi s}{1 - 2s} - s(d + \varphi). \label{eqn: log_E}
\end{align}
For $0<s<1/2$, we have 
\begin{align}\label{eqn: log_s}
    -s- \frac{1}{2} \log(1 - 2s) \leq \frac{s^2}{1-2s},
\end{align}
which holds because the function $\psi(s)=-s - \frac{1}{2} \log(1 - 2s) - \frac{s^2}{1-2s}$ satisfies $\psi'(s)=-1 +\frac{1}{1-2s} - \frac{2s-s^2}{(1-2s)^2}=-\frac{2s^2}{(1-2s)^2}\leq 0$, implying $\max_{0<s<1/2}\psi(s)<\psi(0^+)=0$, \ie, $\psi(s) \leq0$.

Substituting Eqn. (\ref{eqn: log_s}) into Eqn.(\ref{eqn: log_E}), we get 
\begin{align}\label{eqn: log_E_simply}
    \log {\mathbb{E}[e^{s(X-(d+\varphi))}]} \leq \frac{ds^2}{1-2s}+\frac{2\varphi s^2}{1-2s}=\frac{(d+2\varphi)s^2}{1-2s}.
\end{align}
According to the result in \cite{birge1998minimum}, if $\exists~v,c>0$, \st $\log \mathbb{E}[e^{uZ}] \leq \frac{vu^2}{2(1 - cu)}$, then for $\forall~t>0$, the following inequality holds:
\begin{align*}
    P(Z \geq ct + \sqrt{2vt}) \leq e^{-t}.
\end{align*}
Applying this result to Eqn. (\ref{eqn: log_E_simply}), we set $Z=X-(d+\varphi)$, $v=2(d+2\varphi)$ and $c=2$, then
\begin{align*}
    P\left\{X - (d + \varphi) \geq 2\sqrt{(d + 2\varphi)t} + 2t\right\} \leq e^{-t}.
\end{align*}
For $-1/2<s<0$, we have 
\begin{align}\label{eqn: log_s_m}
    -s- \frac{1}{2} \log(1 - 2s) \leq s^2,
\end{align}
which holds because the function $h(s)=-s - \frac{1}{2} \log(1 - 2s) - s^2$ satisfies $h'(s)=-1 + \frac{1}{1 - 2s}-2s = \frac{4s^2}{1 - 2s} \geq 0 $, implying $\max_{-1/2<s<0}\psi(s)<\psi(0^-)=0$, \ie, $h(s) \leq0$.

Substituting Eqn. (\ref{eqn: log_s_m}) into Eqn.(\ref{eqn: log_E}), we get 
\begin{align}\label{eqn: log_E_simply_lower}
    \log {\mathbb{E}[e^{s(X-(d+\varphi))}]} \leq ds^2+\frac{2\varphi s^2}{1-2s}=(d+\frac{2\varphi}{1-2s})s^2\leq (d+2\varphi)s^2.
\end{align}

According to the result in \cite{birge1998minimum}, if $\exists~v>0$, \st, $\log \mathbb{E}[e^{sZ}] \leq \frac{v s^2}{2}$, then for $\forall~t>0$, the following inequality holds:  
\begin{align*}
    P\left(Z \leq -\sqrt{2vt}\right) \leq e^{-t}. 
\end{align*}
Applying this result to Eqn. (\ref{eqn: log_E_simply_lower}), we set $Z=X-(d+\varphi)$, $v=2(d+2\varphi)$, then
\begin{align*}
     P\left\{X - (d + \varphi) \leq -2\sqrt{(d + 2\varphi)t}\right\} \leq e^{-t}.
\end{align*}
\end{proof}

\begin{coll}\label{coll: chi_bound}
    Given $X \sim \chi^2(d, \varphi)$, a noncentralchi-squared distribution with $d$ degrees of freedom and the noncentrality parameter $\varphi$, with probability at least $1 - \delta$, we have
    \begin{equation*}
        \left|X\right| \leq d + \varphi + \sqrt{4(d + 2\varphi) \log\left( \frac{2}{\delta} \right)}+2\log\left( \frac{2}{\delta} \right).
    \end{equation*}
\end{coll}
\begin{proof}
    By Theorem \ref{lemma: non chi-squared}, setting $e^{-t}=\frac{\delta}{2}$ yields the following inequalities:
    \begin{align*}
    &P\left\{X - (d + \varphi) \geq 2\sqrt{(d + 2\varphi) \log\left( \frac{2}{\delta} \right)} + 2 \log\left( \frac{2}{\delta} \right) \right\} \leq \frac{\delta}{2}, \\
    &P\left\{X - (d + \varphi) \leq -2\sqrt{(d + 2\varphi) \log\left( \frac{2}{\delta} \right)}\right\} \leq \frac{\delta}{2}.
\end{align*}
Combining these two inequalities, we obtain:
\begin{align*}
    P\left\{X \geq d {+} \varphi {+} 2\sqrt{(d {+} 2\varphi) \log\left( \frac{2}{\delta} \right)} {+} 2 \log\left( \frac{2}{\delta} \right) ~\mathrm{or}~ X \leq d {+} \varphi {-} 2\sqrt{(d {+} 2\varphi) \log\left( \frac{2}{\delta} \right)} \right\} {\leq} \delta.
\end{align*}
Taking the complement of the above event, we have:
\begin{align*}
    P\left\{ 
    d {+} \varphi {-} 2\sqrt{(d {+} 2\varphi) \log\left( \frac{2}{\delta} \right)} 
    \leq X \leq 
    d {+} \varphi {+} 2\sqrt{(d {+} 2\varphi) \log\left( \frac{2}{\delta} \right)} {+} 2 \log\left( \frac{2}{\delta} \right) \right
    \} {\geq} 1- \delta.
\end{align*}
By relaxing the lower bound of $X$, we conclude
\begin{align*}
    P\left\{ \left|X\right| \leq d {+} \varphi {+} 2\sqrt{(d {+} 2\varphi) \log\left( \frac{2}{\delta} \right)} {+} 2 \log\left( \frac{2}{\delta} \right) \right
    \} {\geq} 1- \delta.
\end{align*}
\end{proof}

\subsection{Proof of Proposition \ref{prop:partial_t_approx}}

\textbf{Proposition \ref{prop:partial_t_approx}}.
\emph{ 
Under the brightness constancy assumption $p(\mathbf{x}+\Delta\mathbf{x}, t+\Delta t) \approx  p(\mathbf{x}, t)$ with small inter-frame motion ($\Delta t \to 0$) and inter-frame displacement ($\Delta \mathbf{x} \to 0$), we have  
\begin{equation}  
\partial_t \log p(\mathbf{x}, t) \approx -\frac{\nabla_{\mathbf{x}} \log p(\mathbf{x}, t) \cdot \Delta\mathbf{x}}{\Delta t}.
\end{equation} 
}  

\begin{proof}
We apply the Taylor expansion \(\log p(\mathbf{x} + \Delta\mathbf{x}, t + \Delta t)\) around \((\mathbf{x}, t)\) to first order:  
\begin{equation*}
    \log p(\mathbf{x} + \Delta\mathbf{x}, t + \Delta t) = \log p(\mathbf{x}, t) + \nabla_{\mathbf{x}} \log p(\mathbf{x}, t) \cdot \Delta\mathbf{x} + \partial_t \log p(\mathbf{x}, t) \cdot \Delta t + o(\|\Delta\mathbf{x}\|^2 + \Delta t^2),
\end{equation*}
where $o(\|\Delta\mathbf{x}\|^2 + \Delta t^2)$ represents higher-order infinitesimal terms.

By assumption, \(\log p(\mathbf{x} + \Delta\mathbf{x}, t + \Delta t) \approx \log p(\mathbf{x}, t)\). Subtracting \(\log p(\mathbf{x}, t)\) from both sides:  
\begin{equation*}
    \nabla_{\mathbf{x}} \log p(\mathbf{x}, t) \cdot \Delta\mathbf{x} + \partial_t \log p(\mathbf{x}, t) \cdot \Delta t + o(\|\Delta\mathbf{x}\|^2 + \Delta t^2) \approx 0.
\end{equation*}
Under $\Delta t \to 0$ and $\Delta\mathbf{x} \to 0$, we obtain:
\begin{equation*}
    \nabla_{\mathbf{x}} \log p(\mathbf{x}, t) \cdot \Delta\mathbf{x} + \partial_t \log p(\mathbf{x}, t) \cdot \Delta t \approx 0.
\end{equation*}
Rearranging terms gives:
\begin{equation*}
    \partial_t \log p(\mathbf{x}, t) \approx -\frac{\nabla_{\mathbf{x}} \log p(\mathbf{x}, t) \cdot \Delta\mathbf{x}}{\Delta t}.
\end{equation*}
\end{proof}

\newpage

\subsection{Derivations of Chain Rule to the Conservation of Probability Mass}
For $\frac{\partial p}{\partial t}+\nabla_{\mathbf{x}}\cdot\mathbf{J}=0$, we substitute $\mathbf{J}=p\mathbf{v}$ and then divide the entire equation by $p$ (which is strictly positive everywhere in its support), yielding: 
$$\frac{1}{p}\partial_t p+\frac{1}{p}\nabla_{\mathbf{x}}\cdot(p\mathbf{v})=0.$$ 
Applying the vector calculus product rule $\nabla_{\mathbf{x}}\cdot(p\mathbf{v})=p(\nabla_{\mathbf{x}} \cdot \mathbf{v})+\mathbf{v}\cdot(\nabla_{\mathbf{x}}p)$ and the chain rule of calculus, $\frac{1}{p}\frac{\partial p}{\partial t}=\partial_t\log p$ and $\frac{\nabla_{\mathbf{x}} p}{p}=\nabla_{\mathbf{x}}\log p$, we obtain Eqn.(\ref{eqn: log_p_with_delta_v}): 
$$\partial_t\log p+\nabla_{\mathbf{x}}\cdot\mathbf{v}+\mathbf{v}\cdot\nabla_{\mathbf{x}}\log p=0.$$ 
This transformation does not alter the underlying fluid constraint—it is a variable change making explicit how velocity couples to log-density's temporal and spatial gradients.

\subsection{Derivations of Gradients in NSG}

In the following, we provide the specific forms of the terms $\nabla_{\mathbf{x}} \log p(\mathbf{x}, t)$, $-\partial_t \log p(\mathbf{x}, t)+\lambda$ and $\mathbf{g}(\mathbf{x}, t)$ when the data are from Gaussian distributions, which will be used in Appendix \ref{sec: Proof of Proposition NSG}, \ref{sec: Upper Bounds for Gradients}, \ref{sec: proof bound}.

\begin{prop}\label{prop: score}
Given the real video distribution $p(\mathbf{x}, t)=\mathcal{N}(\mathbf{0}, \sigma(t)^2 \mathbf{I}_d)$ and the generated video distribution $q(\mathbf{y}, t){=}\mathcal{N}(\boldsymbol{\mu}, \sigma(t)^2 \mathbf{I}_d)$ with $\boldsymbol{\mu}{\neq} \mathbf{0}$ and $\sigma(t){\neq} \mathbf{0}$, the gradients $\nabla_{\mathbf{x}} \log p(\mathbf{x}, t)$ and $\nabla_{\mathbf{y}} \log p(\mathbf{y}, t)$ are:
\begin{align*}
    \nabla_{\mathbf{x}} \log p(\mathbf{x}, t) &= -\frac{\mathbf{x}}{\sigma(t)^2} \sim \mathcal{N}(\mathbf{0}, \frac{1}{\sigma(t)^2} \mathbf{I}_d), \\
    \nabla_{\mathbf{y}} \log p(\mathbf{y}, t) &= -\frac{\mathbf{y} }{\sigma(t)^2} \sim \mathcal{N}(-\frac{\boldsymbol{\mu}}{\sigma(t)}, \sigma(t)^2 \mathbf{I}_d).
\end{align*}
\end{prop}
\begin{proof}
Recall that for a Gaussian distribution $p(\mathbf{z})=\mathcal{N}(\boldsymbol{\nu}, \sigma^2 \mathbf{I}_d)$, the probability density function is
\begin{equation*}
    p(\mathbf{z}) = \frac{1}{(2\pi \sigma^2)^{d/2}} \exp\left(-\frac{\|\mathbf{z} - \boldsymbol{\nu}\|^2}{2\sigma^2}\right).
\end{equation*}
The log-density is
\begin{equation*}
    \log p(\mathbf{z}) = -\frac{d}{2} \log(2\pi \sigma^2) - \frac{\|\mathbf{z} - \boldsymbol{\nu}\|^2}{2\sigma^2}.
\end{equation*}
Thus, we have
\begin{equation*}
    \nabla_{\mathbf{z}} \log p(\mathbf{z}) = -\frac{\mathbf{z} - \boldsymbol{\nu}}{\sigma^2}.
\end{equation*}
For the real video distribution $p(\mathbf{x}, t) = \mathcal{N}(\mathbf{0}, \sigma(t)^2 \mathbf{I}_d)$, we have $\boldsymbol{\nu} = \mathbf{0}$. Taking the gradient \wrt $\mathbf{x}$:
\begin{align*}
    \nabla_{\mathbf{x}} \log p(\mathbf{x}, t) 
    = \nabla_{\mathbf{x}} \left(-\frac{\|\mathbf{x}\|^2}{2\sigma(t)^2}\right) 
    = -\frac{\mathbf{x}}{\sigma(t)^2} \sim \mathcal{N}\left(\mathbf{0}, \frac{1}{\sigma(t)^2} \mathbf{I}_d\right).
\end{align*}
For the generated video distribution $q(\mathbf{y}, t) = \mathcal{N}(\boldsymbol{\mu}, \sigma(t)^2 \mathbf{I}_d)$, evaluated under $p(\mathbf{y}, t)$), the gradient \wrt $\mathbf{y}$ is
\begin{align*}
    \nabla_{\mathbf{y}} \log p(\mathbf{y}, t) 
    = -\frac{\mathbf{y}}{\sigma(t)^2} \sim \mathcal{N}\left(-\frac{\boldsymbol{\mu}}{\sigma(t)^2}, \frac{1}{\sigma(t)^2}\mathbf{I}_d \right).
\end{align*}
\end{proof}

\begin{prop}\label{prop: score_dis}
Given the real video distribution $p(\mathbf{x}, t)=\mathcal{N}(\mathbf{0}, \sigma(t)^2 \mathbf{I}_d)$ and the generated video distribution $q(\mathbf{y}, t){=}\mathcal{N}(\boldsymbol{\mu}, \sigma(t)^2 \mathbf{I}_d)$ with $\boldsymbol{\mu}{\neq} \mathbf{0}$ and $\sigma(t){\neq} \mathbf{0}$, the partial derivatives $-\partial_t \log p(\mathbf{x}, t)$ and $-\partial_t \log p(\mathbf{y}, t)$ are:
\begin{align*}
    -\partial_t \log p(\mathbf{x}, t) &= \frac{d \dot{\sigma}(t)}{\sigma(t)} - \frac{\|\mathbf{x}\|^2 \dot{\sigma}(t)}{\sigma(t)^3} \sim \frac{d \dot{\sigma}(t)}{\sigma(t)} - \frac{\dot{\sigma}(t)}{\sigma(t)} \chi^2(d), \\
    -\partial_t \log p(\mathbf{y}, t) &= \frac{d \dot{\sigma}(t)}{\sigma(t)} - \frac{\|\mathbf{y}\|^2 \dot{\sigma}(t)}{\sigma(t)^3}\sim \frac{d \dot{\sigma}(t)}{\sigma(t)} - \frac{\dot{\sigma}(t)}{\sigma(t)} \chi^2(d, \varphi),
\end{align*}
where $\dot{\sigma}(t) \triangleq \frac{d}{dt}\sigma(t)$ and $\varphi = \frac{\|\boldsymbol{\mu}\|^2}{\sigma(t)^2}$. Here, $\chi^2(d)$ is the central chi-squared distribution and $\chi^2(d, \varphi)$ is the noncentral chi-squared distribution with noncentrality parameter $\varphi$ \cite{abdel1954approximate}.
\end{prop}

\begin{proof}
We first derive the expression for the real video distribution $p(\mathbf{x}, t)$. The log-density is
\begin{align*}
\log p(\mathbf{x}, t) &= -\frac{d}{2}\log(2\pi \sigma(t)^2) - \frac{\|\mathbf{x}\|^2}{2\sigma(t)^2}.
\end{align*}

Taking the time derivative $\partial_t$ (denoted by dot notation):
\begin{align*}
\partial_t \log p(\mathbf{x}, t) &= -\frac{d}{2} \cdot \frac{1}{2\pi \sigma(t)^2} \cdot 2\pi \cdot 2\sigma(t)\dot{\sigma}(t)  + \frac{\|\mathbf{x}\|^2}{\sigma(t)^3}\dot{\sigma}(t) \\
&= -\frac{d\dot{\sigma}(t)}{\sigma(t)} + \frac{\|\mathbf{x}\|^2 \dot{\sigma}(t)}{\sigma(t)^3}.
\end{align*}
Thus, we get
\begin{align*}
-\partial_t \log p(\mathbf{x}, t) = \frac{d\dot{\sigma}(t)}{\sigma(t)} - \frac{\|\mathbf{x}\|^2 \dot{\sigma}(t)}{\sigma(t)^3} \sim \frac{d \dot{\sigma}(t)}{\sigma(t)} - \frac{\dot{\sigma}(t)}{\sigma(t)} \chi^2(d),
\end{align*}
where the last formula is based on ${\|\frac{\mathbf{x}}{\sigma(t)}\|^2 } \sim \chi^2(d)$.

For the generated video distribution $q(\mathbf{y}, t) = \mathcal{N}(\boldsymbol{\mu}, \sigma(t)^2 \mathbf{I}_d)$, under $p(\mathbf{y},t)$ with $\varphi=\frac{\|\boldsymbol{\mu}\|^2}{\sigma(t)^2}$, we have
\begin{align*}
-\partial_t \log p(\mathbf{y}, t) = \frac{d\dot{\sigma}(t)}{\sigma(t)} - \frac{\|\mathbf{y}\|^2 \dot{\sigma}(t)}{\sigma(t)^3} \sim \frac{d \dot{\sigma}(t)}{\sigma(t)} - \frac{\dot{\sigma}(t)}{\sigma(t)} \chi^2(d, \varphi),
\end{align*}
where the last formula is based on ${\|\frac{\mathbf{y}}{\sigma(t)}\|^2 } \sim \chi^2(d, \varphi)$.
\end{proof}

\subsection{Proof of Proposition \ref{prop: NSG}}
\label{sec: Proof of Proposition NSG}

\textbf{Proposition \ref{prop: NSG}}.
\emph{
Let the real video distribution be $p(\mathbf{x}, t)=\mathcal{N}(\mathbf{0}, \sigma(t)^2 \mathbf{I}_d)$ and the generated video distribution be $q(\mathbf{y}, t){=}\mathcal{N}(\boldsymbol{\mu}, \sigma(t)^2 \mathbf{I}_d)$, respectively, where $\mathbf{I}_d \in \mathbb{R}^{d\times d}$ is an identity matrix and $\boldsymbol{\mu}{\neq} \mathbf{0} \in \mathbb{R}^d$ is the distribution shift and $\sigma(t){\neq} \mathbf{0}$, the NSG $\mathbf{g}(\mathbf{x}, t)$ and $\mathbf{g}(\mathbf{y}, t)$ satisfy:
\begin{align*}
   \mathbf{g}(\mathbf{x}, t) &= -\frac{\mathbf{x}/\sigma(t)^2}{ D_r(\mathbf{x})}, ~~-\frac{\mathbf{x} }{\sigma(t)^2} \sim \mathcal{N} \Big(\mathbf{0}, \sigma(t)^2 \mathbf{I}_d \Big), ~D_r(\mathbf{x}) \sim \lambda + \frac{d \dot{\sigma}(t)}{\sigma(t)} - \frac{\dot{\sigma}(t)}{\sigma(t)} \chi^2(d); \\
   \mathbf{g}(\mathbf{y}, t) &= -\frac{\mathbf{y}/\sigma(t)^2}{ D_f(\mathbf{y})},~-\frac{\mathbf{y} }{\sigma(t)^2} \sim \mathcal{N}\Big(-\frac{\boldsymbol{\mu}}{\sigma(t)}, \sigma(t)^2 \mathbf{I}_d \Big), ~D_f(\mathbf{y}) \sim \lambda +\frac{d \dot{\sigma}(t)}{\sigma(t)} - \frac{\dot{\sigma}(t)}{\sigma(t)} \chi^2(d, \varphi),
\end{align*}
where $D_r(\mathbf{x}) = \lambda +\frac{d \dot{\sigma}(t)}{\sigma(t)} - \frac{\|\mathbf{x}\|^2 \dot{\sigma}(t)}{\sigma(t)^3}$, $D_f(\mathbf{y}) = \lambda + \frac{d \dot{\sigma}(t)}{\sigma(t)} - \frac{\|\mathbf{y}\|^2 \dot{\sigma}(t)}{\sigma(t)^3}$, $\dot{\sigma}(t) \triangleq \frac{d}{dt}\sigma(t)$,  and $\varphi = \frac{\|\boldsymbol{\mu}\|^2}{\sigma(t)^2}$, $\chi^2(d)$ is the central chi-squared distribution with $d$ degrees of freedom and $\chi^2(d, \varphi)$ is the noncentral chi-squared distribution with noncentrality parameter $\varphi$ and $d$ degrees of freedom \cite{abdel1954approximate}.
}

\begin{proof}
According to the definition of NSG,
\begin{align}\label{eqn: NSG_def_app}
   \mathbf{g}(\mathbf{x}, t) = \frac{\nabla_{\mathbf{x}} \log p(\mathbf{x}, t)}{-\partial_t \log p(\mathbf{x}, t)+\lambda},
\end{align}
we can substitute the results of $\nabla_{\mathbf{x}} \log p(\mathbf{x}, t)$ in Proposition \ref{prop: score} and $-\partial_t \log p(\mathbf{x}, t)$ in Proposition \ref{prop: score_dis} into Eqn. (\ref{eqn: NSG_def_app}) and directly contribute to the results.
\end{proof}

\subsection{Derivations of Upper Bounds for Gradients}
\label{sec: Upper Bounds for Gradients}

Next, we present some propositions on the upper bounds that will be used in Appendix \ref{sec: proof bound}.

\begin{prop}\label{prop: Dx_dis}
Given the real video distribution $p(\mathbf{x}, t)=\mathcal{N}(\mathbf{0}, \sigma(t)^2 \mathbf{I}_d)$ and the generated video distribution $q(\mathbf{y}, t){=}\mathcal{N}(\boldsymbol{\mu}, \sigma(t)^2 \mathbf{I}_d)$ with $\boldsymbol{\mu}{\neq} \mathbf{0}$ and $\sigma(t){\neq} \mathbf{0}$, let $D_r(\mathbf{x}) = \lambda - \partial_t \log p(\mathbf{x}, t) $ and $D_f(\mathbf{y}) = \lambda - \partial_t \log p(\mathbf{y}, t)$, with probability at least $1 - \delta$, we have
\begin{align*}
    \left|D_r(\mathbf{x}) - D_f(\mathbf{y})\right| \leq \varphi + 2\sqrt{(d + 2\varphi)\log\frac{4}{\delta}} + 2\sqrt{d\log\frac{4}{\delta}} +  2\log\frac{4}{\delta},
\end{align*}
where $\varphi = \frac{\|\boldsymbol{\mu}\|^2}{\sigma(t)^2}$.
\end{prop}

\begin{proof}
From Proposition \ref{prop: score_dis}, we obtain  
\begin{align*}
    \left| D_r(\mathbf{x}) - D_f(\mathbf{y}) \right| = \frac{|\dot{\sigma}(t)|}{\sigma(t)^3} \left| \|\mathbf{x}\|^2 - \|\mathbf{y}\|^2 \right|. 
\end{align*}
where $\dot{\sigma}(t) \triangleq \frac{d}{dt}\sigma(t)$.

Let $Z = \frac{\|\mathbf{x}\|^2}{\sigma(t)^2} \sim \chi^2(d)$ and $W = \frac{\|\mathbf{y}\|^2}{\sigma(t)^2} \sim \chi^2(d, \varphi)$, where $\varphi = \frac{\|\boldsymbol{\mu}\|^2}{\sigma(t)^2}$. The difference becomes
\begin{align*}
    \left| D_r(\mathbf{x}) - D_f(\mathbf{y}) \right| = \frac{|\dot{\sigma}(t)|}{\sigma(t)} \left| W - Z \right|.
\end{align*}  
To bound $|W - Z|$, we use concentration inequalities for chi-squared distributions by Theorem \ref{lemma: non chi-squared}. For $ Z \sim \chi^2(d)$, we have
\begin{align*}
&P\left\{Z - d \geq 2\sqrt{d \log\tfrac{4}{\delta}} + 2\log\tfrac{4}{\delta}\right\} \leq \frac{\delta}{4}, \\
&P\left\{Z - d \leq -2\sqrt{d \log\tfrac{4}{\delta}}\right\} \leq \frac{\delta}{4}.
\end{align*}
Combining these two events, we obtain 
\begin{align}\label{enq: p_zd}
    P\left\{-2\sqrt{d \log\tfrac{4}{\delta}} \leq Z - d \leq 2\sqrt{d \log\tfrac{4}{\delta}} + 2\log\tfrac{4}{\delta} \right\} \geq 1- \frac{\delta}{2}.
\end{align}
Similarly, for $W \sim \chi^2(d, \varphi)$, we have
\begin{align}\label{enq: p_wd}
    P\left\{\varphi-2\sqrt{(d + 2\varphi) \log\tfrac{4}{\delta}} \leq W - d \leq \varphi + 2\sqrt{(d + 2\varphi) \log\tfrac{4}{\delta}} + 2\log\tfrac{4}{\delta} \right\} \geq 1- \frac{\delta}{2}.
\end{align}
Combining the bounds in Eqn. (\ref{enq: p_wd}) and (\ref{enq: p_zd}), we have with probability \( 1 - \delta \):  
\begin{align*}
    |W - Z| \leq \varphi + 2\sqrt{(d + 2\varphi)\log\tfrac{4}{\delta}} + 2\sqrt{d\log\tfrac{4}{\delta}} +  2\log\tfrac{4}{\delta}. 
\end{align*}
Substituting this into the expression for \( |D_r(\mathbf{x}) - D_f(\mathbf{y})| \) completes the proof. 
\end{proof}

\begin{prop}\label{prop: chi_xy}
Given the real video distribution $p(\mathbf{x}, t)=\mathcal{N}(\mathbf{0}, \sigma(t)^2 \mathbf{I}_d)$ and the generated video distribution $q(\mathbf{y}, t){=}\mathcal{N}(\boldsymbol{\mu}, \sigma(t)^2 \mathbf{I}_d)$ with $\boldsymbol{\mu}{\neq} \mathbf{0}$ and $\sigma(t){\neq} \mathbf{0}$,  the following inequalities hold with probability at least \( 1 - \delta \):  
\begin{align*}
&\text{1)} \quad \frac{\|\mathbf{x}\|^2}{\sigma(t)^2} \leq d + \sqrt{4d \log\left( \frac{2}{\delta} \right)} + 2\log\left( \frac{2}{\delta} \right), \\
&\text{2)} \quad \frac{\|\mathbf{y}\|^2}{\sigma(t)^2} \leq d + \varphi + \sqrt{4(d + 2\varphi) \log\left( \frac{2}{\delta} \right)} + 2\log\left( \frac{2}{\delta} \right), \\
&\text{3)} \quad \frac{\|\mathbf{x} - \mathbf{y}\|^2}{2\sigma(t)^2} \leq d + \frac{\varphi}{2} + \sqrt{4(d + \varphi) \log\left( \frac{2}{\delta} \right)} + 2\log\left( \frac{2}{\delta} \right),
\end{align*}
where $\varphi = \frac{\|\boldsymbol{\mu}\|^2}{\sigma(t)^2}$.
\end{prop}

\begin{proof}
1) Since $\mathbf{x} \sim \mathcal{N}(\mathbf{0}, \sigma(t)^2 \mathbf{I}_d)$, $\frac{\|\mathbf{x}\|^2}{\sigma(t)^2}$ follows a central chi-squared distribution $\chi^2(d)$. By the concentration inequality for central chi-squared distributions (Corollary \ref{coll: chi_bound}), with probability $1 - \delta$:  
\begin{align*}
    \frac{\|\mathbf{x}\|^2}{\sigma(t)^2} \leq d + \sqrt{4d \log\left( \frac{2}{\delta} \right)} + 2\log\left( \frac{2}{\delta} \right).
\end{align*}

2) For \( \mathbf{y} \sim \mathcal{N}(\boldsymbol{\mu}, \sigma(t)^2 \mathbf{I}_d) \), $\frac{\|\mathbf{y}\|^2}{\sigma(t)^2}$ follows a noncentral chi-squared distribution $\chi^2(d, \varphi)$, where $\varphi = \frac{\|\boldsymbol{\mu}\|^2}{\sigma(t)^2}$. By Corollary \ref{coll: chi_bound}, with probability $1 - \delta$, we have  
\begin{align*}
    \frac{\|\mathbf{y}\|^2}{\sigma(t)^2} \leq d + \varphi + \sqrt{4(d + 2\varphi) \log\left( \frac{2}{\delta} \right)} + 2\log\left( \frac{2}{\delta} \right).
\end{align*}

3) Since \( \mathbf{x} \sim \mathcal{N}(\mathbf{0}, \sigma(t)^2 \mathbf{I}_d) \) and \( \mathbf{y} \sim \mathcal{N}(\boldsymbol{\mu}, \sigma(t)^2 \mathbf{I}_d) \), their difference \( \mathbf{z} \) satisfies:  
\begin{align*}
    \mathbf{x} - \mathbf{y} \sim \mathcal{N}(-\boldsymbol{\mu}, 2\sigma(t)^2 \mathbf{I}_d).
\end{align*}
Thus, $\frac{\|\mathbf{x} - \mathbf{y}\|^2}{2\sigma(t)^2}$ follows a noncentral chi-squared distribution $\chi^2(d, \varphi/2)$, where 
$\varphi/2 = \frac{\|\boldsymbol{\mu}\|^2}{2\sigma(t)^2}$.

By Corollary \ref{coll: chi_bound}, with probability $1 - \delta$, we have
\begin{align*}
    \frac{\|\mathbf{x} - \mathbf{y}\|^2}{2\sigma(t)^2} \leq d + \frac{\varphi}{2} + \sqrt{4\left(d + \varphi\right) \log\left( \frac{2}{\delta} \right)} + 2\log\left( \frac{2}{\delta} \right).
\end{align*} 

\end{proof}


\subsection{Proof of Theorem \ref{thm: bound of NSG}}
\label{sec: proof bound}

Given a video $\mathbf{x} \in \mathbb{R}^{T \times d}$, its NSG Feature is $\mathbf{G}(\mathbf{x})=\{\mathbf{g}(\mathbf{x},t)\}_{t=1}^T$, where $\mathbf{g}(\mathbf{x},t)$ is defined as:
\begin{equation}
    \mathbf{g}(\mathbf{x}, t) = \frac{\nabla_{\mathbf{x}} \log p(\mathbf{x}, t)}{-\partial_t \log p(\mathbf{x}, t)+\lambda},
\end{equation} 
Here, $\lambda >0$ and $p(\mathbf{x}, t)$ is the probability density of the real video parameterized by time $t$.

Note that Theorem \ref{thm: bound of NSG} share a common lower bound $C$ on both $D_r(\mathbf{x}) {=} \lambda {-} \partial_t \log p(\mathbf{x}, t)$ and $D_f(\mathbf{y}) {=} \lambda {-} \partial_t \log p(\mathbf{y}, t)$. We first derive the conditions for $D_r(\mathbf{x}){>}C{>}0$ and $D_f(\mathbf{y}){>}C{>}0$ to hold in Proposition \ref{prop: assumption}. The same analytical approach can be extended to examine other cases.

\begin{prop}\label{prop: assumption}
Let the real video distribution be  $\mathbf{x} \sim \mathcal{N}(\mathbf{0}, \sigma(t)^2 \mathbf{I}_d)$ and the generated video distribution be $\mathbf{y} \sim \mathcal{N}(\boldsymbol{\mu}, \sigma(t)^2 \mathbf{I}_d) $, respectively, where $\mathbf{I}_d \in \mathbb{R}^{d\times d}$ is an identity matrix and $\boldsymbol{\mu} \neq \mathbf{0} \in \mathbb{R}^d$ is the distribution shift, we have $D_r(\mathbf{x})>C>0$ and $D_f(\mathbf{y})>C>0$ with probability at least $1 - \delta$, provided $C$ and $\lambda$ meet the following conditions:  

1) \quad \text{Case 1} ($\frac{\dot{\sigma}(t)}{\sigma(t)} >0$): 
\begin{align*}
    &C=\lambda - \frac{\dot{\sigma}(t)}{\sigma(t)} \cdot \varphi+ \frac{2\dot{\sigma}(t)}{\sigma(t)}\sqrt{d \log\left( \frac{2}{\delta} \right)} + \frac{2\dot{\sigma}(t)}{\sigma(t)}\log\left( \frac{2}{\delta} \right),\\
    &\lambda > \frac{\dot{\sigma}(t)}{\sigma(t)}(d + \varphi)- \frac{2\dot{\sigma}(t)}{\sigma(t)}\sqrt{d \log\left( \frac{2}{\delta} \right)} - \frac{2\dot{\sigma}(t)}{\sigma(t)}\log\left( \frac{2}{\delta} \right).    
\end{align*}
where $\varphi = \frac{\|\boldsymbol{\mu}\|^2}{\sigma(t)^2}$.

2) \quad \text{Case 2} ( $\frac{\dot{\sigma}(t)}{\sigma(t)} <0$):
\begin{align*}
    &C=\lambda - \frac{2\dot{\sigma}(t)}{\sigma(t)}\sqrt{d \log\left( \frac{2}{\delta} \right)},\\
    &\lambda > \frac{2\dot{\sigma}(t)}{\sigma(t)}\sqrt{d \log\left( \frac{2}{\delta} \right)}.    
\end{align*}
\end{prop}

\begin{proof}
Let $W = \frac{\|\mathbf{y}\|^2}{\sigma(t)^2} \sim \chi^2(d, \varphi)$, where $\varphi = \frac{\|\boldsymbol{\mu}\|^2}{\sigma(t)^2}$. From Theorem \ref{lemma: non chi-squared}, we have
\begin{align}
&P\left\{W - (d + \varphi) \geq 2\sqrt{(d + 2\varphi) \log\left( \frac{2}{\delta} \right)} + 2 \log\left( \frac{2}{\delta} \right) \right\} \leq \frac{\delta}{2}, \label{eqn: lower_W}\\
&P\left\{W - (d + \varphi) \leq -2\sqrt{(d + 2\varphi) \log\left( \frac{2}{\delta} \right)}\right\} \leq \frac{\delta}{2}. \label{eqn: upper_W}
\end{align}

1) \quad \textbf{Case 1} ($\frac{\dot{\sigma}(t)}{\sigma(t)} >0$):

Substituting $D_f(\mathbf{y}) = \lambda - \frac{\dot{\sigma}(t)}{\sigma(t)}(W-d)$ into the bound Eqn. (\ref{eqn: lower_W}):
\begin{align*}
&P\left\{D_f(\mathbf{y}) \leq \lambda - \frac{\dot{\sigma}(t)}{\sigma(t)} \cdot \varphi + \frac{2\dot{\sigma}(t)}{\sigma(t)}\sqrt{(d + 2\varphi) \log\left( \frac{2}{\delta} \right)} + \frac{2\dot{\sigma}(t)}{\sigma(t)}\log\left( \frac{2}{\delta} \right) \right\} \leq \frac{\delta}{2}.
\end{align*}
Thus, the following inequalities hold with probability at least $1 - \delta/2$:
\begin{align*}
    D_f(\mathbf{y}) &\geq \lambda - \frac{\dot{\sigma}(t)}{\sigma(t)} \cdot \varphi + \frac{2\dot{\sigma}(t)}{\sigma(t)}\sqrt{(d + 2\varphi) \log\left( \frac{2}{\delta} \right)} + \frac{2\dot{\sigma}(t)}{\sigma(t)}\log\left( \frac{2}{\delta} \right), \\
    D_r(\mathbf{x}) &\geq \lambda + \frac{2\dot{\sigma}(t)}{\sigma(t)}\sqrt{d \log\left( \frac{2}{\delta} \right)} + \frac{2\dot{\sigma}(t)}{\sigma(t)}\log\left( \frac{2}{\delta} \right).
\end{align*}
To ensure $D_r(\mathbf{x})>C>0$ and $D_f(\mathbf{y})>C>0$, we can select $C$ and $\lambda$ as:
\begin{align*}
    &C=\lambda - \frac{\dot{\sigma}(t)}{\sigma(t)} \cdot \varphi+ \frac{2\dot{\sigma}(t)}{\sigma(t)}\sqrt{d \log\left( \frac{2}{\delta} \right)} + \frac{2\dot{\sigma}(t)}{\sigma(t)}\log\left( \frac{2}{\delta} \right),\\
    &\lambda > \frac{\dot{\sigma}(t)}{\sigma(t)}(d + \varphi)- \frac{2\dot{\sigma}(t)}{\sigma(t)}\sqrt{d \log\left( \frac{2}{\delta} \right)} - \frac{2\dot{\sigma}(t)}{\sigma(t)}\log\left( \frac{2}{\delta} \right).    
\end{align*}

2) \quad \textbf{Case 2} ($\frac{\dot{\sigma}(t)}{\sigma(t)} <0$):

Substituting $D_f(\mathbf{y}) = \lambda - \frac{\dot{\sigma}(t)}{\sigma(t)}(W-d)$ into the bound Eqn. (\ref{eqn: upper_W}):
\begin{align*}
&P\left\{D_f(\mathbf{y}) \leq \lambda  - \frac{\dot{\sigma}(t)}{\sigma(t)} \cdot \varphi - \frac{2\dot{\sigma}(t)}{\sigma(t)}\sqrt{(d + 2\varphi) \log\left( \frac{2}{\delta} \right)}  \right\} \leq \frac{\delta}{2}.
\end{align*}
Thus, the following inequalities hold with probability at least $1 - \delta/2$:
\begin{align*}
    D_f(\mathbf{y}) &\geq \lambda - \frac{\dot{\sigma}(t)}{\sigma(t)} \cdot \varphi - \frac{2\dot{\sigma}(t)}{\sigma(t)}\sqrt{(d + 2\varphi) \log\left( \frac{2}{\delta} \right)}, \\
    D_r(\mathbf{x}) &\geq \lambda - \frac{2\dot{\sigma}(t)}{\sigma(t)}\sqrt{d \log\left( \frac{2}{\delta} \right)}.
\end{align*}
To ensure $D_r(\mathbf{x})>C>0$ and $D_f(\mathbf{y})>C>0$, we can select $C$ and $\lambda$ as:
\begin{align*}
    &C=\lambda - \frac{2\dot{\sigma}(t)}{\sigma(t)}\sqrt{d \log\left( \frac{2}{\delta} \right)},\\
    &\lambda > \frac{2\dot{\sigma}(t)}{\sigma(t)}\sqrt{d \log\left( \frac{2}{\delta} \right)}.    
\end{align*}
\end{proof}

Building upon the established Propositions \ref{prop: NSG},  \ref{prop: Dx_dis}, \ref{prop: chi_xy} and \ref{prop: assumption}, we next prove Theorem \ref{thm: bound of NSG}.

\textbf{Theorem \ref{thm: bound of NSG}}.
\emph{
    Let the real video distribution be  $\mathbf{x} {\sim} \mathcal{N}(\mathbf{0}, \sigma(t)^2 \mathbf{I}_d)$ and the generated video distribution be $\mathbf{y} {\sim} \mathcal{N}(\boldsymbol{\mu}, \sigma(t)^2 \mathbf{I}_d) $, respectively, where $\mathbf{I}_d {\in} \mathbb{R}^{d\times d}$ is an identity matrix and $\boldsymbol{\mu} {\neq} \mathbf{0} {\in} \mathbb{R}^d$ is the distribution shift. Given $\mathbf{G}(\mathbf{x}){=}\{\mathbf{g}(\mathbf{x},t)\}_{t=1}^T$,  denote $\varphi {=} {\|\boldsymbol{\mu}\|^2}/{\sigma(t)^2}$ and assume $\left|-\partial_t \log p(\mathbf{x}, t)+\lambda \right|\geq C>0$ and $\left|-\partial_t \log p(\mathbf{y}, t)+\lambda \right|\geq C>0$, with probability at least $1 - \delta$, we have
\begin{align*}
    \left\| \mathbf{G}(\mathbf{x}) {-} \mathbf{G}(\mathbf{y}) \right\|^2 {\leq} \mathcal{O} \left(\frac{T}{C^4 \sigma(t)^2} \! \left[\varphi d +d^2 + \varphi + \log \frac{T}{\delta} \! \cdot \! (\varphi + d) + \log^2 \frac{T}{\delta}\right]\right).
\end{align*}
}

\begin{proof}
Based on the definition of $\mathbf{G}(\mathbf{x})$, we have
\begin{align}
    \left\| \mathbf{G}(\mathbf{x}) - \mathbf{G}(\mathbf{y}) \right\|^2 = \sum_{t=1}^T \left\| \mathbf{g}(\mathbf{x}, t) - \mathbf{g}(\mathbf{y}, t) \right\|^2.
\end{align}
Next, we focus on the bound of $\mathbf{g}(\mathbf{x}, t) - \mathbf{g}(\mathbf{y}, t)$.

Let $D_r(\mathbf{x}) = \lambda - \partial_t \log p(\mathbf{x}, t) $ and $D_f(\mathbf{y}) = \lambda - \partial_t \log p(\mathbf{y}, t)$, by Propositions \ref{prop: score} and \ref{prop: score_dis}, we have 
\begin{align}
    &\left\| \mathbf{g}(\mathbf{x}, t) - \mathbf{g}(\mathbf{y}, t) \right\|^2 \nonumber\\
    =&\left\| \frac{\mathbf{x}/\sigma(t)^2}{ D_r(\mathbf{x})} - \frac{\mathbf{y}/\sigma(t)^2}{ D_f(\mathbf{y})} \right\|^2 \nonumber\\
    =&\left\| \frac{\mathbf{x}/\sigma(t)^2}{ D_r(\mathbf{x})} - \frac{\mathbf{x}/\sigma(t)^2}{ D_f(\mathbf{y})} + \frac{\mathbf{x}/\sigma(t)^2}{ D_f(\mathbf{y})} - \frac{\mathbf{y}/\sigma(t)^2}{ D_f(\mathbf{y})} \right\|^2 \nonumber\\
    \leq&  2\left\| \frac{\mathbf{x}/\sigma(t)^2}{ D_r(\mathbf{x})} - \frac{\mathbf{x}/\sigma(t)^2}{ D_f(\mathbf{y})}\right\|^2 + 2\left\|\frac{\mathbf{x}/\sigma(t)^2}{ D_f(\mathbf{y})} - \frac{\mathbf{y}/\sigma(t)^2}{ D_f(\mathbf{y})} \right\|^2 \nonumber\\
    =& 2 \left(\frac{1}{ D_r(\mathbf{x})} - \frac{1}{ D_f(\mathbf{y})} \right)^2 \cdot  \frac{\left\|\mathbf{x}\right\|^2}{\sigma(t)^4} 
    + 
    2\left(\frac{1}{D_f(\mathbf{y})} \right)^2 \cdot\frac{\left\|\mathbf{x}-\mathbf{y}\right\|^2}{ \sigma(t)^4} \nonumber\\
    =& 2 \left(\frac{D_f(\mathbf{y}) - D_r(\mathbf{x})}{ D_r(\mathbf{x}) D_f(\mathbf{y})} \right)^2 \cdot  \frac{\left\|\mathbf{x}\right\|^2}{\sigma(t)^4} 
    + 
    2\left(\frac{1}{D_f(\mathbf{y})} \right)^2 \cdot\frac{\left\|\mathbf{x}-\mathbf{y}\right\|^2}{ \sigma(t)^4} \nonumber\\
    \leq& 2 \frac{\left|D_f(\mathbf{y}) - D_r(\mathbf{x})\right|^2}{ C^4}  \cdot  \frac{\left\|\mathbf{x}\right\|^2}{\sigma(t)^4} 
    + 
    \frac{2}{C^2}  \cdot\frac{\left\|\mathbf{x}-\mathbf{y}\right\|^2}{ \sigma(t)^4}. \label{eqn: upper_g}
\end{align}

For the first term in Eqn. (\ref{eqn: upper_g}), according to Proposition \ref{prop: Dx_dis}, with probability at least $1 - 2\delta/3$, we have
\begin{align}\label{eqn: first}
    & 2 \frac{\left|D_f(\mathbf{y}) - D_r(\mathbf{x})\right|^2}{ C^4}  \cdot  \frac{\left\|\mathbf{x}\right\|^2}{\sigma(t)^4} 
    \nonumber\\
    \leq& \frac{2}{C^4 \sigma(t)^2} \left ( \varphi + 2\sqrt{(d + 2\varphi)\log\tfrac{12}{\delta}} + 2\sqrt{d\log\tfrac{12}{\delta}} +  2\log\tfrac{12}{\delta} \right) \cdot \left( d + \sqrt{4d \log\left( \tfrac{6}{\delta} \right)} + 2\log\left( \tfrac{6}{\delta} \right) \right) \nonumber\\
    \leq& \frac{2}{C^4 \sigma(t)^2} \left ( \varphi + 4\sqrt{(d + 2\varphi)\log\tfrac{12}{\delta}} +  2\log\tfrac{12}{\delta} \right) \cdot \left( d + \sqrt{4d \log\left( \tfrac{6}{\delta} \right)} + 2\log\left( \tfrac{6}{\delta} \right) \right).
\end{align}

For the second term in Eqn. (\ref{eqn: upper_g}), applying Proposition \ref{prop: chi_xy}, with probability at least $1 - \delta/3$, we have
\begin{align}\label{eqn: second}
    \frac{2}{C^2}  \cdot\frac{\left\|\mathbf{x}-\mathbf{y}\right\|^2}{ \sigma(t)^4} \leq \frac{4}{C^2 \sigma(t)^2} \left( d + \frac{\varphi}{2} + \sqrt{4(d + \varphi) \log\left( \frac{6}{\delta} \right)} + 2\log\left( \frac{6}{\delta} \right) \right).
\end{align}

For simplicity, let $L = \log\tfrac{12}{\delta}$. Since $\log\tfrac{6}{\delta} = L - \log 2 <L$, we have
\begin{align}
    2 \frac{\left|D_f(\mathbf{y}) - D_r(\mathbf{x})\right|^2}{ C^4}  \cdot  \frac{\left\|\mathbf{x}\right\|^2}{\sigma(t)^4} 
     &\leq \frac{2}{C^4 \sigma(t)^2} \left( \varphi + 4\sqrt{(d + 2\varphi)L} + 2L \right) \cdot \left( d + 2\sqrt{dL} + 2L \right)\nonumber\\
    &\leq \frac{2}{C^4 \sigma(t)^2} \left( \varphi + 2(d + 2\varphi) + L + 2L \right) \cdot \left( d + d+L + 2L \right) \nonumber\\
    &=\frac{2}{C^4 \sigma(t)^2} (5\varphi + 2d +3L) (2d+3L) \nonumber\\
    &=\frac{2}{C^4 \sigma(t)^2} \left(10\varphi d +4d^2 + 3L \cdot (5\varphi + 4d) + 9L^2\right). \label{eqn: first_L}
\end{align}
\begin{align}
\frac{2}{C^2}  \cdot\frac{\left\|\mathbf{x}-\mathbf{y}\right\|^2}{ \sigma(t)^4} &\leq \frac{4}{C^2 \sigma(t)^2} \left( d + \frac{\varphi}{2} + 2\sqrt{(d + \varphi)L} + 2L \right) \nonumber\\
&\leq \frac{4}{C^2 \sigma(t)^2} \left( d + \frac{\varphi}{2} + (d + \varphi)L + 2L \right) \nonumber\\
&=\frac{4}{C^2 \sigma(t)^2} \left( d + \frac{\varphi}{2} + L(d + \varphi + 2) \right). \label{eqn: second_L}
\end{align}

Combining Eqn. (\ref{eqn: first_L}) and (\ref{eqn: second_L}) and (\ref{eqn: upper_g}), and substituting $L=\log \frac{12}{\delta}$, with probability at least $1 - \delta$, we have 
\begin{align*}
    &\left\| \mathbf{g}(\mathbf{x}, t) - \mathbf{g}(\mathbf{y}, t) \right\|^2 \\
    \leq& \frac{2}{C^4 \sigma(t)^2} \left(10\varphi d +4d^2 + 3L \cdot (5\varphi + 4d) + 9L^2\right) + \frac{4}{C^2 \sigma(t)^2} \left( d + \frac{\varphi}{2} + L(d + \varphi + 2) \right).
\end{align*}

For further simplicity, note that $\frac{1}{C^2} \leq \frac{1}{C^4}$, we obtain
\begin{align}
    &\left\| \mathbf{g}(\mathbf{x}, t) - \mathbf{g}(\mathbf{y}, t) \right\|^2 \nonumber\\
    \leq& \frac{2}{C^4 \sigma(t)^2} \left(10\varphi d +4d^2 + 2d+ \varphi + L \cdot (17\varphi + 14d + 4) + 9L^2\right) \nonumber\\
    =&\frac{2}{C^4 \sigma(t)^2} \left(10\varphi d +4d^2 + 2d+ \varphi + \log \frac{12}{\delta} \cdot (17\varphi + 14d + 4) + 9 \log^2 \frac{12}{\delta}\right). \label{eqn: g_upper}
\end{align}

Summing over time steps $t=1,\ldots, T$, and replacing $\delta$ in Eqn. (\ref{eqn: g_upper}) into $\delta/T$ , we get:
\begin{align*}
    \left\| \mathbf{G}(\mathbf{x}) {-} \mathbf{G}(\mathbf{y}) \right\|^2 {\leq} \frac{2T}{C^4 \sigma(t)^2} \! \left(10\varphi d +4d^2 + 2d+ \varphi + \log \frac{12T}{\delta} \! \cdot \! (17\varphi + 14d + 4) + 9 \log^2 \frac{12T}{\delta}\right).
\end{align*}
Thus, we obtain the final results
\begin{align*}
    \left\| \mathbf{G}(\mathbf{x}) {-} \mathbf{G}(\mathbf{y}) \right\|^2 {\leq} \mathcal{O} \left(\frac{T}{C^4 \sigma(t)^2} \! \left[\varphi d +d^2 + \varphi + \log \frac{T}{\delta} \! \cdot \! (\varphi + d) + \log^2 \frac{T}{\delta}\right]\right).
\end{align*}
\end{proof}

\newpage
\section{More Related Work}

\textbf{Maximum Mean Discrepancy (MMD).}
Maximum Mean Discrepancy (MMD) is an effective metric for two-sample testing to assess whether two samples originate from the same distribution \cite{gerber2023kernel,gretton2012kernel,kalinke2023nystrom,liu2020learning,han2025trustworthy,muller1997integral,chi2021tohan}.
Originally introduced by M$\ddot{\mathrm{u}}$ller et al. \cite{muller1997integral} as an instance of an integral probability metric, MMD admits several sample-based estimators. Particularly, Gretton et al. \cite{gretton2012kernel} introduce the U-statistic estimator, which is unbiased for the squared MMD and achieves near-minimal variance among all unbiased alternatives. Further, Tolstikhin et al. \cite{tolstikhin2016minimax} derive lower bounds on the estimation error of MMD under finite samples when employing a radial universal kernel. 

Building upon the traditional formulation of MMD, recent advancements incorporate learnable kernels to enhance its discriminative capability. Liu et al. \cite{liu2020learning} develop a data-splitting strategy for kernel optimization and selection, effectively addressing the kernel adaptation challenges for complex-data scenarios. Kim et al. \cite{kim2022minimax} propose an adaptive two-sample test designed for comparing two H$\ddot{\mathrm{o}}$lder densities supported on the $d$-dimensional unit ball. In addition, Zhang et al. \cite{zhang2024detecting} introduce MMD-MP, a multi-population aware optimization framework to further improve the stability of kernel-based MMD training. At present, MMD has been extensively applied to distributional measurement and discrepancy detection tasks across both textual and visual modalities \cite{zhangs2023EPSAD,zhang2024detecting,song2025detecting}.

\section{More Details for Experiment Settings}
\label{sec: more details for exp}

\subsection{More Details on Datasets}
\label{sec:details on dataset}
\textbf{GenVideo}~\cite{chen2024demamba} is a large-scale benchmark for AI-generated video detection, comprising 1.22 million real videos and 1.05 million AI-generated videos. The real video collection aggregates content from three established datasets: MSR-VTT (web video clips)~\cite{xu2016msr}, Kinetics-400 (human action videos)~\cite{kay2017kinetics}, and Youku-mPLUG (diverse online videos captured from Youku.com)~\cite{xu2023youku}. The AI-generated portion features videos produced by 19 distinct generation models, including both open-source implementations (ZeroScope~\cite{zeroscope}, I2VGen-XL~\cite{zhang2023i2vgen}, SVD~\cite{blattmann2023stable}, VideoCrafter~\cite{chen2024videocrafter2}, DynamiCrafter~\cite{xing2024dynamicrafter}, Stable Diffusion(SD)~\cite{zhang2024pia}, SEINE~\cite{chen2023seine}, Latte~\cite{ma2024latte}, OpenSora~\cite{zheng2024open}, ModelScope~\cite{wang2023modelscope}, HotShot~\cite{hotshotxl}, Show-1~\cite{zhang2023show1}, Gen2~\cite{esser2023structure}, Crafter~\cite{chen2023videocrafter1}, Lavie~\cite{wang2024lavie}) and commercial closed-source systems (Pika, Sora, MoonValley, MorphStudio). This dataset  spans multiple generation paradigms, including text-to-video and image-to-video synthesis. Throughout all experiments, we filter videos with less than 8 frames and only uniformly sample 8 frames for each video during training and testing.

\subsection{More Details on Evaluation Metrics}
\label{sec:details on metric}
Video generation detection is inherently a binary classification task. Here, we introduce four fundamental evaluation metrics of binary classification:~\textbf {True Positive} (TP) means correctly predicted positive instances (ground truth is positive, prediction is positive).\textbf{True Negative} (TN) means correctly predicted negative instances (ground truth is negative, prediction is negative). \textbf{False Positive} (FP) means incorrectly predicted positive instances (ground truth is negative, prediction is positive). \textbf{False Negative} (FN) means incorrectly predicted negative instances (ground truth is positive, prediction is negative). 

\textbf{AUROC} denotes the Area Under the Receiver Operating Characteristic Curve~\cite{huang2005using,jimenez2012insights}, which is a widely used statistic for assessing the discriminatory capacity of distribution models. Formally, 
\begin{align*}
    AUROC = \int TP(t) FP(t)dt,
\end{align*}
where $TP(t) = TP(t)/(TP(t)+FN(t))$ is the true positive rate and $FP(t) = FP(t)/(FP(t)+TN(t))$ is false positive rate with a threshed $t$.

\textbf{Accuracy} (ACC) measures the model's overall correctness in classification by calculating the ratio of correctly predicted instances (both true positive and true negative) to the total instances.
\begin{align*}
    \text{Accuracy} = \frac{TP + TN}{TP + TN + FP + FN}.
\end{align*}

\textbf{Recall}~\cite{chinchor1993muc} evaluates the model's ability to identify all relevant instances of a class, measuring the proportion of true positives among all actual positive instances. It emphasizes minimizing false negatives, ensuring comprehensive coverage of positive cases.
\begin{align*}
\text{Recall} = \frac{TP}{TP + FN}.
\end{align*}
\textbf{Precision}~\cite{chinchor1993muc} quantifies the model's capability to avoid false positives by calculating the proportion of true positives among all predicted positive instances. It ensures reliability in positive predictions.
\begin{align*}
\text{Precision} = \frac{TP}{TP + FP}.
\end{align*}
\textbf{F1-score}~\cite{chinchor1993muc} balances Precision and Recall using their harmonic mean, providing a robust metric for scenarios with imbalanced class distributions. It penalizes extreme biases toward either precision or recall.
\begin{align*}
\text{F1-score} = 2 \times \frac{\text{Precision} \times \text{Recall}}{\text{Precision} + \text{Recall}}.
\end{align*}


\subsection{Implementation Details on NSG-VD}
\label{sec: Details on Our Method}

In our NSG-VD, we employ the pre-trained diffusion model $s_\theta$ of Guided Diffusion using the 256 × 256 unconditional checkpoint from the guided-diffusion library \footnote{\url{https:// github.com/openai/guided-diffusion}} following  \cite{dhariwal2021diffusion}. For a given video $\mathbf{x}$ at $t$-th frame, we  compute its score feature $\nabla_{\mathbf{x}} \log p(\mathbf{x}, t) $ by diffusing $\mathbf{x}_t$ at diffusion timestep $5/1,000$ and passing it through $s_\theta$. For the deep kernel $\phi_{\mathbf{G}}$, we employ a single-layer of Swin transformer \cite{liu2021swin}, mapping input features of dimension $8 \times 224 \times224$ to a $300$-dimensional output.

We conduct our experiments on a server with 1× NVIDIA RTX 3090 GPU using Python 3.10.17 and Pytorch 2.7.0. We use Adam optimizer \citep{kingma2014adam} to optimize the kernel parameters $\omega$ with batchsize 24, learning rate $0.0001$, weight decay $0.1$, $\sigma_\phi=0.1$ and $\sigma_{\Phi}=100$. For the testing, we set the decision threshold $\tau=1$ in Eqn. (\ref{eqn: decision}). The overall algorithms for training and testing are in Alg. \ref{alg: NSG_AD_train} and \ref{alg: NSG_AD testing}.

\subsection{Pseudo Code of NSG-VD}
\label{sec: Pseudo_code}

\begin{figure}[h]
\vspace{-1.5em}
\begin{minipage}[h]{0.47\linewidth}
    \centering
    \footnotesize
    \begin{algorithm}[H]\small
    \caption{Training deep kernel of MMD}
    \label{alg: NSG_AD_train}
 \begin{algorithmic}
    \STATE \textbf{Input:} Real and generated videos $S_\mmP^{tr}$, $S_\mmQ^{tr}$; \\
    $\omega \gets \omega_0$; $\lambda \gets 10^{-10}$; learning rate $\eta$;
    
    \FOR{$r = 1,2,\dots,r_{max}$}
    
    \STATE $k_\omega \gets$ kernel function using Eqn. (\ref{eqn:deep_kernel});
    
    \STATE $M(\omega) \gets \widehat{\mathrm{MPP}}_u(S^{tr}_\mmP, S^{tr}_\mmQ; k_\omega)$;
    
    \STATE $V_\lambda(\omega) \gets \hat{\sigma}^{2}(S^{tr}_\mmP, S^{tr}_\mmQ; k_\omega)$ using Eqn. (\ref{eqn:objMMDopt});
    
    \STATE $\hat J_\lambda(\omega) \gets {M(\omega)}/\sqrt{V_\lambda(\omega)}$;
    
    \STATE $\omega \gets \omega + \eta\nabla_{\textnormal{Adam}} \hat J_\lambda(\omega)$;
    
    \ENDFOR
    
    \STATE \textbf{Output:} $k_\omega^*$
\end{algorithmic}
    \end{algorithm}
\end{minipage}\hspace{0.27cm}
\begin{minipage}[h]{0.49\linewidth}
    \centering
    \footnotesize
    \begin{algorithm}[H]\small
    \caption{Detecting videos with NSG-VD}
    \label{alg: NSG_AD testing}
\begin{algorithmic}
\STATE \textbf{Input:} 
Referenced videos $S^{re}_{\mmP}$, testing videos $S^{te}_{\mmQ}$; decision $f(\cdot)$; deep kernel $k_\omega$; threshold $\tau$;
\FOR{$\bx_i$ in $S^{te}_{\mmQ}$}
\STATE $Q_i {\gets} \widehat{\operatorname{MMD}}_{b}^{2}\left(S^{re}_{\mmP}, \{{\bx_i}\};k_{\omega} \right)$ using Eqn. (\ref{eqn: MMD_NSG});
\STATE $f(\bx_i; S^{re}_{\mmP}, k_\omega, \tau) = \mathbb{I}\left(Q_i > \tau\right)$;
\STATE Obtaining $f(\bx_i)$ using Eqn. (\ref{eqn: decision});
\ENDFOR
\STATE \textbf{Output:}  Predictions $\{ f(\bx_i) \}$ of the testing set
\end{algorithmic}
    \end{algorithm}
\end{minipage}
\vspace{-0.85em}
\end{figure}

\newpage
\section{More Experimental Results}

\subsection{More Results on Standard Evaluation}
\label{sec: Results of random seed}

To demonstrate the statistical robustness and reproducibility of our proposed NSG-VD method, we report standard deviations of four metrics across 10 datasets with three different seeds (Table \ref{tab:Statistical Significance}). The table shows that our NSG-VD achieves consistently high performance with minimal variance (\eg, $0.41\%$ for Recall, $0.87\%$ for Accuracy), indicating strong reliability and repeatability of our methods.

\begin{table}[h]
\centering
\caption{Standard deviations of NSG-VD with three different random seeds (\%), where we train all models with $10, 000$ real and generated videos from Kinetics-400 and  SEINE, respectively.}
\label{tab:Statistical Significance}
    \vspace{-0.5em}
\begin{center}
\begin{threeparttable}
\resizebox{0.75\linewidth}{!}{
\begin{tabular}{c|ccccc}
\toprule
Dataset  & Recall & Accuracy & F1 & AUROC \\
\midrule
ModelScope & 92.78 $\pm$ 0.48 & 84.31 $\pm$ 1.88 & 85.54 $\pm$ 1.54 & 92.49 $\pm$ 3.07 \\
MorphStudio & 99.44 $\pm$ 0.96 & 86.53 $\pm$ 2.64 & 88.10 $\pm$ 2.14 & 97.08 $\pm$ 0.92 \\
MoonValley  & 100.00 $\pm$ 0.00 & 87.64 $\pm$ 0.64 & 89.00 $\pm$ 0.50 & 99.05 $\pm$ 0.30 \\
HotShot  & 100.00 $\pm$ 0.00 & 88.06 $\pm$ 2.30 & 89.35 $\pm$ 1.83 & 98.64 $\pm$ 0.49 \\
Show1 & 100.00 $\pm$ 0.00 & 88.89 $\pm$ 2.55 & 90.03 $\pm$ 2.08 & 96.96 $\pm$ 0.93 \\
Gen2 & 98.61 $\pm$ 1.73 & 87.08 $\pm$ 2.20 & 88.43 $\pm$ 1.86 & 95.27 $\pm$ 2.17 \\
Crafter & 99.72 $\pm$ 0.48 & 88.89 $\pm$ 1.05 & 89.98 $\pm$ 0.86 & 98.07 $\pm$ 0.41 \\
Lavie & 98.61 $\pm$ 0.96 & 88.75 $\pm$ 2.17 & 89.77 $\pm$ 1.85 & 95.32 $\pm$ 2.92 \\
Sora & 94.05 $\pm$ 1.03 & 84.12 $\pm$ 1.97 & 83.85 $\pm$ 1.56 & 93.10 $\pm$ 1.36 \\
WildScrape  & 91.67 $\pm$ 2.21 & 85.41 $\pm$ 2.60 & 86.29 $\pm$ 2.37 & 92.75 $\pm$ 0.31 \\ \midrule
\rowcolor{pink!30}Avg. & 97.49 $\pm$ 0.41 & 86.97 $\pm$ 0.87 & 88.04 $\pm$ 0.74 & 95.87 $\pm$ 0.87 \\
\bottomrule
\end{tabular}
}
\end{threeparttable}
\end{center}
\vspace{-1pt}
\end{table}

\subsection{More Results on Impact of Spatial Gradients and Temporal Derivatives}

To comprehensively analyze how spatial gradients and temporal derivatives contribute to detection performance across diverse generative paradigms, we include detailed results across 10 diverse generative paradigms in Table \ref{tab: detailed sptio-temporal}. The spatial gradients achieve strong performance across most generated models (\eg, $81.67\%$ Recall on ModelScope, $97.50\%$ Recall on MorphStudio), with only minor performance gaps on models like HotShot ($72.23\%$ Recall) and Show1 (74.17\% Recall).  

In contrast, the temporal derivatives show complementary strengths and relatively better detection capabilities on these challenging cases, notably achieving $75.40\%$ Recall on HotShot and $77.60\%$ Recall on Show1, where temporal dynamics (\eg, rapid motion transitions in HotShot dataset) may play a more pronounced role in exposing synthetic anomalies.

Notably, our proposed NSG-VD demonstrates superior reliability by integrating both components, achieving an average F1-score of $90.87\%$, a significant improvement over individual features. This highlights the complementary nature of spatiotemporal dynamics modeling, enabling consistent detection performance even when individual cues exhibit dataset-specific limitations.

\begin{table}[h]
    \caption{Impact of spatial gradients and temporal derivatives across 10 generated models (\%), where we train all models with $10, 000$ real and generated videos from Kinetics-400 and Pika, respectively.}
    \label{tab: detailed sptio-temporal}
    \begin{threeparttable}
    \resizebox{1.0\linewidth}{!}{
    \begin{tabular}{c|c|cccccccccc|c}
    \toprule
       \multirow{2}{*}{Method} & \multirow{2}{*}{Metric} & Model & Morph & Moon & \multirow{2}{*}{HotShot} & \multirow{2}{*}{Show1} & \multirow{2}{*}{Gen2} & \multirow{2}{*}{Crafter} & \multirow{2}{*}{Lavie} & \multirow{2}{*}{Sora} & Wild & \multirow{2}{*}{ Avg.} \\
       &  & Scope & Studio & Valley &  & & & & & & Scrape &  \\ \midrule
         \multirow{4}{*}{Spatial Gradients} & Recall & 81.67 & 97.50 & 100.00 & 72.23 & 74.17 & 98.33 & 98.13 & 77.12 & 98.21 & 82.50 & \cellcolor{blue!8}\underline{87.99} \\
         & Accuracy & 77.50 & 89.58 & 87.08 & 75.00 & 76.25 & 87.92 & 88.75 & 76.67 & 88.39 & 81.25 & \cellcolor{blue!8}\underline{82.84} \\
         & F1 & 78.40 & 90.35 & 88.56 & 74.36 & 75.74 & 89.06 & 89.73 & 76.86 & 89.43 & 81.48 & \cellcolor{blue!8}\underline{83.40} \\
         & AUROC & 87.32 & 98.43 & 99.99 & 78.76 & 82.72 & 98.98 & 98.64 & 84.15 & 99.01 & 90.52 & \cellcolor{blue!8}\underline{91.85} \\
        \midrule
        \multirow{4}{*}{Tmporal Derivative} & Recall & 52.60 & 40.20 & 58.40 & 75.40 & 77.60 & 49.00 & 72.40 & 47.20 & 60.71 & 70.00 & \cellcolor{blue!8}60.35 \\
         & Accuracy & 67.40 & 61.20 & 70.30 & 78.80 & 79.90 & 65.60 & 77.30 & 64.70 & 69.64 & 76.10 & \cellcolor{blue!8}71.09 \\
         & F1 & 61.74 & 50.89 & 66.29 & 78.05 & 79.43 & 58.75 & 76.13 & 57.21 & 66.67 & 74.55 & \cellcolor{blue!8}66.97 \\
         & AUROC & 72.97 & 68.64 & 81.46 & 87.09 & 87.80 & 74.48 & 84.24 & 72.97 & 76.28 & 83.62 & \cellcolor{blue!8}78.95 \\
        \midrule
        \multirow{4}{*}{\shortstack{NSG-VD\\(Ours)}} & Recall & 68.33 & 98.33 & 100.00 & 92.50 & 87.50 & 80.00 & 98.33 & 94.17 & 78.57 & 82.50 & \cellcolor{pink!30}\textbf{88.02} \\
         & Accuracy & 81.67 & 98.33 & 96.67 & 91.67 & 90.83 & 88.33 & 95.83 & 94.17 & 88.39 & 88.75 & \cellcolor{pink!30}\textbf{91.46} \\
         & F1 & 78.85 & 98.33 & 96.77 & 91.74 & 90.52 & 87.27 & 95.93 & 94.17 & 87.13 & 88.00 & \cellcolor{pink!30}\textbf{90.87} \\
         & AUROC & 92.26 & 98.66 & 98.15 & 94.45 & 96.38 & 94.83 & 98.16 & 97.41 & 96.40 & 94.73 & \cellcolor{pink!30}\textbf{96.14} \\
    \bottomrule
    \end{tabular}
    }
    \end{threeparttable}
\vspace{-1pt}
\end{table}

\newpage
\section{More Discussions on NSG-VD}

\subsection{Efficiency of NSG-VD}

\begin{figure*}[h]
    \vspace{-5pt}
    \begin{minipage}[h]{0.45\linewidth}
    \centering
    \includegraphics[height=0.7\textwidth]{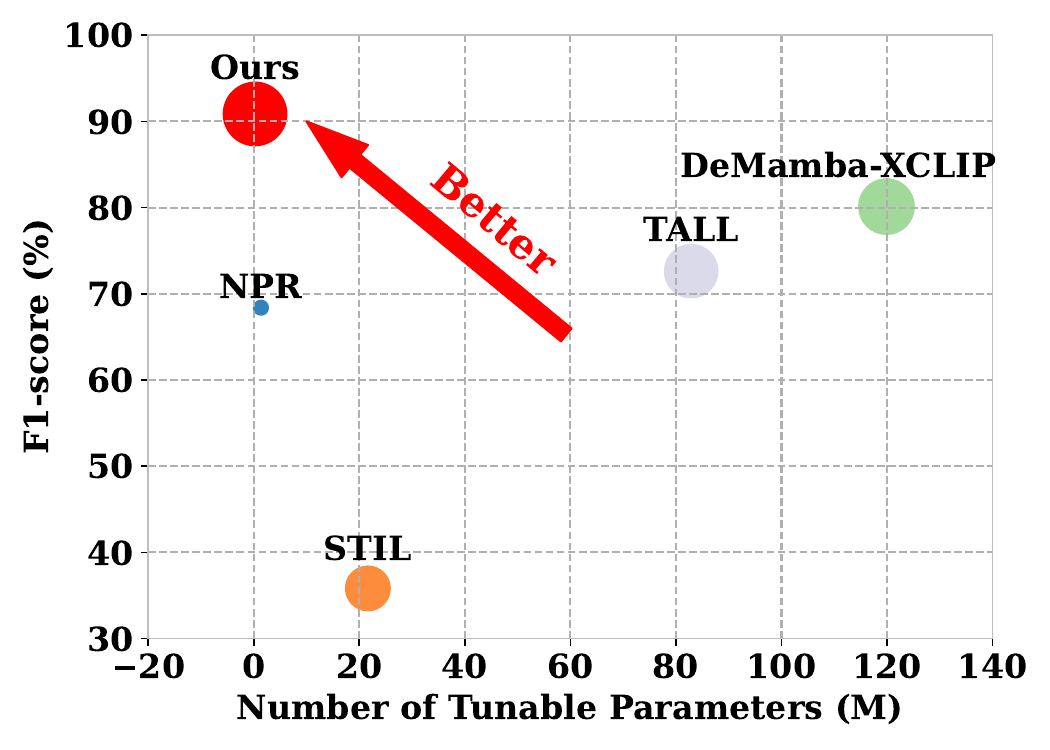}
    \end{minipage}
    \hspace{8pt}
    \begin{minipage}[h]{0.52\linewidth}
    \vspace{-1.0em}
    \begin{center}
    \begin{threeparttable}
    \large
    \renewcommand{\arraystretch}{1.8}
    \resizebox{1\linewidth}{!}{
    \begin{tabular}{l|ccc}
    \toprule
        Method & All Params(M) $\downarrow$ & Tun. Params(M) $\downarrow$ & F1(\%) $\uparrow$   \\
        \midrule
    DeMamba   & 119.89 & 119.89 & 80.12  \\
    TALL      & 82.89  & 82.89  & 72.63    \\
    STIL      & 21.63  & 21.63  &  35.82   \\
    NPR       & \textbf{1.37}   & 1.37   & 68.39  \\
    \rowcolor{pink!30}NSG-VD (Ours) & 527.45 & \textbf{0.25}  & \textbf{90.87} \\
    \bottomrule
    \end{tabular}
    }
    \end{threeparttable}
    \end{center}
    \end{minipage}
    \vspace{-0.5em} 
    \caption{Comparisons with baselines in terms of training costs and performance (\%), where we train all models with $10, 000$ real and generated videos from Kinetics-400 and Pika, respectively.}
    \label{fig:param}
\end{figure*}

\textbf{Performance vs. Training Cost Analysis.} 
We compare the number of trainable parameters and detection performance of NSG-VD against state-of-the-art baselines. As shown in Figure \ref{fig:param}, our NSG-VD achieves a $90.87\%$ F1-score with only $0.25$ M trainable parameters, demonstrating superior parameter efficiency compared to methods like DeMamba ($80.12\%$ F1-score, $119.89$ M parameters) and TALL ($72.63\%$ F1-score, $82.89$ M parameters). This demonstrates that NSG-VD’s physics-driven design enables high accuracy through minimal parameter tuning. 

Existing baselines struggle to balance parameter scale and performance. NPR achieves only $68.39\%$ F1-score despite its minimal trainable parameters ($1.37$ M), while STIL requires $21.63$ M parameters to attain $35.82\%$ F1-score—a suboptimal trade-off compared to NSG-VD’s superior performance with 100× fewer parameters. These results underscore the limitations of conventional artifact-driven frameworks in effective parameter budget utilization, further validating the importance of our physics-guided spatiotemporal modeling paradigm for AI-generated video detection.  

\textbf{Performance vs. Inference Time Analysis.} 
We evaluate the efficiency of our NSG-VD by analyzing both detection performance and computational overhead under the same setting as Table \ref{tab: standard Pika}.  As shown in Table \ref{tab: efficiency}, our NSG-VD achieves superior detection performance (\eg, $97.13\%$ Recall, $87.45\%$ F1-score) with a practical inference latency of $0.3605 s$ per video, which remains viable for non-real-time applications (\eg, judicial video evidence analysis) despite being slower than other baselines. This latency stems from our current implementation of pre-trained diffusion models for gradient estimation—a design choice prioritizing theoretical validation over computational optimization.

Importantly, this current implementation prioritizes accuracy over speed to establish the theoretical and empirical validity of physics-guided spatiotemporal modeling. 
Empirically, we observe that the inference speed can be greatly enhanced with minimal performance degradation by scaling the resolution of pre-trained diffusion models, \eg, reducing resolution to $128\times128$ and $64\times64$ cuts inference time by $67.73\%$ and $91.73\%$, respectively, while retaining over $98.63\%$ and $94.57\%$ of the original AUROC (Table \ref{tab: efficiency}). This trade-off underscores the flexibility of our approach in balancing accuracy and efficiency according to application needs.
Future work may further boost efficiency via diffusion model compression \cite{fang2023structural, wu2024ptq4dit, li2025svdqunat} or efficient architecture design \cite{zhao2024mobilediffusion,zeng2025light}, highlighting NSG-VD's potential for practical deployment as video generation and detection requirements advance.  

\begin{table}[h]
\vspace{-7pt}
    \centering
    \caption{Comparisons with baselines in terms of Inference time and performance, where we train all models with $10, 000$ real and generated videos from Kinetics-400 and SEINE, respectively.}
    \vspace{-0.5em}
    \label{tab: efficiency}
    \begin{center}
    \begin{threeparttable}
    \large
    \renewcommand{\arraystretch}{1.0}
    \resizebox{0.92\linewidth}{!}{
    \begin{tabular}{l|ccccc}
    \toprule
        Method & Recall(\%) $\uparrow$ & Accuracy(\%) $\uparrow$ & F1(\%) $\uparrow$ & AUROC(\%) $\uparrow$ & Infer. Time~(s) $\downarrow$  \\
        \midrule
    DeMamba    & 71.25 & 84.54 & 80.87 & 94.42 & 0.0311  \\
    NPR        & 59.31 & 77.95 & 71.33 & 89.92 & \textbf{0.0036}  \\
    TALL      & 61.47 & 80.20 & 74.05 & 95.14 & 0.0044  \\
    STIL       & 42.35 & 71.08 & 55.81 & 94.94 & 0.0108  \\
    \rowcolor{pink!30}NSG-VD (64x64)  & 78.29 & 83.26 & 81.99 & 90.27 & 0.0298 \\
    \rowcolor{pink!30}NSG-VD (128x128)  & 84.93 & \textbf{86.99} & 86.25 & 94.15 & 0.1163 \\
    \rowcolor{pink!30}NSG-VD (256x256)  & \textbf{97.13} & 86.05 & \textbf{87.45} & \textbf{95.46} & 0.3605 \\
    \bottomrule
    \end{tabular}
    }
    \end{threeparttable}
    \end{center}
    \vspace{-1.2em}  
\end{table}

\newpage
\subsection{Numerical Stability of NSG}

\begin{figure*}[h]
    \centering
    \includegraphics[width=0.6\linewidth]{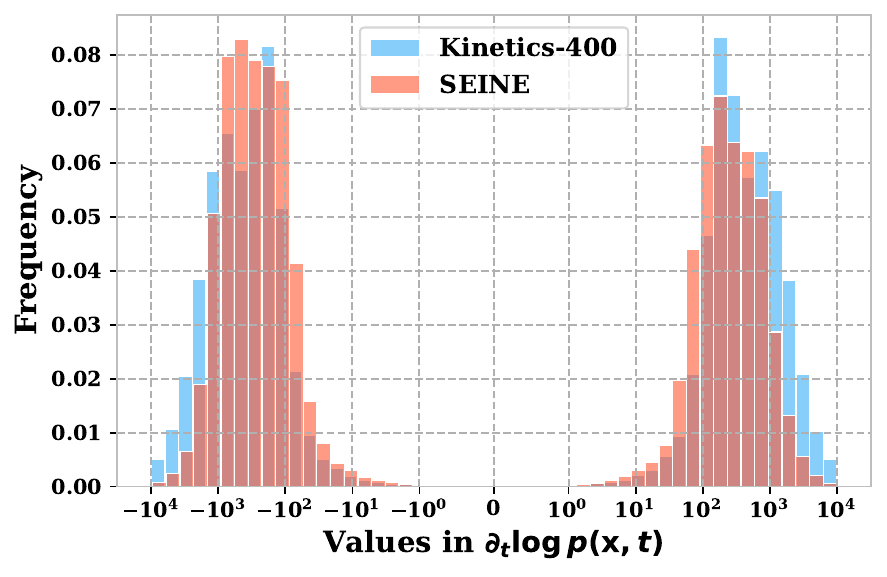}
    \vspace{-0.7em}
    \caption{Distribution of the values of temporal derivatives $\partial_t \log p(\mathbf{x}, t)$ in the NSG statistic across $10, 000$ real and generated videos from Kinetics-400 and SEINE, respectively.
    }
    \label{fig: Numerical Stability}
\end{figure*}

To ensure the numerical stability of the NSG statistic with the temporal derivatives $\partial_t \log p(\mathbf{x}, t)$ in its denominator, we examine the distribution of values in $\partial_t \log p(\mathbf{x}, t)$ across $10, 000$ real and generated videos from Kinetics-400 and SEINE, respectively. From Figure \ref{fig: Numerical Stability}, nearly all values lie outside the critical near-zero range $[-0.1, 0.1]$. This indicates that almost no value in $\partial_t \log p(\mathbf{x}, t)$ approaches zero in practice, effectively mitigating instability risks from division by vanishingly small values. 

The observed distribution aligns with the physical intuition that temporal density changes in real or synthetic videos are inherently non-stationary, resulting in measurable temporal derivatives. Additionally, the regularization term $\lambda > 0$ in the NSG denominator (Eqn. \ref{eqn:NSG}) further safeguards against edge cases where $\partial_t \log p(\mathbf{x}, t)$ might marginally approach zero. These design choices collectively ensure robust numerical stability for NSG across diverse video distributions.

\subsection{Impact of MMD for NSG-VD}
\label{sec: Impact of MMD}

To validate the inherent superiority of the NSG statistic independent of the training objective, we compare our framework trained with both Maximum Mean Discrepancy (NSG-VD) and standard binary cross-entropy loss (NSG-BCE) against baselines using BCE. From Table \ref{tab: Impact of MMD}, NSG-BCE achieves superior performance across all metrics compared to state-of-the-art baselines, even when adopting a conventional training paradigm. For example, it achieves $77.67\%$ average Recall and $82.70\%$ F1-score, significantly outperforming Demamba by $6.42\% \uparrow$ in Recall and $1.83\% \uparrow$ in F1-score, and TALL by $16.20\% \uparrow$ in Recall and $8.65\% \uparrow$ in F1-score. This demonstrates that the NSG statistic’s ability to capture spatiotemporal dynamics remains effective regardless of the training objective.

Notably, NSG-BCE excels in challenging scenarios where other methods struggle. For instance, it achieves $64.29\%$ Recall on Sora (vs. $42.86\%$ for Demamba) and $63.60\%$ Recall on WildScrape (vs. $48.00\%$ for Demamba), highlighting its ability to generalize beyond superficial artifacts. The performance gap widens further in NSG-VD ($97.13\%$ Recall), where MMD explicitly models distributional shifts by the NSG feature in a producing kernel Hilbert space and enables more precise separation between real and synthetic videos. These results confirm that the NSG statistic’s physics-driven design captures fundamental spatiotemporal dynamics, providing an intrinsic advantage over conventional features regardless of the training strategy.

\begin{table}[h]
    \caption{Impact of MMD in our NSG-VD across 10 generated paradigms (\%), where we train all models with $10, 000$ real and generated videos from Kinetics-400 and SEINE, respectively.}
    \label{tab: Impact of MMD}
    \begin{threeparttable}
    \resizebox{1.0\linewidth}{!}{
    \begin{tabular}{c|c|cccccccccc|c}
    \toprule
       \multirow{2}{*}{Method} & \multirow{2}{*}{Metric} & Model & Morph & Moon & \multirow{2}{*}{HotShot} & \multirow{2}{*}{Show1} & \multirow{2}{*}{Gen2} & \multirow{2}{*}{Crafter} & \multirow{2}{*}{Lavie} & \multirow{2}{*}{Sora} & Wild & \multirow{2}{*}{ Avg.} \\
       &  & Scope & Studio & Valley &  & & & & & & Scrape &  \\ \midrule
        \multirow{4}{*}{DeMamba} & Recall & 47.40 & 87.80 & 88.20 & 77.40 & 75.00 & 85.60 & 91.60 & 68.60 & 42.86 & 48.00 & \cellcolor{blue!8}71.25 \\
         & Accuracy & 72.80 & 93.00 & 93.20 & 87.80 & 86.60 & 91.90 & 94.90 & 83.40 & 68.75 & 73.10 & \cellcolor{blue!8}\underline{84.54} \\
         & F1 & 63.54 & 92.62 & 92.84 & 86.38 & 84.84 & 91.36 & 94.73 & 80.52 & 57.83 & 64.09 & \cellcolor{blue!8}80.87 \\
         & AUROC & 88.29 & 98.39 & 98.76 & 97.84 & 96.89 & 98.76 & 99.35 & 96.87 & 80.93 & 88.11 & \cellcolor{blue!8}94.42 \\
        \midrule
        \multirow{4}{*}{NPR} & Recall & 46.40 & 76.40 & 69.80 & 63.80 & 56.00 & 75.00 & 83.80 & 58.80 & 35.71 & 27.40 & \cellcolor{blue!8}59.31 \\
         & Accuracy & 71.40 & 86.40 & 83.10 & 80.10 & 76.20 & 85.70 & 90.10 & 77.60 & 66.96 & 61.90 & \cellcolor{blue!8}77.95 \\
         & F1 & 61.87 & 84.89 & 80.51 & 76.22 & 70.18 & 83.99 & 89.43 & 72.41 & 51.95 & 41.83 & \cellcolor{blue!8}71.33 \\
         & AUROC & 85.73 & 96.01 & 93.79 & 91.44 & 89.96 & 95.13 & 96.87 & 89.46 & 84.15 & 76.66 & \cellcolor{blue!8}89.92 \\
        \midrule
        \multirow{4}{*}{TALL} & Recall & 58.60 & 75.00 & 79.40 & 60.20 & 62.00 & 77.80 & 88.20 & 43.80 & 33.93 & 35.80 & \cellcolor{blue!8}61.47 \\
         & Accuracy & 78.80 & 87.00 & 89.20 & 79.60 & 80.50 & 88.40 & 93.60 & 71.40 & 66.07 & 67.40 & \cellcolor{blue!8}80.20 \\
         & F1 & 73.43 & 85.23 & 88.03 & 74.69 & 76.07 & 87.02 & 93.23 & 60.50 & 50.00 & 52.34 & \cellcolor{blue!8}74.05 \\
         & AUROC & 97.10 & 98.12 & 98.63 & 96.37 & 96.45 & 97.76 & 99.38 & 94.80 & 83.35 & 89.45 & \cellcolor{blue!8}\underline{95.14} \\
        \midrule
        \multirow{4}{*}{STIL} & Recall & 28.60 & 57.40 & 78.40 & 46.80 & 18.80 & 66.40 & 69.00 & 24.80 & 14.29 & 19.00 & \cellcolor{blue!8}42.35 \\
         & Accuracy & 64.20 & 78.60 & 89.10 & 73.30 & 59.30 & 83.10 & 84.40 & 62.30 & 57.14 & 59.40 & \cellcolor{blue!8}71.08 \\
         & F1 & 44.41 & 72.84 & 87.79 & 63.67 & 31.60 & 79.71 & 81.56 & 39.68 & 25.00 & 31.88 & \cellcolor{blue!8}55.81 \\
         & AUROC & 95.53 & 97.91 & 99.40 & 96.49 & 92.79 & 98.06 & 98.86 & 91.00 & 92.79 & 86.58 & \cellcolor{blue!8}94.94 \\
        \midrule
        \multirow{4}{*}{\shortstack{NSG-BCE\\(Ours)}} & Recall & 53.40 & 96.40 & 94.80 & 90.60 & 77.60 & 79.40 & 83.20 & 73.40 & 64.29 & 63.60 & \cellcolor{pink!30}\underline{77.67} \\
         & Accuracy & 72.70 & 94.20 & 93.40 & 91.30 & 84.80 & 85.70 & 87.60 & 82.70 & 74.11 & 77.80 & \cellcolor{pink!30}84.43 \\
         & F1 & 66.17 & 94.32 & 93.49 & 91.24 & 83.62 & 84.74 & 87.03 & 80.93 & 71.29 & 74.13 & \cellcolor{pink!30}\underline{82.70} \\
         & AUROC & 84.67 & 98.79 & 97.77 & 96.90 & 92.69 & 93.00 & 93.86 & 91.32 & 83.58 & 87.99 & \cellcolor{pink!30}92.06 \\
        \midrule
        \multirow{4}{*}{\shortstack{NSG-VD\\(Ours)}} & Recall & 91.67 & 100.00 & 100.00 & 100.00 & 100.00 & 98.33 & 100.00 & 97.50 & 94.64 & 89.17 & \cellcolor{pink!30}\textbf{97.13} \\
         & Accuracy & 82.50 & 88.33 & 89.58 & 84.58 & 86.25 & 87.08 & 86.67 & 87.92 & 89.29 & 78.33 & \cellcolor{pink!30}\textbf{86.05} \\
         & F1 & 83.97 & 89.55 & 90.57 & 86.64 & 87.91 & 88.39 & 88.24 & 88.97 & 89.83 & 80.45 & \cellcolor{pink!30}\textbf{87.45} \\
         & AUROC & 90.67 & 97.62 & 98.38 & 95.88 & 96.69 & 97.87 & 97.64 & 95.09 & 96.14 & 88.65 & \cellcolor{pink!30}\textbf{95.46} \\
    \bottomrule
    \end{tabular}
    }
    \end{threeparttable}
\vspace{-1pt}
\end{table}

\newpage
\subsection{Impact of Size of Reference Set for NSG-VD}
We investigate the impact of reference set size by evaluating subsets containing between $10$ and $500$ samples, with other settings remaining consistent with Table \ref{tab: standard Pika}. 
Performance is assessed using comprehensive criteria, including AUROC, Accuracy, F1 Score, and Recall.
Intuitively, a larger reference set enables more accurate estimation of the underlying distribution of real videos, thereby supporting more stable and reliable detection. In contrast, smaller reference sets may introduce substantial sampling and estimation biases.
As shown in Figure \ref{fig: Reference size}, our NSG-VD demonstrates consistently robust performance across varying reference set sizes, with the exception of extreme cases involving very limited samples (\eg, $n=10$). Consequently, we set $n=100$ for all experiments.

\begin{figure*}[h]
    \centering
    \includegraphics[width=0.6\linewidth]{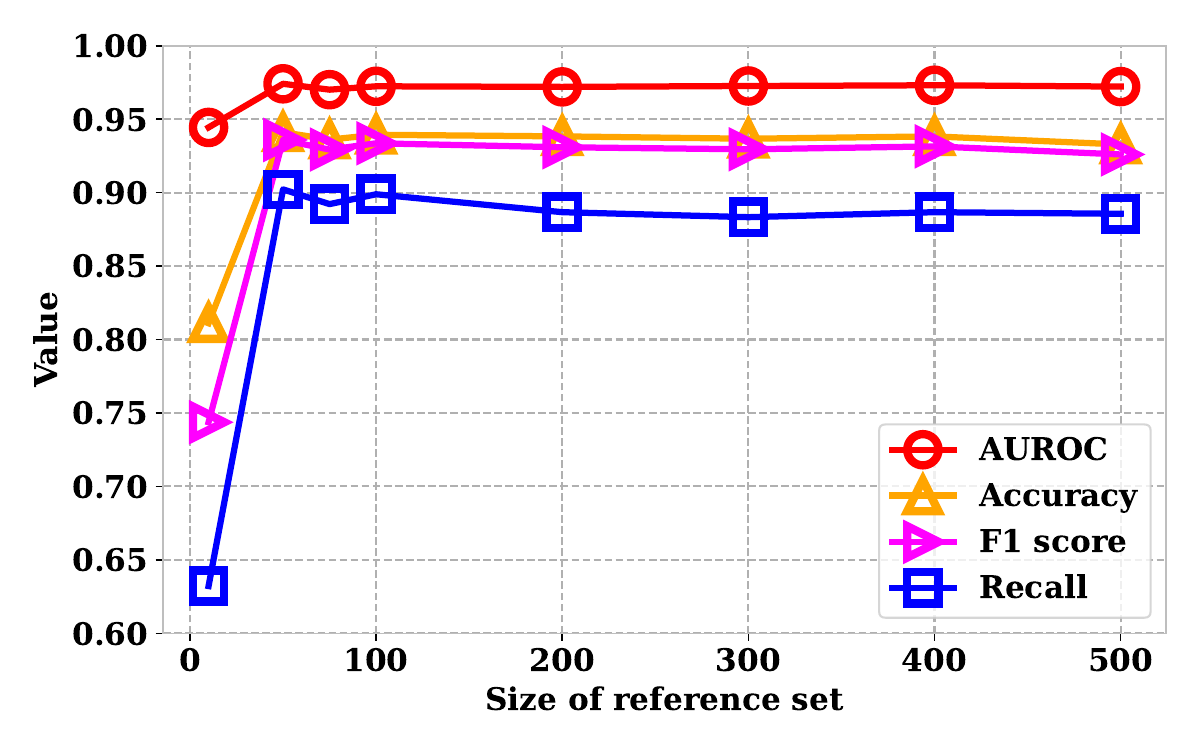}
    \vspace{-0.7em}
    \caption{Impact of reference set size for NSG-VD, where we train all models with $10, 000$ real and generated videos from Kinetics-400 and Pika, respectively.
    }
    \label{fig: Reference size}
\end{figure*}

\newpage

\subsection{Impact of Diversity of Real Videos in the Reference Set}
We conduct additional ablation studies on \textit{real‑domain mixed} reference sets, revealing a key strength of NSG-VD: strong generalization to unseen generated video domains when most real test samples are covered by the reference distribution.
Specifically, we train on Kinetics-400 (real) and SEINE (generated) videos, and test on MSR-VTT (real) and 10 generated videos using reference sets with varying ratios of MSR-VTT and Kinetics-400. From Table \ref{tab:domain coverage}, even a small proportion ($3:7$) yields satisfactory performance ($84.19\%$ of Accuray, $81.12\%$ of F1-Score) compared with baselines, which quickly saturates. This confirms that NSG‑VD needs only modest in‐domain real coverage, while the fake side can remain highly heterogeneous. 

\begin{table}[h]
    \centering
    \caption{Performance under different domain coverage ratios between MSR-VTT and Kinetics-400.}
    \vspace{-0.em}
    \label{tab:domain coverage}
    \small
    \renewcommand{\arraystretch}{1.3}
    \setlength{\tabcolsep}{8pt}
    \begin{threeparttable}
    \resizebox{\linewidth}{!}{
    \begin{tabular}{l|ccccc|cc}
    \toprule
        \makecell[c]{\textbf{Domain Coverage}\\(MSR-VTT : Kinetics-400)} & 
        \makecell[c]{0:10\\(Low)} & 
        \makecell[c]{3:7\\(Medium)} & 
        \makecell[c]{5:5\\(Balanced)} & 
        \makecell[c]{7:3\\(High)} & 
        \makecell[c]{10:0\\(Full)} & 
        DeMamba & 
        TALL \\
    \midrule
        {Average Accuracy (\%)} & 77.82 & 84.19 & 85.57 & \textbf{87.06} & 86.05 & 84.21 & 80.20 \\
        {Average F1-Score (\%)} & 75.68 & 81.12 & 83.36 & 85.41 & \textbf{87.45} & 80.87 & 74.05 \\
    \bottomrule
    \end{tabular}
    }
    \end{threeparttable}
    \vspace{-0.6em}
\end{table}

\subsection{Discussions on Assumption of the Divergence Term}
We assume $\nabla_{\mathbf{x}}\cdot\mathbf{v}$ is subdominant in smoothly varying video distributions for three reasons: \textbf{First}, its direct estimation is ill-posed in high-dimensional video data. Solving $\partial_t\mathbf{x}=\mathbf{v}(\mathbf{x},t)$ is an underdetermined inverse problem, and video noise (\eg, blur or compression) further amplifies estimation errors, making explicit divergence computation unstable and infeasible \cite{horn1981determining,rieutord2014fluid}. \textbf{Second}, many physical flows approximate incompressibility ($\nabla_{\mathbf{x}}\cdot\mathbf{v}\approx 0$), a simplification grounded in fluid dynamics \cite{rieutord2014fluid} and quantum mechanics \cite{bohm2013quantum} that preserves physical interpretability. \textbf{Third}, our NSG remains robust even if $\nabla_{\mathbf{x}}\cdot \mathbf{v}\neq0$, as it captures cumulative spatiotemporal inconsistencies across all terms in Eqn. (\ref{eqn: log_p_with_delta_v}). Experiments confirm the resilience of NSG-VD to deviations from this assumption.

\section{Limitations and Future Directions}
\label{sec: Limitations and Future Directions}

While our proposed NSG-VD method demonstrates strong performance across diverse AI-generated video detection scenarios, several limitations and opportunities for future work remain:

\textbf{Limitations.} First, the current formulation of the NSG statistic relies on simplified physical assumptions (\eg, the incompressible flow approximation in continuity equations), which may fail to capture highly dynamic or discontinuous motion patterns in complex real-world scenarios. Second, the effectiveness of NSG-VD critically depends on the quality of pre-trained diffusion models used for score estimation; domain shifts or limited training data may degrade the reliability of estimated NSG features. Third, while NSG-VD achieves competitive performance, its reliance on diffusion models introduces computational overhead, making it less suitable for large-scale real-time detection tasks. Lastly, while our deep kernel design improves detection performance, its architecture could be further optimized to better adapt to heterogeneous spatiotemporal patterns.

\textbf{Future Directions.} To address these limitations, future work could explore more sophisticated physical models that account for compressible flows or discontinuous motion dynamics \cite{panton2024incompressible}, enhancing the NSG statistic’s adaptability to complex scenarios. Additionally, developing effective domain-specific fine-tuning strategies \cite{xie2023difffit, denker2024deft,zhong2024domain,guo2022deep} could improve the reliability of score estimation under distribution shifts. For real-time deployment, investigating lightweight diffusion model compression techniques (\eg, pruning \cite{fang2023structural, ma2024deepcache}, quantization \cite{wu2024ptq4dit, li2025svdqunat}) would reduce computational costs. Finally, advancing the design of the deep kernel network—such as incorporating attention mechanisms \cite{vaswani2017attention} or hierarchical feature fusion \cite{liu2021swin}—could further optimize the MMD-based detection framework, enabling better 
performance for AI-generated video detection.

\newpage
\section{Broader Impacts}
\label{sec: Broader Impacts}

The development of AI-generated video detection methods like NSG-VD has significant societal, ethical, and technical implications. Our work contributes to mitigating the risks of malicious deepfake content, such as misinformation, identity theft, and political manipulation, by enabling more reliable verification of video authenticity. By leveraging physics-informed principles, NSG-VD provides a reliable framework for detecting synthetic videos that may otherwise evade traditional artifact-based detection methods. This could strengthen trust in digital media, support content moderation efforts, and aid legal or journalistic investigations involving video evidence.

This research aligns with broader efforts to establish trustworthy multimedia ecosystems. By bridging physics principles with machine learning, NSG-VD advances interpretable detection mechanisms—a critical step toward auditing AI-generated content while fostering public awareness of synthetic media risks. We encourage interdisciplinary collaboration among researchers, ethicists, and legislators to ensure such technologies serve as safeguards rather than instruments of control.

\newpage
\section{Visualizations}
\label{sec: Visualizations}

\begin{figure*}[!h]
    \begin{center}
    {\includegraphics[width=0.87\linewidth]{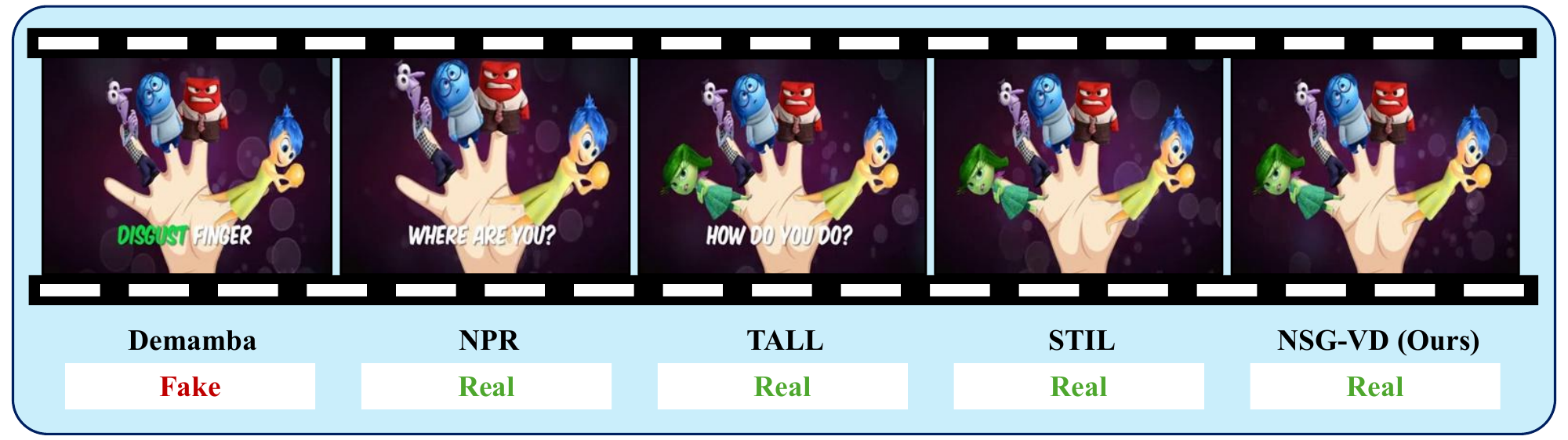}}
    {\includegraphics[width=0.87\linewidth]{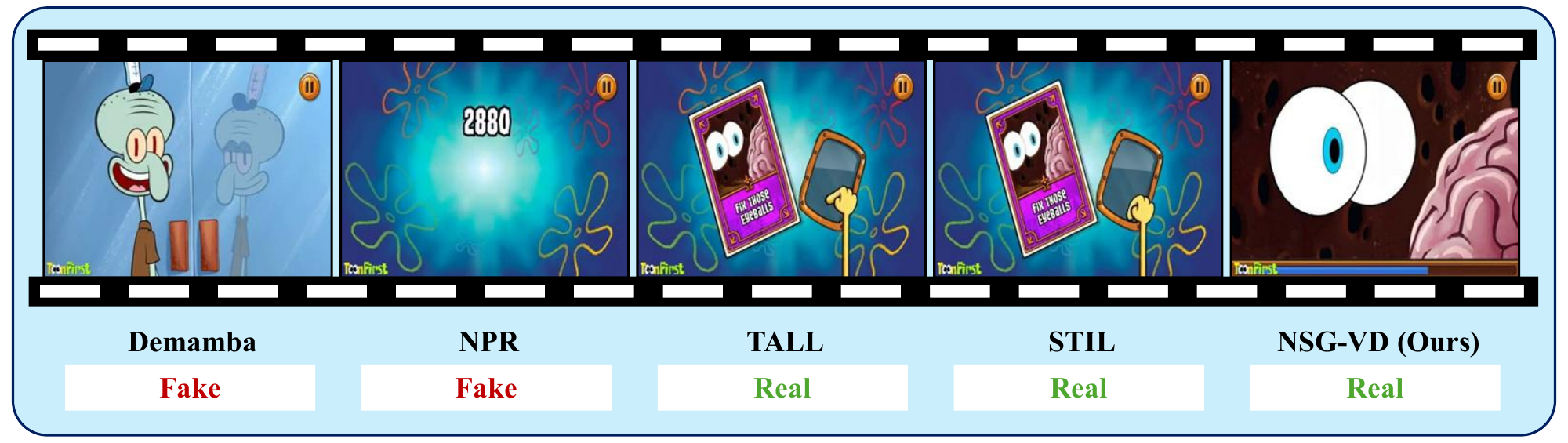}}
    {\includegraphics[width=0.87\linewidth]{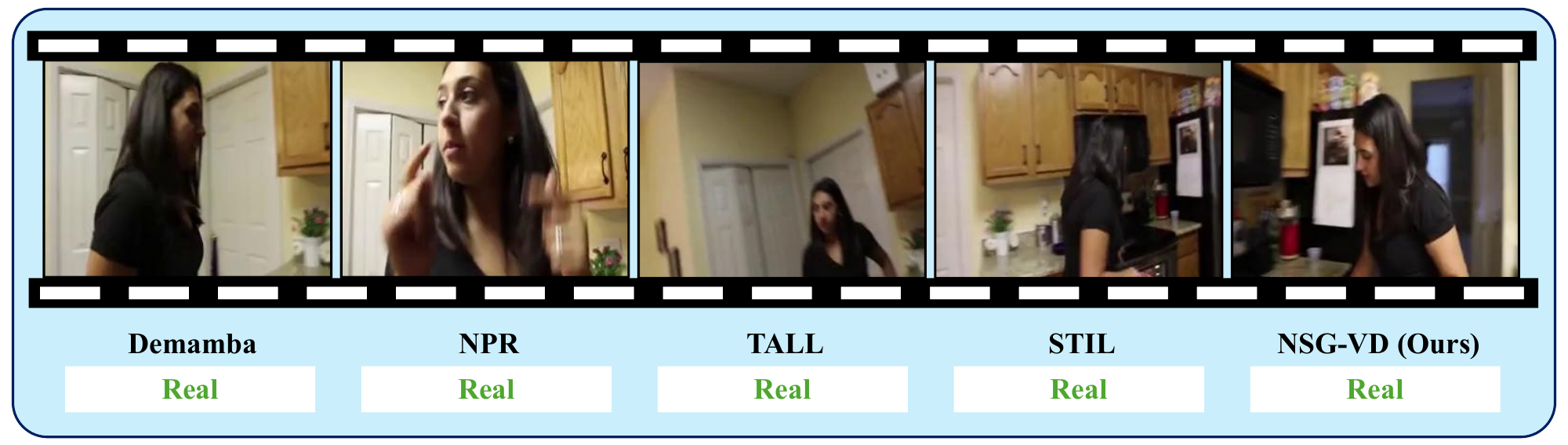}}
    \caption{Results of the detection on \textit{real} videos from the MSR-VTT dataset.}
    \label{fig: visualization_Natural}
    \end{center}
\end{figure*}

\begin{figure*}[!h]
    \begin{center}
    {\includegraphics[width=0.87\linewidth]{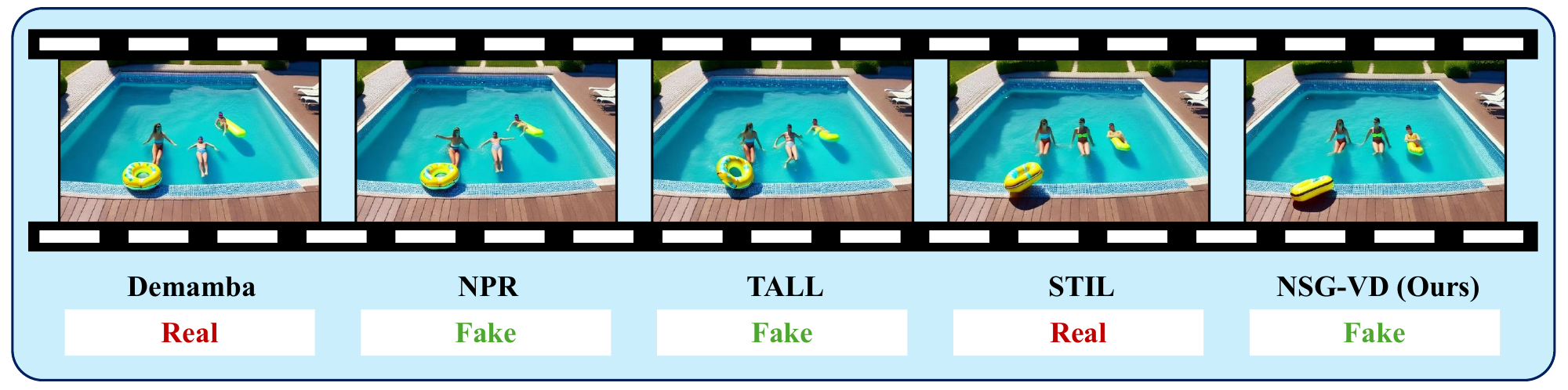}}
    {\includegraphics[width=0.87\linewidth]{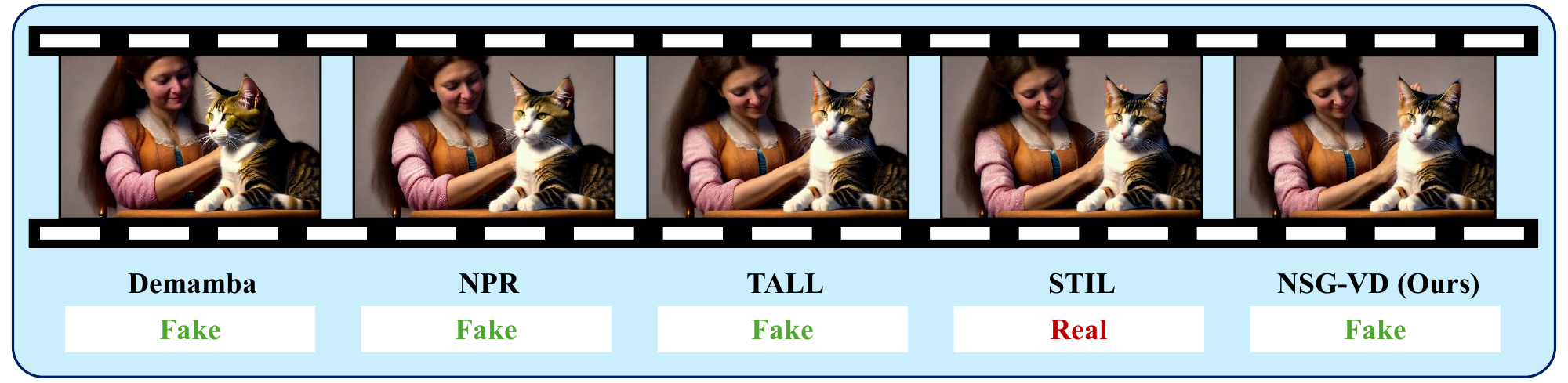}}
    {\includegraphics[width=0.87\linewidth]{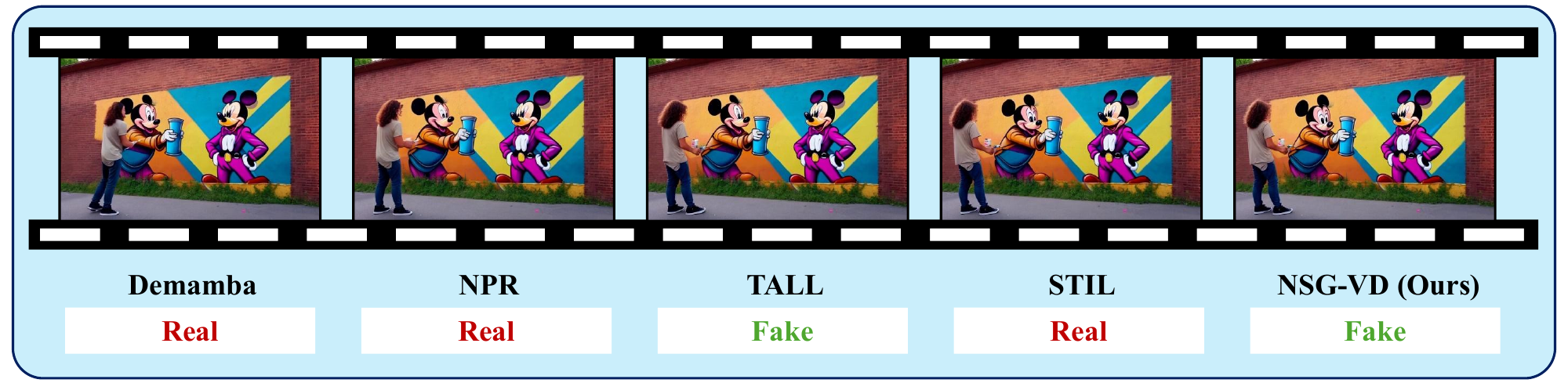}}
    \caption{Results of the detection on \textit{generated} videos from the Crafter dataset.}
    \label{fig: visualization_Crafter}
    \end{center}
\end{figure*}

\begin{figure*}[!h]
    \begin{center}
    {\includegraphics[width=0.95\linewidth]{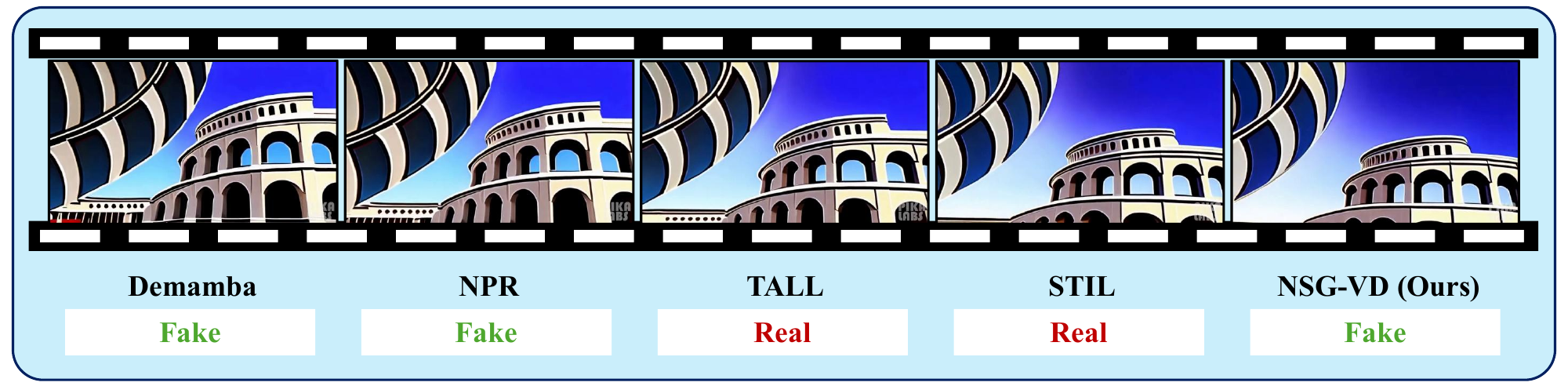}}
    {\includegraphics[width=0.95\linewidth]{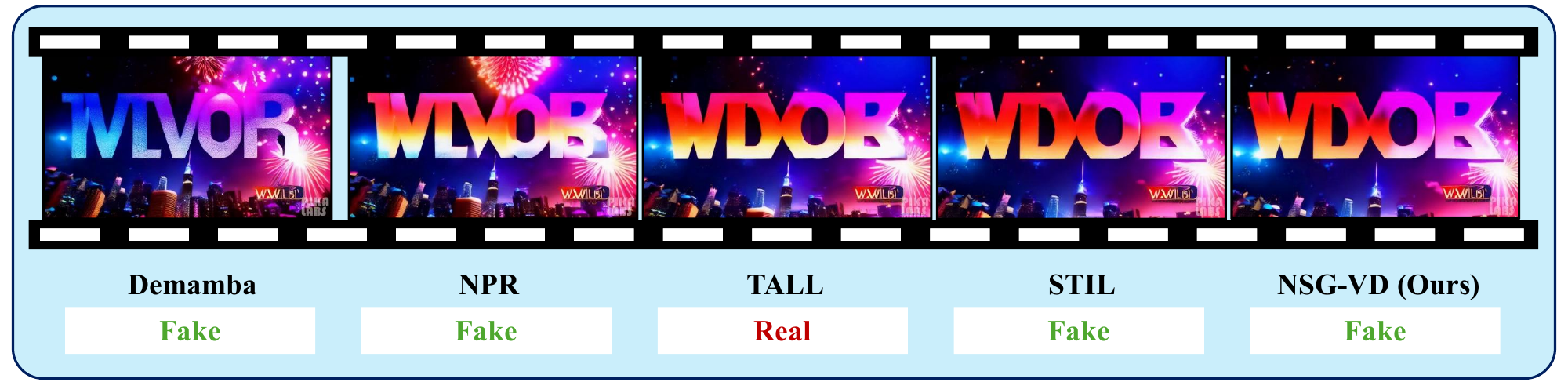}}
    {\includegraphics[width=0.95\linewidth]{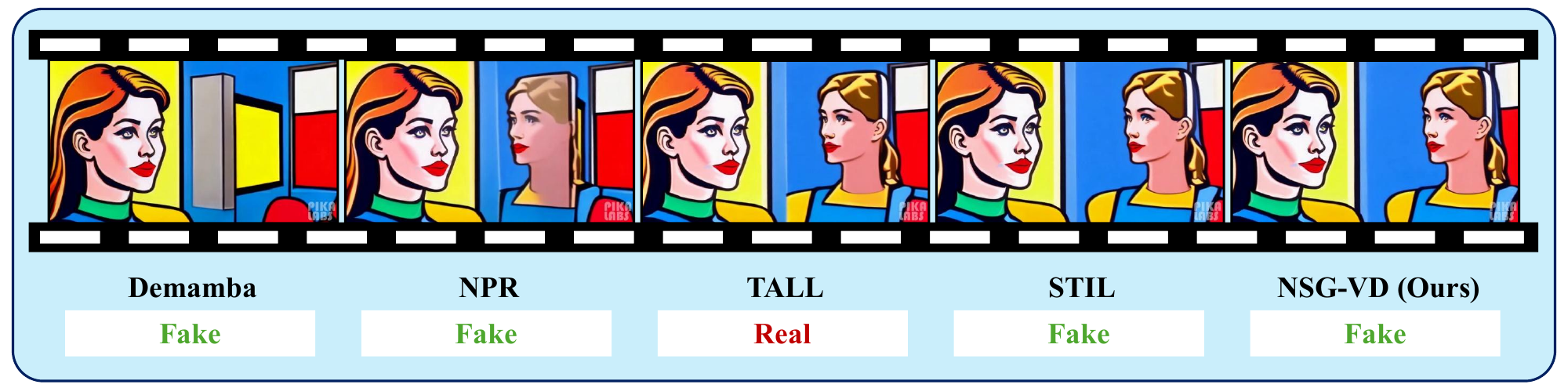}}
    \caption{Results of the detection on \textit{generated} videos from the Gen2 dataset.}
    \label{fig: visualization_Gen2}
    \end{center}
\end{figure*}

\begin{figure*}[!h]
    \begin{center}
    {\includegraphics[width=0.95\linewidth]{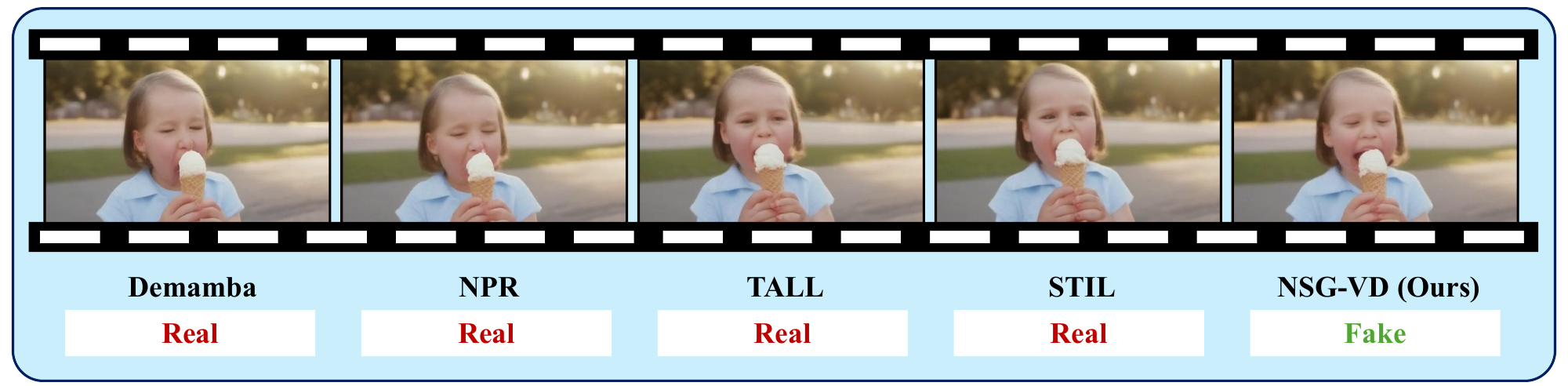}}
    {\includegraphics[width=0.95\linewidth]{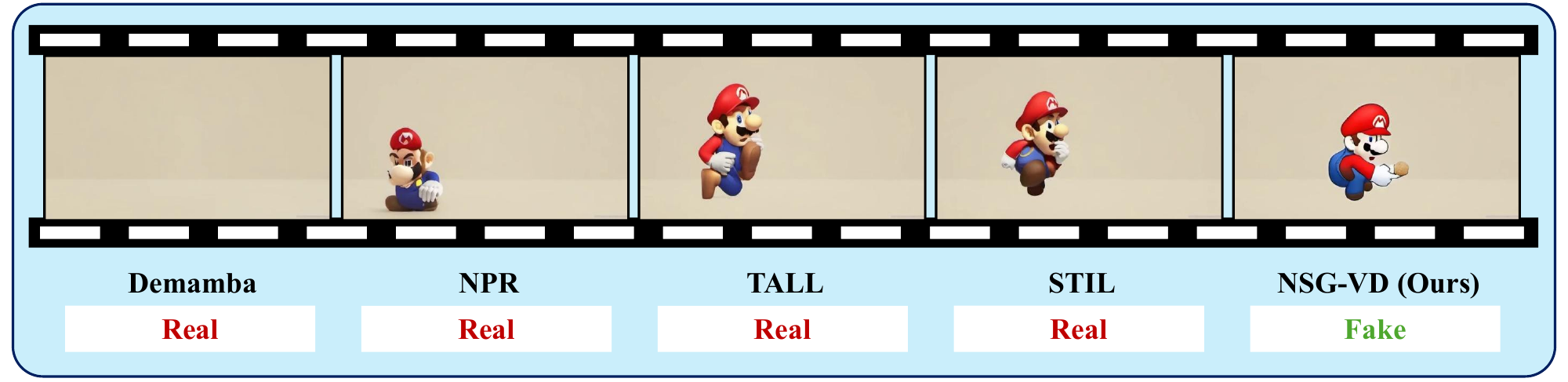}}
    {\includegraphics[width=0.95\linewidth]{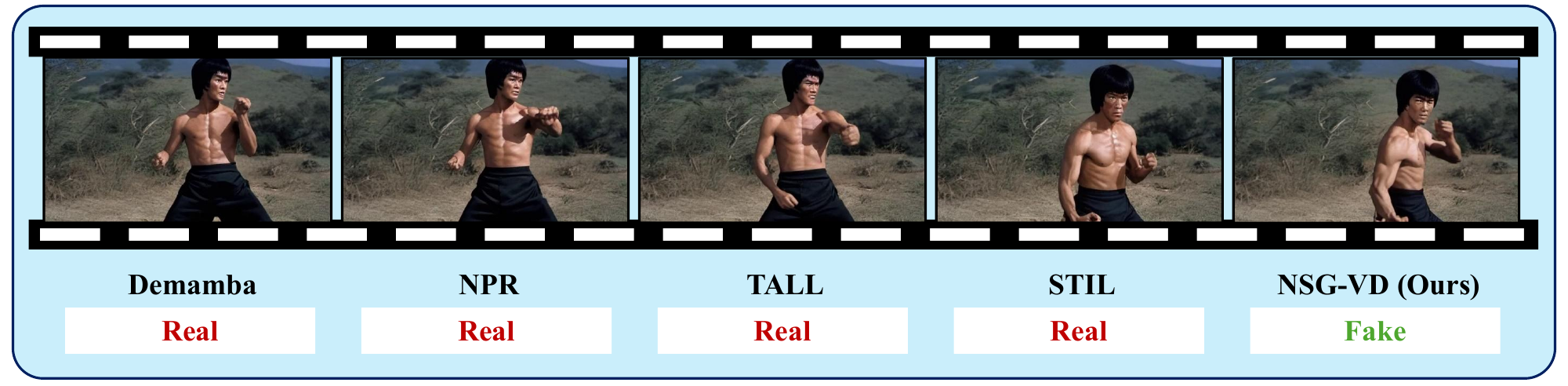}}
    \caption{Results of the detection on \textit{generated} videos from the HotShot dataset.}
    \label{fig: visualization_HotShot}
    \end{center}
\end{figure*}

\begin{figure*}[!h]
    \begin{center}
    {\includegraphics[width=0.95\linewidth]{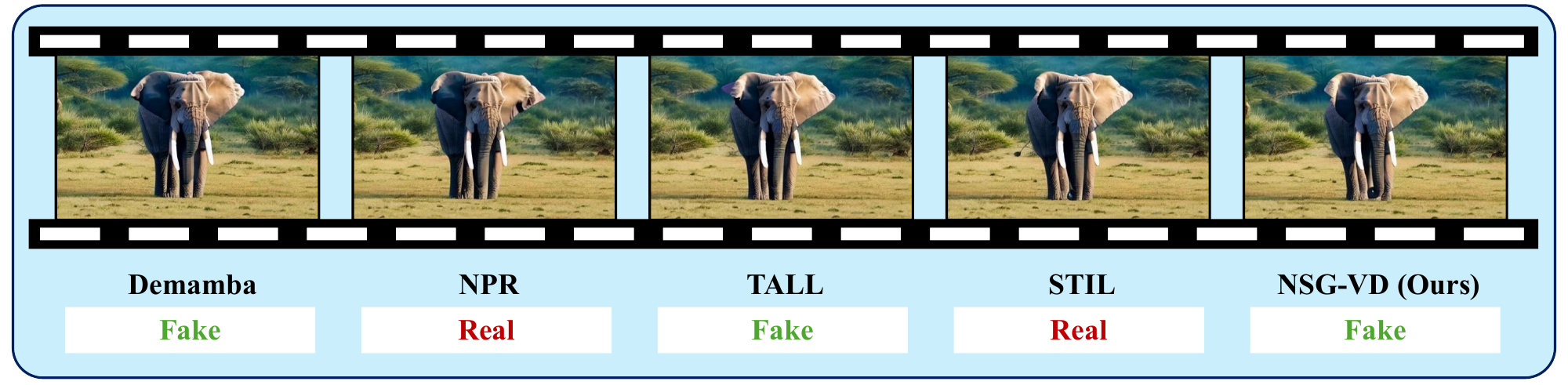}}
    {\includegraphics[width=0.95\linewidth]{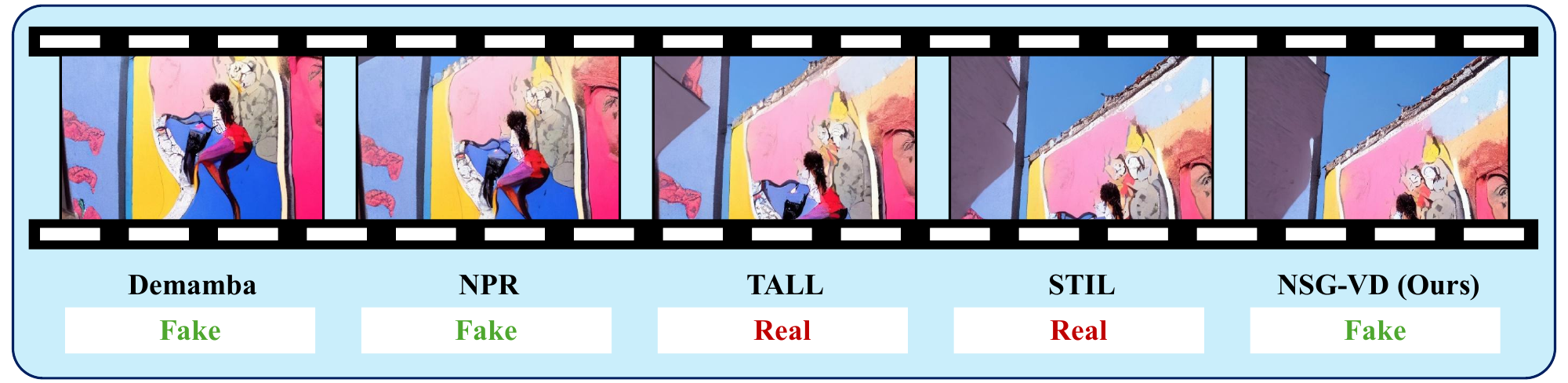}}
    {\includegraphics[width=0.95\linewidth]{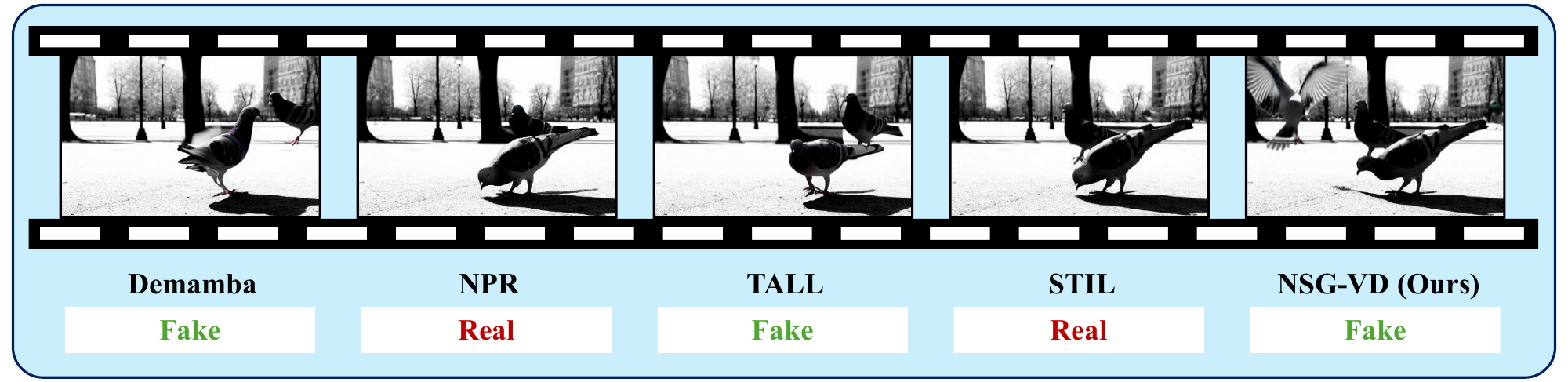}}
    \caption{Results of the detection on \textit{generated} videos from the Lavie dataset.}
    \label{fig: visualization_Lavie}
    \end{center}
\end{figure*}

\begin{figure*}[!h]
    \begin{center}
    {\includegraphics[width=0.95\linewidth]{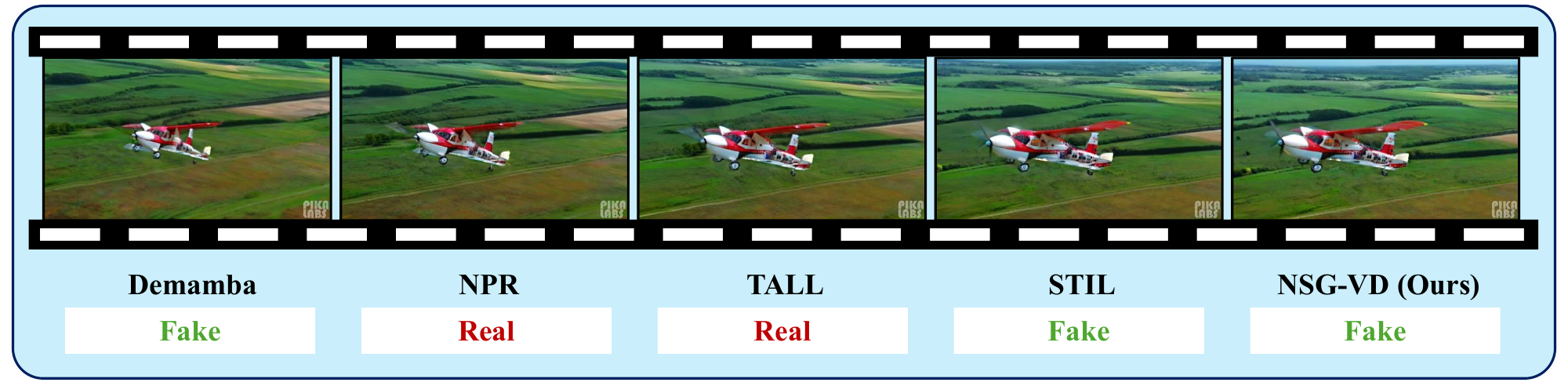}}
    {\includegraphics[width=0.95\linewidth]{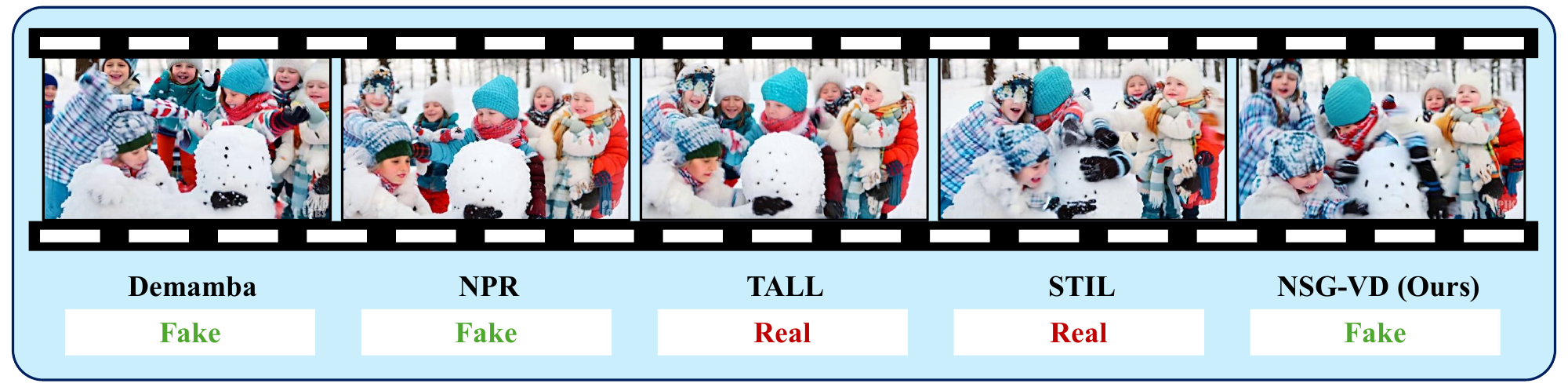}}
    {\includegraphics[width=0.95\linewidth]{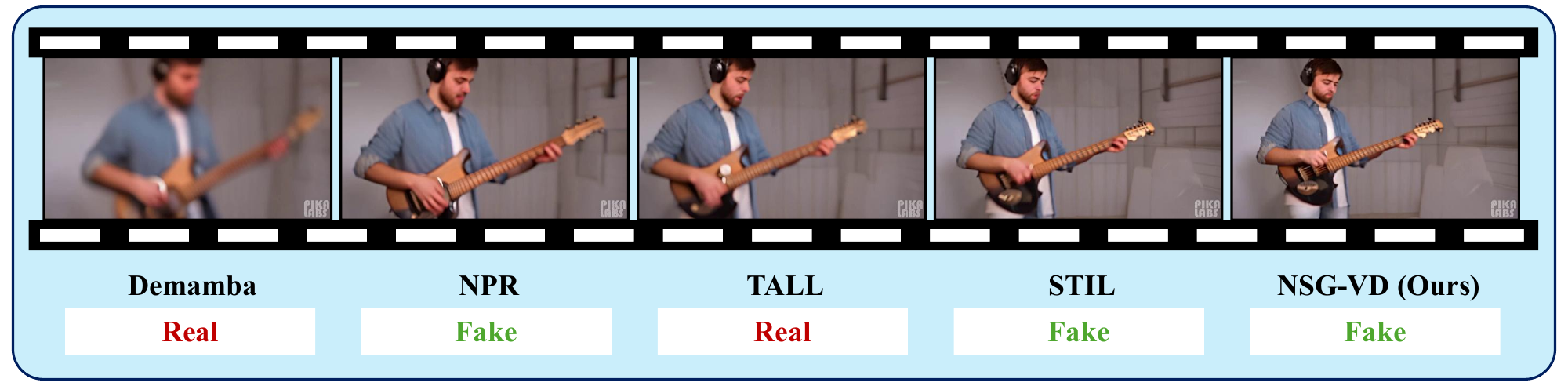}}
    \caption{Results of the detection on \textit{generated} videos from the ModelScope dataset.}
    \label{fig: visualization_ModelScope}
    \end{center}
\end{figure*}

\begin{figure*}[!h]
    \begin{center}
    {\includegraphics[width=0.95\linewidth]{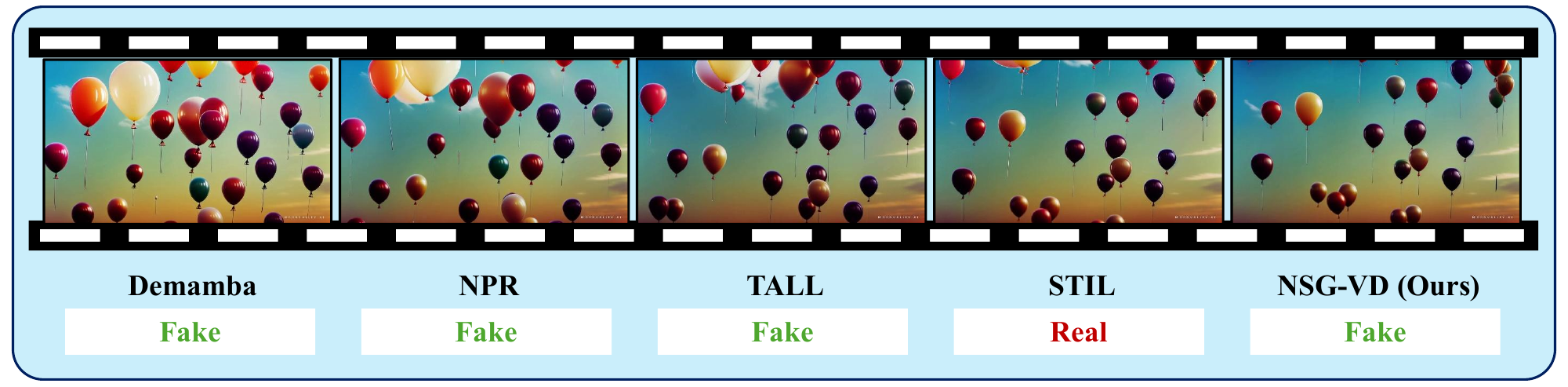}}
    {\includegraphics[width=0.95\linewidth]{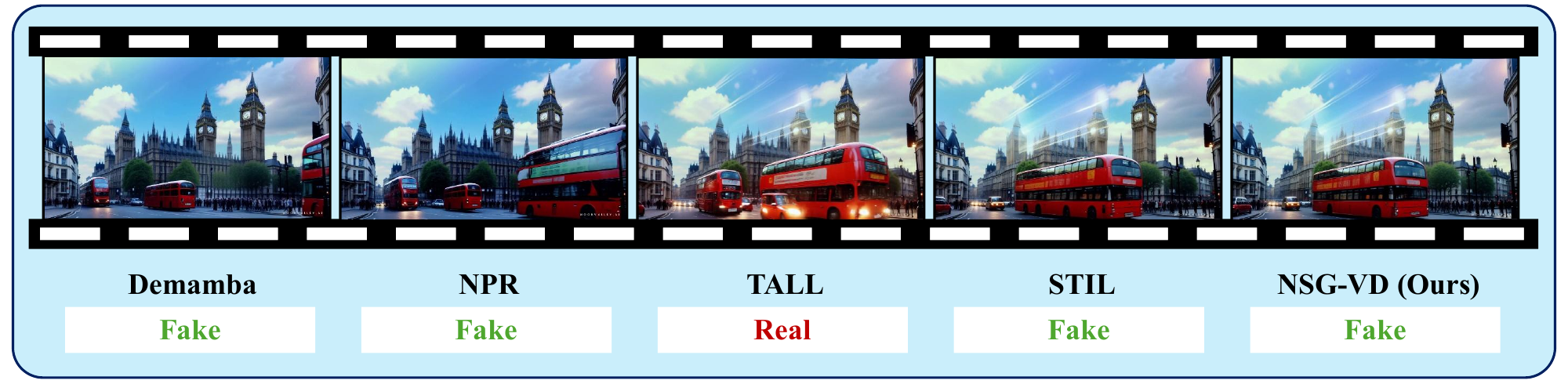}}
    {\includegraphics[width=0.95\linewidth]{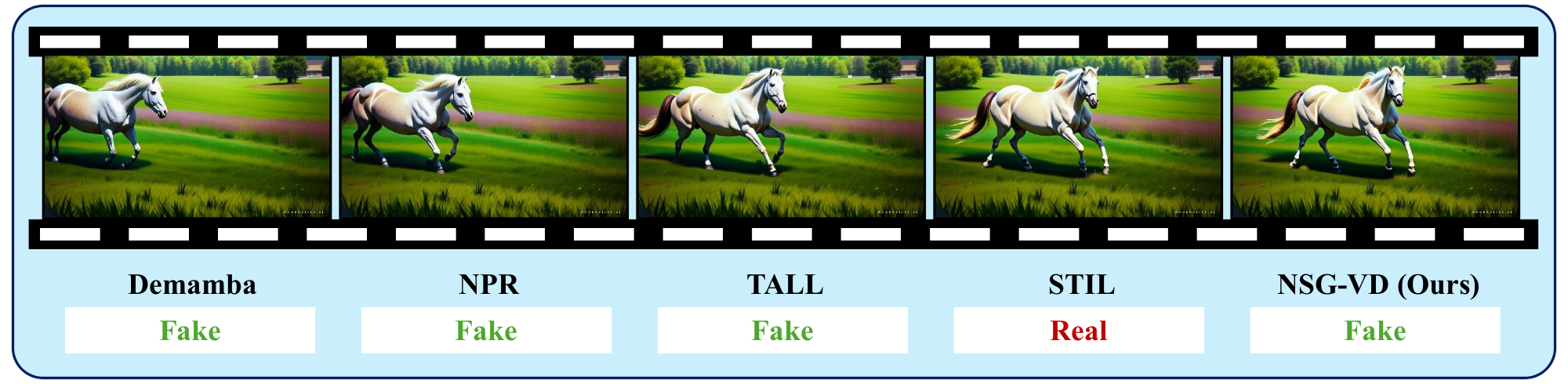}}
    \caption{Results of the detection on \textit{generated} videos from the MoonValley dataset.}
    \label{fig: visualization_MoonValley}
    \end{center}
\end{figure*}

\begin{figure*}[!h]
    \begin{center}
    {\includegraphics[width=0.95\linewidth]{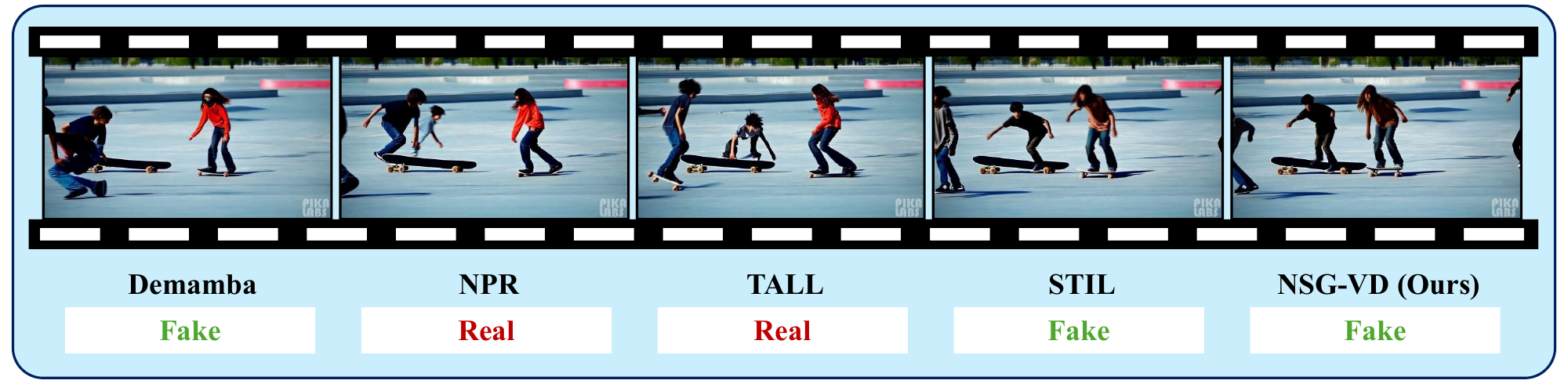}}
    {\includegraphics[width=0.95\linewidth]{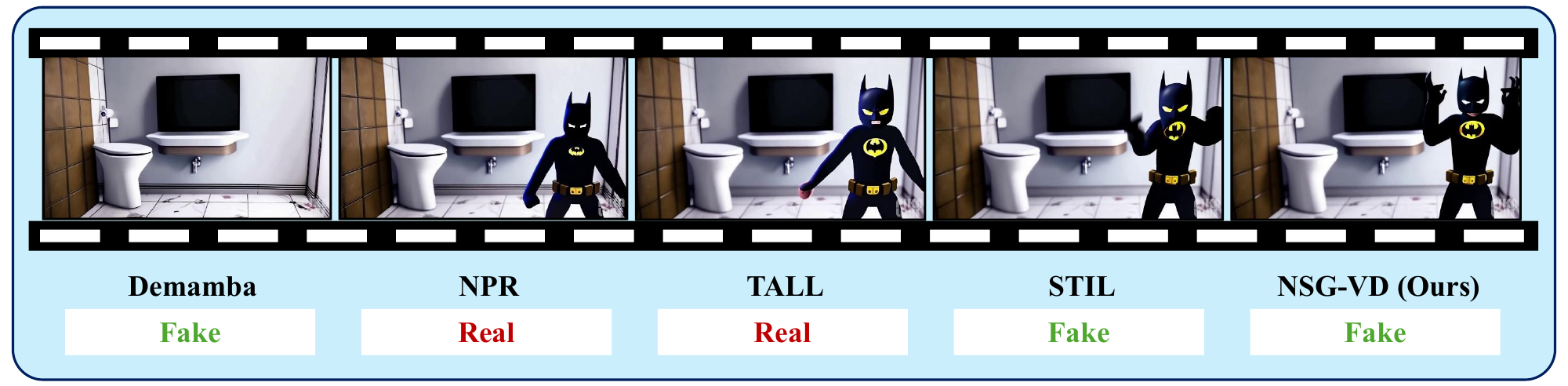}}
    {\includegraphics[width=0.95\linewidth]{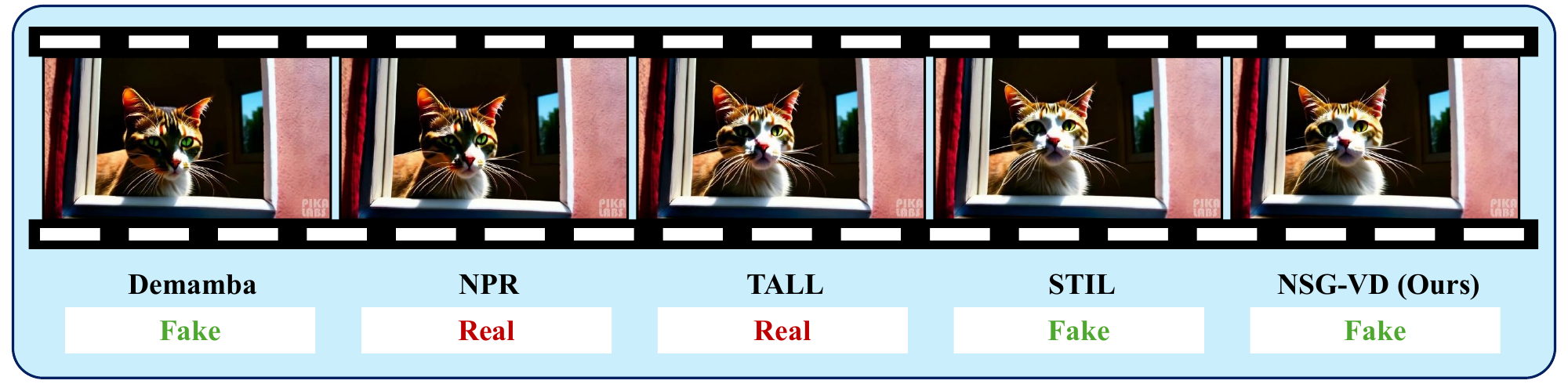}}
    \caption{Results of the detection on \textit{generated} videos from the MorphStudio dataset.}
    \label{fig: visualization_MorphStudio}
    \end{center}
\end{figure*}

\begin{figure*}[!h]
    \begin{center}
    {\includegraphics[width=0.95\linewidth]{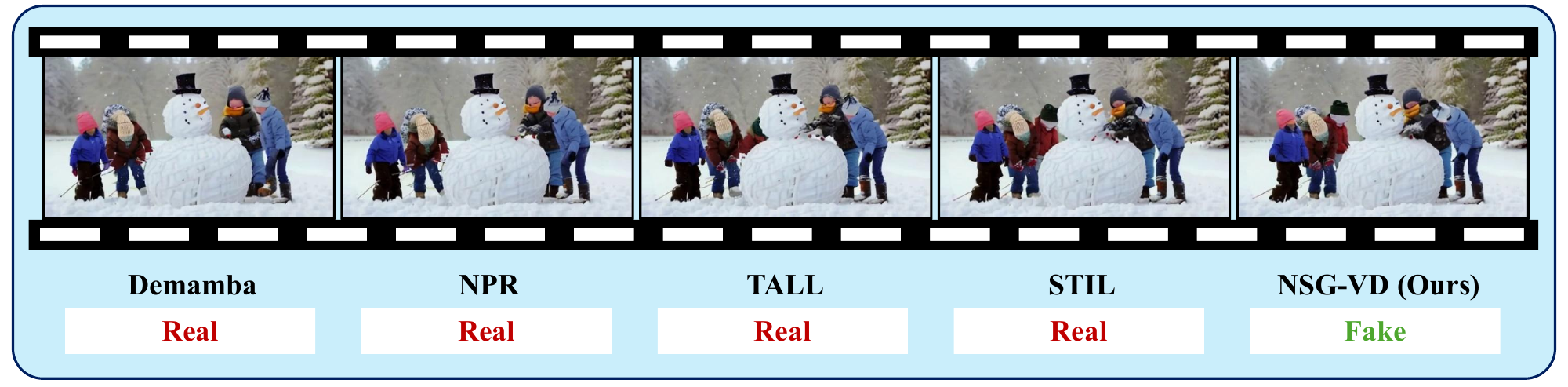}}
    {\includegraphics[width=0.95\linewidth]{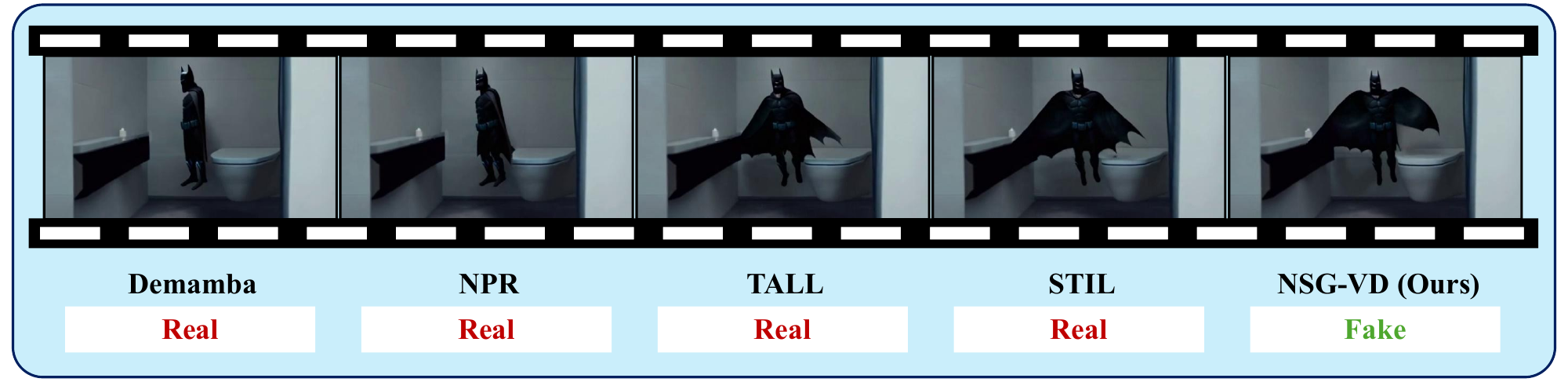}}
    {\includegraphics[width=0.95\linewidth]{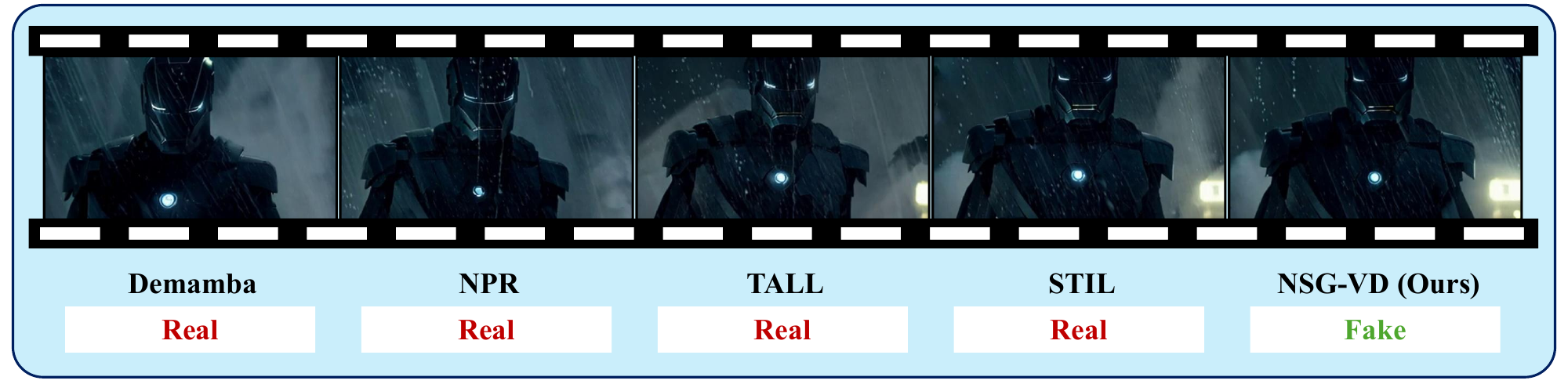}}
    \caption{Results of the detection on \textit{generated} videos from the Show1 dataset.}
    \label{fig: visualization_Show_1}
    \end{center}
\end{figure*}

\begin{figure*}[!h]
    \begin{center}
    {\includegraphics[width=0.95\linewidth]{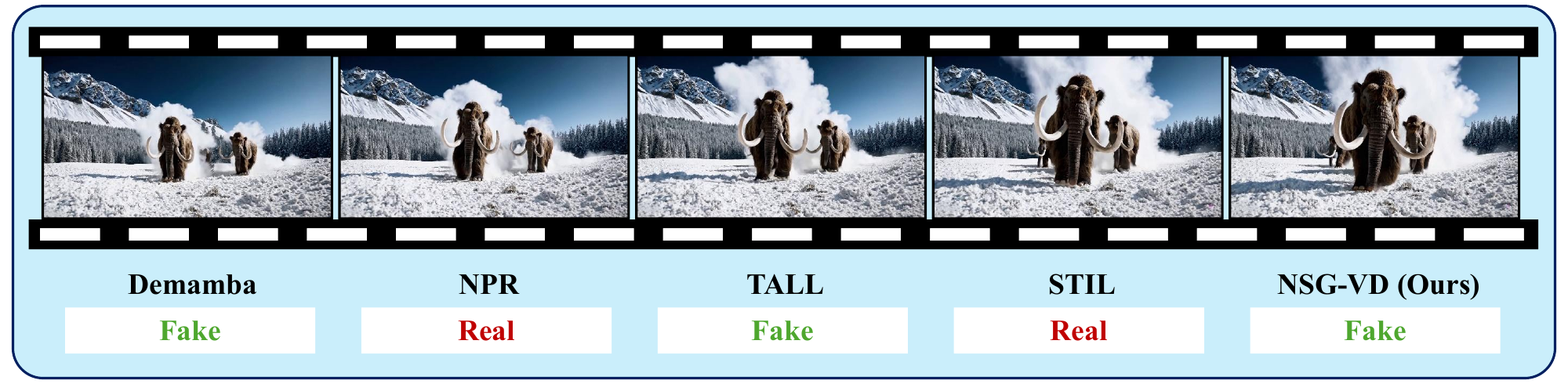}}
    {\includegraphics[width=0.95\linewidth]{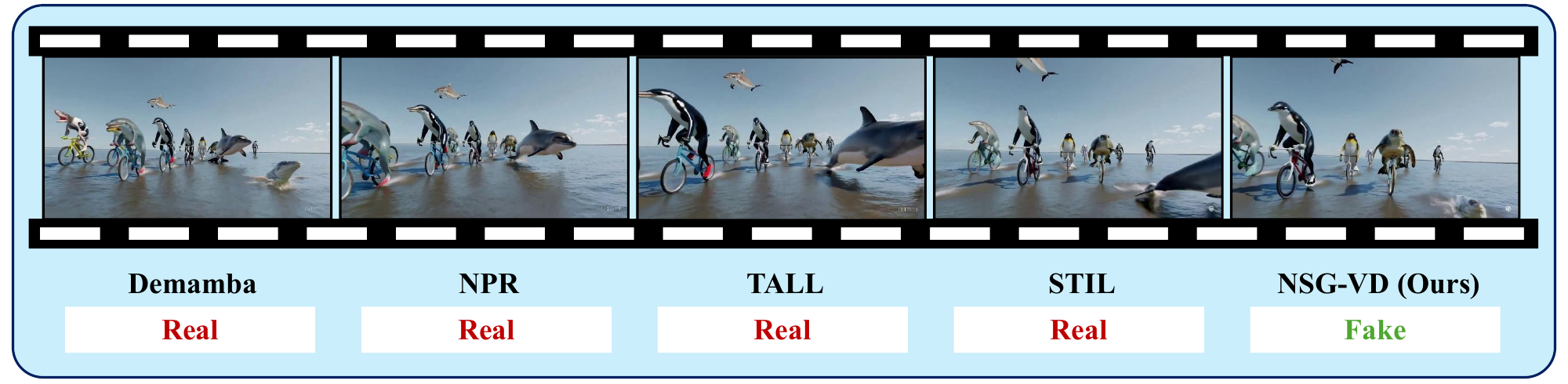}}
    {\includegraphics[width=0.95\linewidth]{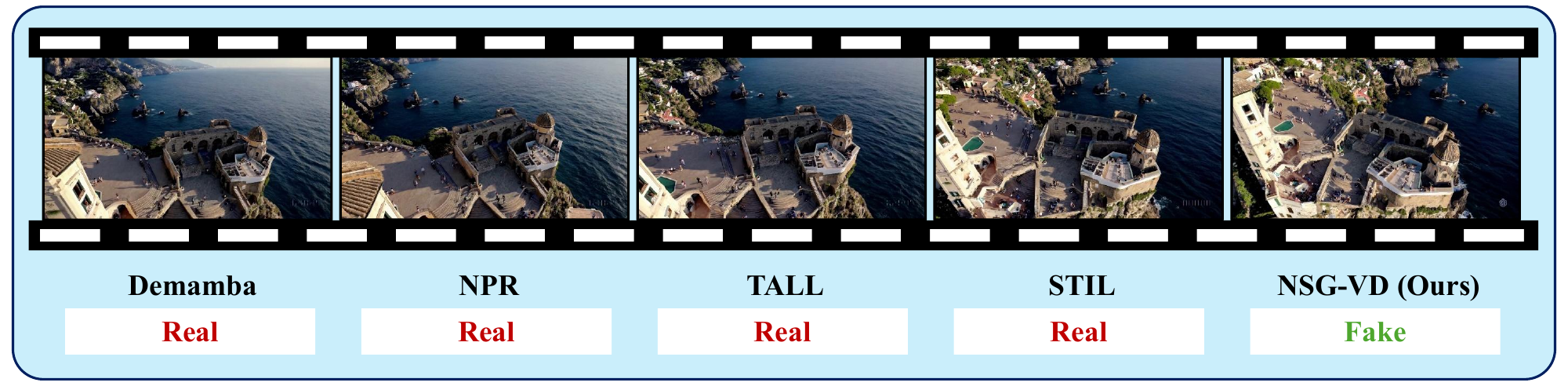}}
    \caption{Results of the detection on \textit{generated} videos from the Sora dataset.}
    \label{fig: visualization_Sora}
    \end{center}
\end{figure*}

\begin{figure*}[!h]
    \begin{center}
    {\includegraphics[width=0.95\linewidth]{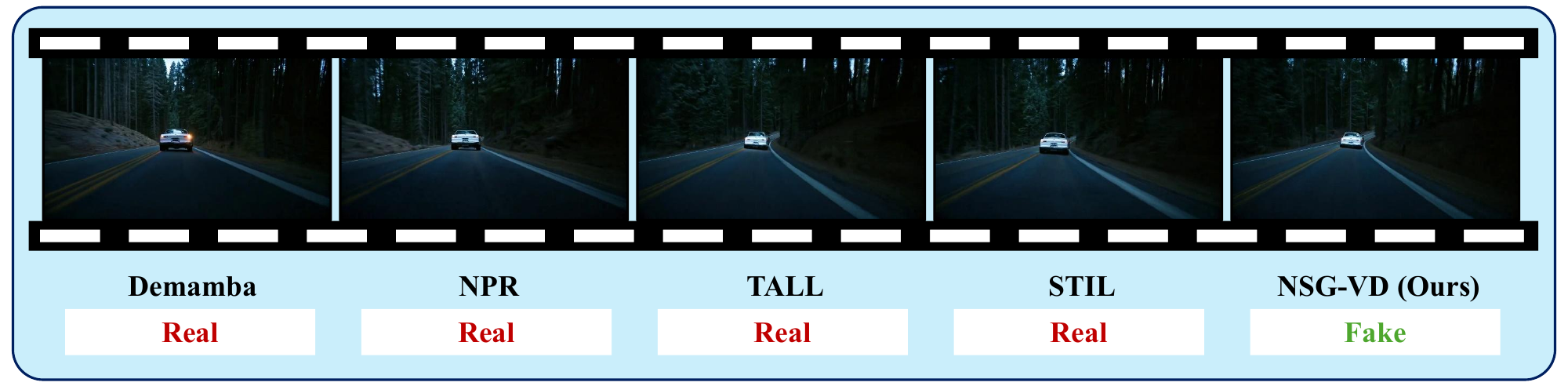}}
    {\includegraphics[width=0.95\linewidth]{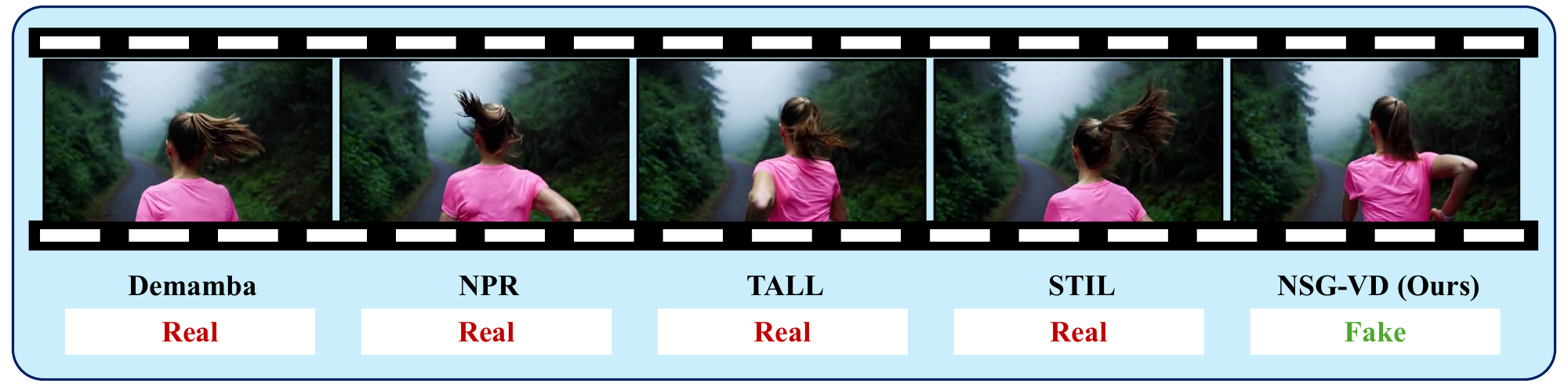}}
    {\includegraphics[width=0.95\linewidth]{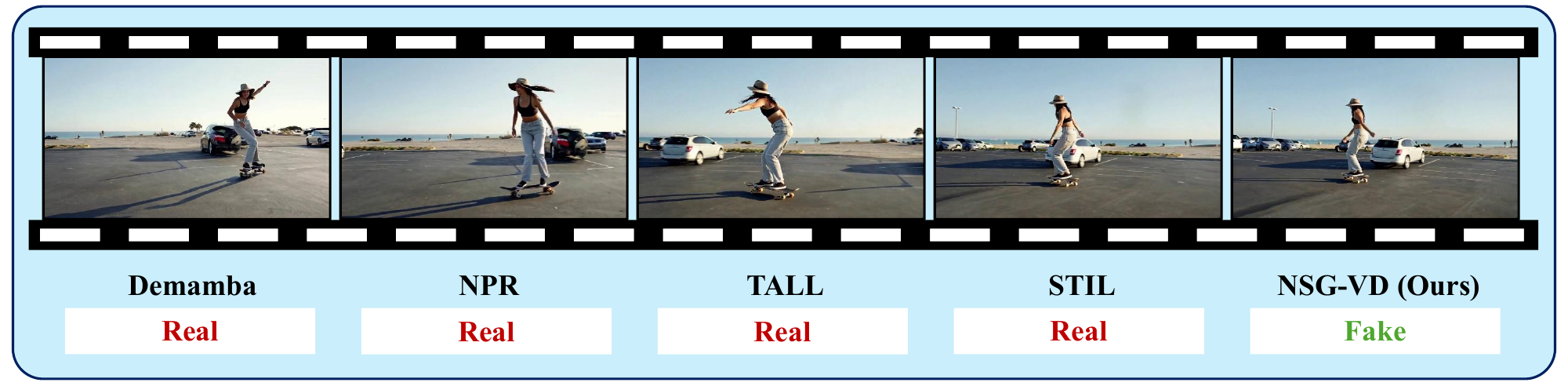}}
    \caption{Results of the detection on \textit{generated} videos from the Seaweed dataset.}
    \label{fig: visualization_Seaweed_1}
    \end{center}
\end{figure*}

\begin{figure*}[!h]
    \begin{center}
    {\includegraphics[width=0.95\linewidth]
    {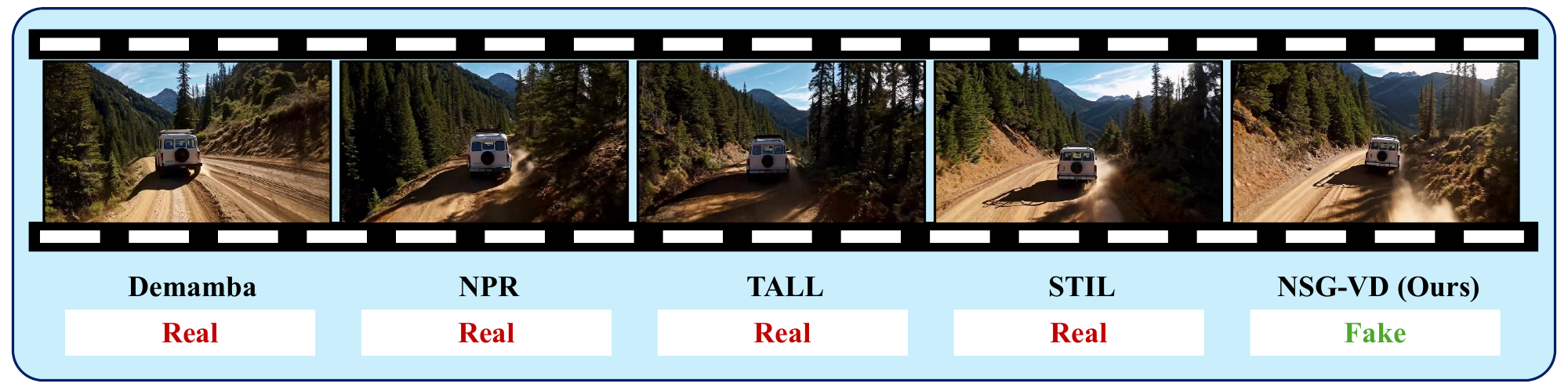}}
    {\includegraphics[width=0.95\linewidth]{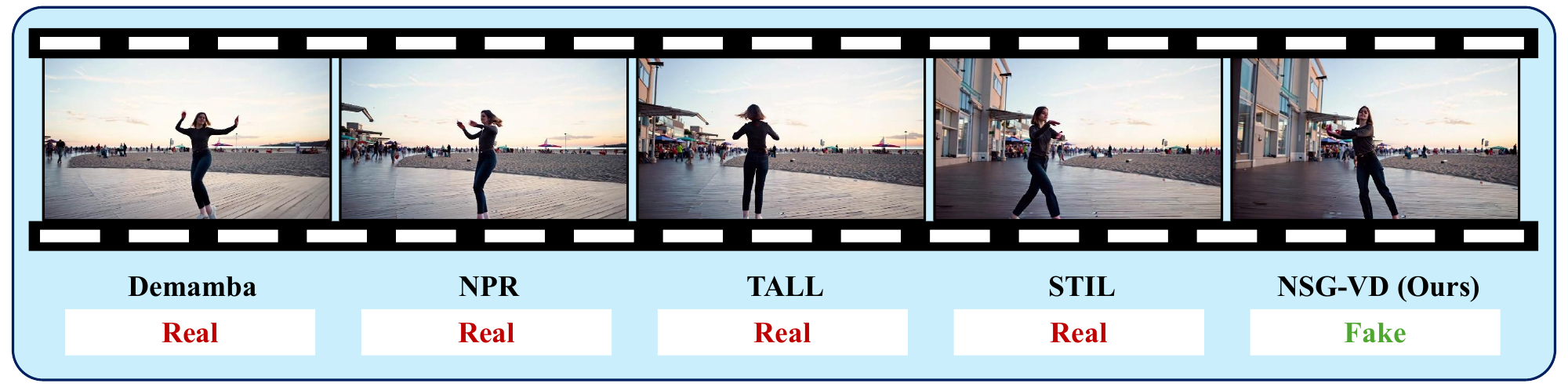}}
    {\includegraphics[width=0.95\linewidth]{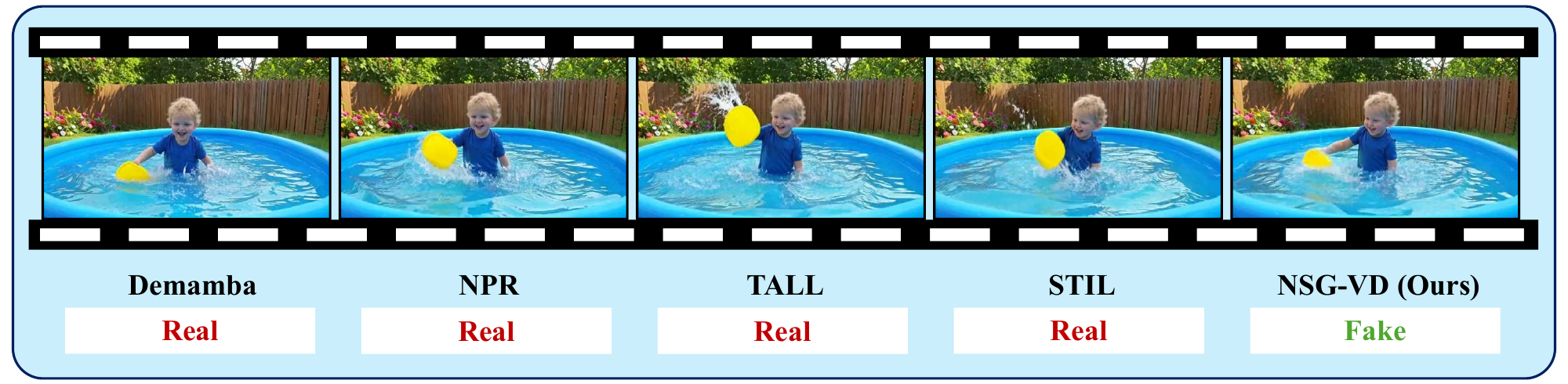}}
    \caption{Results of the detection on \textit{generated} videos from the Seaweed dataset.}
    \label{fig: visualization_Seaweed_2}
    \end{center}
\end{figure*}

\begin{figure*}[!h]
    \begin{center}
    {\includegraphics[width=0.95\linewidth]{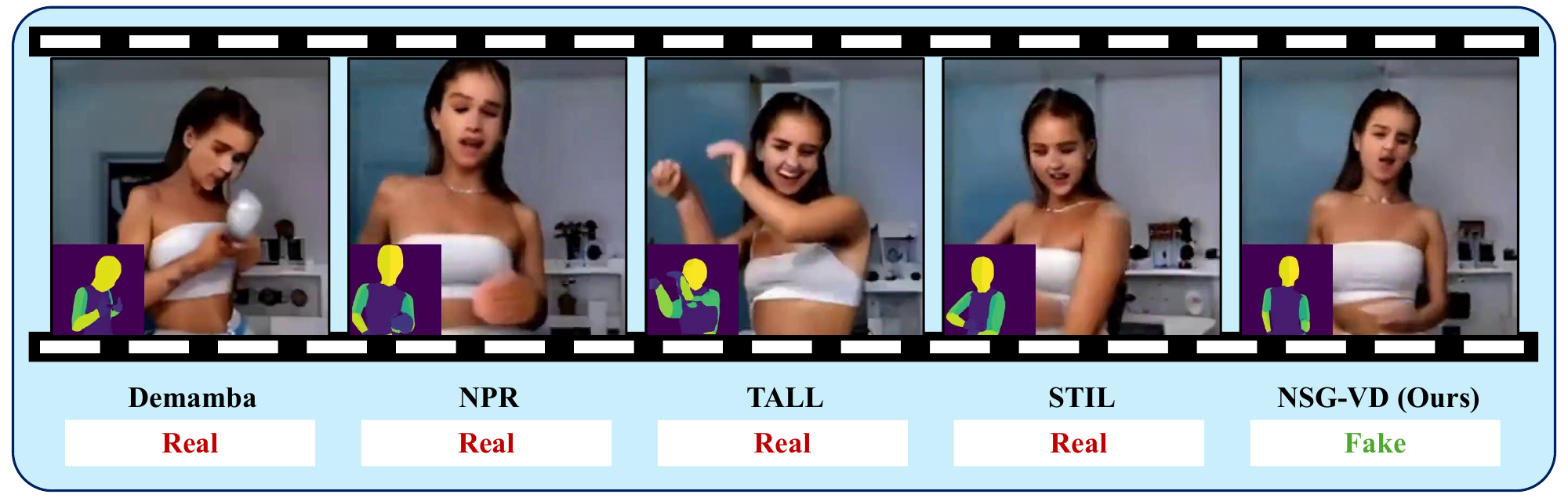}}
    {\includegraphics[width=0.95\linewidth]{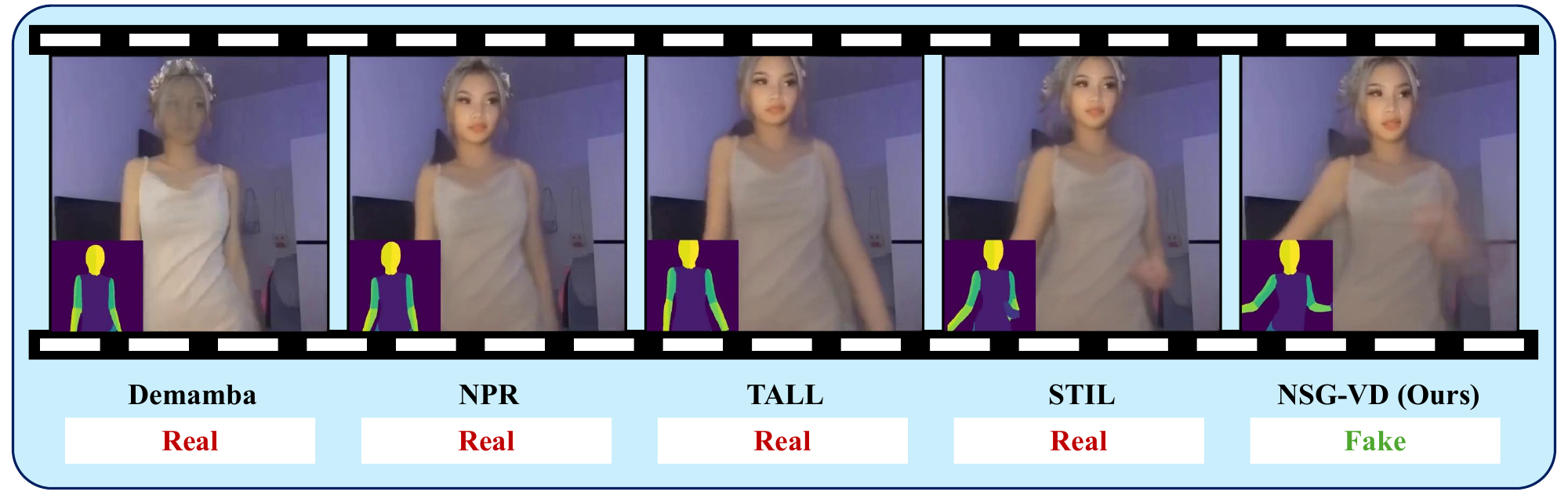}}
    {\includegraphics[width=0.95\linewidth]{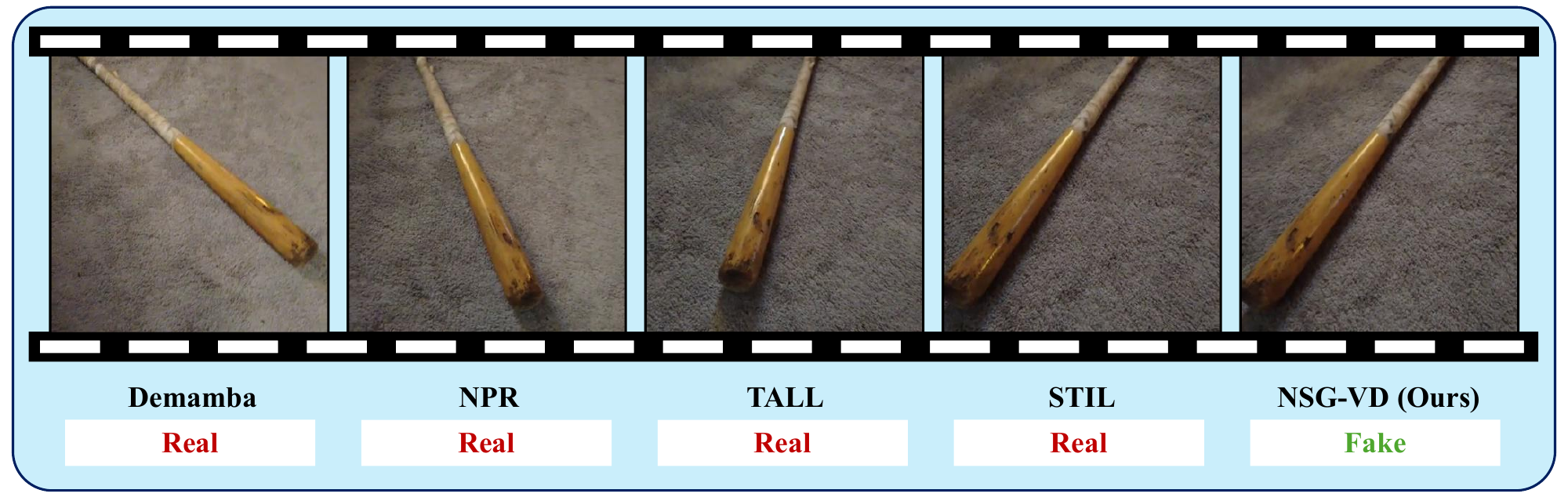}}
    \caption{Results of the detection on \textit{generated} videos from the WildScrape dataset.}
    \label{fig: visualization_WildScrape}
    \end{center}
\end{figure*}

To further demonstrate the excellent performance of our NSG-VD, we present visual detection results on both real and generated videos across all 10 datasets. As illustrated in Figures \ref{fig: visualization_Natural}-\ref{fig: visualization_WildScrape}, both the baselines and NSG-VD demonstrate satisfactory detection on real video samples. For generated videos, the existing baselines achieve reasonable performance on early generation models (e.g., Crafter, Gen2, and MoonValley), but exhibit significant performance degradation when applied to more advanced generative models (e.g., Show1, Sora, and WildScrape). In contrast, NSG-VD consistently achieves strong detection performance across all generation levels.

On this basis, we consider the recently proposed Seaweed \cite{seawead2025seaweed} method (as shown in Figures \ref{fig: visualization_Seaweed_1}, \ref{fig: visualization_Seaweed_2}), which generates highly realistic long-form videos. All four baselines exhibit near-complete failure on this dataset, whereas NSG-VD continues to deliver effective detection performance.

\newpage
\section*{NeurIPS Paper Checklist}

\begin{enumerate}

\item {\bf Claims}
    \item[] Question: Do the main claims made in the abstract and introduction accurately reflect the paper's contributions and scope?
    \item[] Answer: \answerYes{} 
    \item[] Justification: The abstract and introduction clearly state the paper's contributions, including the proposed Normalized Spatiotemporal Gradient (NSG) statistic, the NSG-VD detection method based on probability flow conservation, and the theoretical analysis.

\item {\bf Limitations}
    \item[] Question: Does the paper discuss the limitations of the work performed by the authors?
    \item[] Answer: \answerYes{} 
    \item[] Justification:  We discuss the limitations in Appendix \ref{sec: Limitations and Future Directions}.

\item {\bf Theory assumptions and proofs}
    \item[] Question: For each theoretical result, does the paper provide the full set of assumptions and a complete (and correct) proof?
    \item[] Answer: \answerYes{} 
    \item[] Justification: The paper provides complete theoretical results for the NSG feature lower bound (Theorem \ref{thm: bound of NSG} in Section \ref{sec: Theoretical Analysis}), including assumptions (\eg, Gaussian-distributed real/fake videos and temporal derivatives constraints). The proof is detailed in Appendix \ref{sec: proof bound}, including all mathematical derivations and references to supporting lemmas. All theorems and lemmas are numbered and cross-referenced, and proofs are accessible in Appendix \ref{sec:proofs}.

    \item {\bf Experimental result reproducibility}
    \item[] Question: Does the paper fully disclose all the information needed to reproduce the main experimental results of the paper to the extent that it affects the main claims and/or conclusions of the paper (regardless of whether the code and data are provided or not)?
    \item[] Answer: \answerYes{} 
    \item[] Justification: We fully disclose all information needed to reproduce the main experimental results of the paper, see Section \ref{sec: experiments} and Appendix \ref{sec: more details for exp}.

\item {\bf Open access to data and code}
    \item[] Question: Does the paper provide open access to the data and code, with sufficient instructions to faithfully reproduce the main experimental results, as described in supplemental material?
    \item[] Answer: \answerNo{} 
    \item[] Justification: We will release our code upon acceptance.

\item {\bf Experimental setting/details}
    \item[] Question: Does the paper specify all the training and test details (e.g., data splits, hyperparameters, how they were chosen, type of optimizer, etc.) necessary to understand the results?
    \item[] Answer: \answerYes{} 
    \item[] Justification: We provide full experimental detail content in our experimental settings (see Section \ref{sec: experiments} and Appendix \ref{sec: more details for exp}).

\item {\bf Experiment statistical significance}
    \item[] Question: Does the paper report error bars suitably and correctly defined or other appropriate information about the statistical significance of the experiments?
    \item[] Answer: \answerYes{} 
    \item[] Justification: We evaluate statistical significance by reporting mean and standard deviation across multiple runs with different random seeds (Appendix \ref{sec: Results of random seed}).

\item {\bf Experiments compute resources}
    \item[] Question: For each experiment, does the paper provide sufficient information on the computer resources (type of compute workers, memory, time of execution) needed to reproduce the experiments?
    \item[] Answer: \answerYes{} 
    \item[] Justification: We provide detailed information about on computing resources in Appendix \ref{sec: Details on Our Method}.
    
\item {\bf Code of ethics}
    \item[] Question: Does the research conducted in the paper conform, in every respect, with the NeurIPS Code of Ethics \url{https://neurips.cc/public/EthicsGuidelines}?
    \item[] Answer: \answerYes{} 
    \item[] Justification: The paper meets the NeurIPS Code of Ethics.

\item {\bf Broader impacts}
    \item[] Question: Does the paper discuss both potential positive societal impacts and negative societal impacts of the work performed?
    \item[] Answer: \answerYes{} 
    \item[] Justification: We discuss broader impacts in Appendix \ref{sec: Broader Impacts}.
    
\item {\bf Safeguards}
    \item[] Question: Does the paper describe safeguards that have been put in place for responsible release of data or models that have a high risk for misuse (e.g., pretrained language models, image generators, or scraped datasets)?
    \item[] Answer: \answerNA{} 
    \item[] Justification: This paper poses no such risks.

\item {\bf Licenses for existing assets}
    \item[] Question: Are the creators or original owners of assets (e.g., code, data, models), used in the paper, properly credited and are the license and terms of use explicitly mentioned and properly respected?
    \item[] Answer: \answerYes{} 
    \item[] Justification: We strictly follow the license of the assets.

\item {\bf New assets}
    \item[] Question: Are new assets introduced in the paper well documented and is the documentation provided alongside the assets?
    \item[] Answer: \answerNA{} 
    \item[] Justification: The paper does not release new assets.

\item {\bf Crowdsourcing and research with human subjects}
    \item[] Question: For crowdsourcing experiments and research with human subjects, does the paper include the full text of instructions given to participants and screenshots, if applicable, as well as details about compensation (if any)? 
    \item[] Answer: \answerNA{} 
    \item[] Justification: The paper does not involve crowdsourcing nor research with human subjects.

\item {\bf Institutional review board (IRB) approvals or equivalent for research with human subjects}
    \item[] Question: Does the paper describe potential risks incurred by study participants, whether such risks were disclosed to the subjects, and whether Institutional Review Board (IRB) approvals (or an equivalent approval/review based on the requirements of your country or institution) were obtained?
    \item[] Answer: \answerNA{} 
    \item[] Justification: The paper does not involve crowdsourcing nor research with human subjects.

\item {\bf Declaration of LLM usage}
    \item[] Question: Does the paper describe the usage of LLMs if it is an important, original, or non-standard component of the core methods in this research? Note that if the LLM is used only for writing, editing, or formatting purposes and does not impact the core methodology, scientific rigorousness, or originality of the research, declaration is not required.
    \item[] Answer: \answerNA{} 
    \item[] Justification: This paper does not involve LLMs.

\end{enumerate}

\end{document}